%% file: arxiv_main.tex
\documentclass[11pt]{article}

\usepackage[margin=1in]{geometry}
\usepackage[utf8]{inputenc}
\usepackage[T1]{fontenc}
\usepackage{lmodern}
\usepackage{graphicx}
\usepackage{url}
\usepackage{booktabs}
\usepackage{amsmath,amssymb,amsfonts,amsthm}
\usepackage{mathtools}
\usepackage{enumitem}
\usepackage{array}
\usepackage{tabularx}
\usepackage{multirow}
\usepackage{float}
\usepackage{longtable}
\usepackage{wrapfig}
\usepackage{nicefrac}
\usepackage{microtype}
\usepackage[table]{xcolor}
\usepackage{titletoc}
\usepackage[most]{tcolorbox}
\usepackage[numbers,sort&compress]{natbib}
\usepackage[colorlinks=true,linkcolor=blue,citecolor=blue,urlcolor=blue]{hyperref}
\usepackage[nameinlink,capitalise,noabbrev]{cleveref}

\newtheorem{theorem}{Theorem}
\newtheorem{lemma}{Lemma}
\newtheorem{proposition}{Proposition}
\newtheorem{remark}{Remark}

\newtheorem{corollary}{Corollary}

\newcommand{\sphere}{\mathcal S^{d-1}}
\newcommand{\R}{\mathbb R}
\newcommand{\E}{\mathbb E}
\newcommand{\abs}[1]{\left|#1\right|}
\newcommand{\ip}[2]{\left\langle #1, #2 \right\rangle}
\newcommand{\norm}[1]{\left\lVert #1 \right\rVert}
\newcommand{\N}{\mathcal N}

\newcolumntype{Y}{>{\raggedright\arraybackslash}X}

\crefname{theorem}{Thm.}{Thms.}
\Crefname{theorem}{Thm.}{Thms.}
\crefname{lemma}{Lem.}{Lems.}
\Crefname{lemma}{Lem.}{Lems.}
\crefname{proposition}{Prop.}{Props.}
\Crefname{proposition}{Prop.}{Props.}
\crefname{corollary}{Cor.}{Cors.}
\Crefname{corollary}{Cor.}{Cors.}
\crefname{definition}{Def.}{Defs.}
\Crefname{definition}{Def.}{Defs.}
\crefname{remark}{Rem.}{Rems.}
\Crefname{remark}{Rem.}{Rems.}
\crefname{equation}{Eq.}{Eqs.}
\Crefname{equation}{Eq.}{Eqs.}
\crefname{figure}{Fig.}{Figs.}
\Crefname{figure}{Fig.}{Figs.}
\crefname{table}{Tab.}{Tabs.}
\Crefname{table}{Tab.}{Tabs.}
\crefname{section}{Sec.}{Secs.}
\Crefname{section}{Sec.}{Secs.}

\title{Contrastive Neural Algorithmic Reasoning for Graph Coloring}

\author{%
  Thien Le\\
  Harvard University\\ SEAS\\
  \texttt{thien\_le@seas.harvard.edu}
\and 
  Tianyu Zhao\\
  Harvard University\\T.H. Chan School of Public Health
\\
  \texttt{tzhao1@hsph.harvard.edu}
\and
  Melanie Weber\\
  Harvard University\\SEAS\\
  \texttt{mweber@g.harvard.edu}
}

\date{}

\begin{document}

\maketitle

\begin{abstract}
Graph coloring seeks to assigns colors to a graph's nodes so that adjacent nodes receive different colors, using as few colors as possible. Here, we study approximate $k$-coloring, where the goal is to use at most $k$ colors while minimizing the number of monochromatic edges. This problem is central to graph theory and has applications in areas such as scheduling and resource allocation. Recent unsupervised GNN approaches optimize each instance directly, precluding generalization across graph sizes and distributions. We instead propose a contrastive learning framework that learns transferable coloring geometry where the embeddings of same-color nodes align, while adjacent nodes' representations are pushed toward distinct directions. We analyze the resulting population objective over bounded-size graphs. For unit-norm embeddings, we show that its optima have a line-prototype structure: Representations of nodes of the same color collapse to a shared one-dimensional subspace, and edges connect orthogonal subspaces. This geometry yields stationarity conditions in the supervised setting and is preserved by projected subgradient dynamics under a balanced-coloring assumption. In an unnormalized variant, gradient descent has a max-margin bias governed by a quotient-graph hard-margin problem. Experiments on synthetic and real-world graphs show that contrastive GNN encoders generalize effectively and produce low-conflict colorings, matching and sometimes improving on greedy approaches.
\end{abstract}

\input{sections/introduction}
\input{sections/setup}
\input{sections/absolute_value_contrastive_loss}
\input{sections/supervised_training_with_optimal_colorings}
\input{sections/implicit_bias_toward_orthogonal_line_prototypes}
\input{sections/lovasz_theta_certificates}
\input{sections/experiments}
\input{sections/limitations}

\section*{Acknowledgments}
This research was developed with funding from the Defense Advanced Research Projects Agency
(DARPA) under agreements no.
HR0011-25-3-0205 and HR0011262E027. The views, opinions, and/or findings expressed
are those of the authors and should not be interpreted as representing the official views or policies
of the Department of Defense or the U.S. Government. MW acknowledges partial support from an
Alfred P. Sloan Fellowship in Mathematics and the AI2050 program at Schmidt Sciences (Grant
G-25-69786).

\bibliographystyle{plainnat}
\bibliography{references}

\newpage
\startcontents[appendix]

\appendix

\addcontentsline{toc}{section}{Appendix Table of Contents}
\printcontents[appendix]{}{1}{\setcounter{tocdepth}{2}}

\input{appendix/refresher}
\input{appendix/proofs_absolute_value_loss_appendix}
\input{appendix/proofs_supervised_training_appendix}
\input{appendix/proofs_implicit_bias_appendix}

\input{appendix/proofs_lovasz}
\input{appendix/cora_ablation_experiments_appendix}

\input{appendix/main_citation_experiments}

\input{appendix/main_color_experiments}
\input{appendix/main_cycle_experiments}

\end{document}

%% file: sections/introduction.tex
\section{Introduction}
Graph coloring is a canonical problem in discrete optimization. Given a graph $G=(V,E)$, the goal is to assign colors to vertices such that adjacent vertices receive distinct colors, while minimizing the total number of colors used. This minimum is known as the chromatic number. Classical hardness results show that, for any $\epsilon>0$, approximating the chromatic number within a factor of $O(n^{1-\epsilon})$ is NP-hard~\citep{Wigderson1983,karger98colorsdp,feigekilian1998zeroknowledge,dinur2009coloringugc}. Despite this worst-case intractability, graph coloring has received sustained interest because of its broad range of applications, such as resource allocation and scheduling, where colors correspond to resources and edges to conflicts in resource usage. In practice, heuristics such as greedy algorithms~\citep{daniel1979dsatur} and SAT-based solvers~\citep{selman1992newmethod,vangelder2008coloring} are used to solve real-world instances. More recently, unsupervised machine learning approaches have been introduced~\citep{schuetz2022pignn,vanderbush2026warmstart}, typically learning color logits through a physics-inspired loss.

While these methods apply to arbitrary graphs, they either offer limited interpretability (greedy/ SAT-based solvers) or lack a mechanism for incorporating prior knowledge about the test graphs (unsupervised learning). As a result, it is difficult \textit{a priori} to characterize the graph classes on which they are likely to perform well, making it challenging for practitioners working with specific graph families to assess their effectiveness. Moreover, SAT-based solvers and unsupervised learning methods typically optimize each graph instance independently, which can be computationally expensive. We propose a pipeline that addresses these limitations along two dimensions: \emph{interpretability}, by leveraging the geometry of node representations, and \emph{scalability and domain-knowledge incorporation}, by adopting a supervised learning framework. These goals are closely connected. Since optimal coloring is computationally intractable on unrestricted inputs, one should not expect a single method to solve the problem in full generality. Instead, it is natural to exploit the simplifying geometric structure present in node embeddings for restricted graph families of interest. 


Our approach leverages contrastive learning to provide a simple representation-level view of graph coloring. Given a coloring, vertices assigned the same color are treated as positive pairs, while adjacent vertices are treated as negative pairs. We train a GNN encoder using an absolute-value variant of InfoNCE: rather than aligning directed vectors, the objective aligns unoriented lines. Under this geometry, each color is represented by a \emph{prototype}, corresponding to a one-dimensional subspace, and vertices of that color may be mapped to either direction along the same line. Decoding then consists of canonicalizing each prototype line to a point on the sphere, followed by clustering the resulting embeddings to determine vertex colors. {Our proposed algorithms is schematically shown in Fig.~\ref{fig:schematic}.} 


\subsection{Contribution of this paper}
As far as the authors are aware, this paper proposes the first neural supervised learning approach to graph coloring that comes with a colorability certificate. 
More formally, we study graph coloring {with at most $k_{\max}$ colors} via a population objective over a distribution $\mathcal D_{\mathcal G}$ supported on graphs of bounded size. For each graph $G$, a coloring rule $c_G$ supplies the positive pairs, and a node embedding function {$f(G) \in \mathcal{S}^{d-1}$} maps vertices to the unit sphere. Our main theoretical question is as follows:

\begin{tcolorbox}
    Given a proper (or optimal) coloring rule, what embedding geometry is enforced by a 
    contrastive objective, and how does optimization select among the resulting prototype configurations?
\end{tcolorbox}

Our answer to this question is three-fold:
\begin{enumerate}[leftmargin=1.5em]
    \item We prove a sharp lower bound for the absolute-value InfoNCE objective and characterize every optimizer that attains it (\Cref{thm:abs-lower-bound,thm:abs-characterization}). Under the assumption that $d\ge k_{\max}$, collapsed orthogonal color prototypes attain the bound. Conversely, equality forces same-color vertices onto common lines and edge-adjacent vertices into orthogonal directions. For optimal colorings, this ensures graph-wise line collapse by chromatic color class (\Cref{cor:optimal-coloring-collapse}).
    \item We analyze supervised training when graph-wise optimal colorings are provided. The population objective decomposes over graphs (\Cref{lem:graphwise-decomposition}), and the collapsed-prototype restriction gives explicit {Clarke-stationary equations}~\citep{rockafellar1998variational,li2020nonsmooth} (also see Appendix, \Cref{app:nonsmooth} for a refresher) governed by the coloring's weighted quotient graph (\Cref{thm:prototype-stationary}). These equations bridge the loss geometry with the graph structure. 
    \item Third, we study optimization and certification. Under an equitable optimal-coloring assumption, projected subgradient dynamics preserve the collapsed prototype manifold 
    (\Cref{thm:invariance}). In an unnormalized homogeneous prototype model, gradient descent maximizes the minimum absolute contrastive margin on the active quotient graph in the max-margin limit (\Cref{thm:abs-max-margin}). Approximate edgewise orthogonality also yields a Lov\'asz $\vartheta(\overline G)$ certificate (\Cref{lemma:lovasz-certificate,cor:lovasz-prototype-certificate}). 
\end{enumerate}

We complement our theoretical analysis with experiments on synthetic combinatorial instances and real-world datasets, showing that our approach supports size generalization and out-of-distribution transfer across graph families (\Cref{sec:citation-experiments}). Well-tuned message-passing and GPS-style encoders match, and sometimes outperform, greedy coloring baselines.

Together, these results provide a theoretical and empirical account of what coloring-aware contrastive training can certify, through a precise geometric mechanism by which color partitions become line prototypes and graph edges become orthogonality constraints.


\begin{wrapfigure}{r}{0.48\textwidth}
    \centering
    \vspace{-1.0em}
    \includegraphics[width=0.46\textwidth]{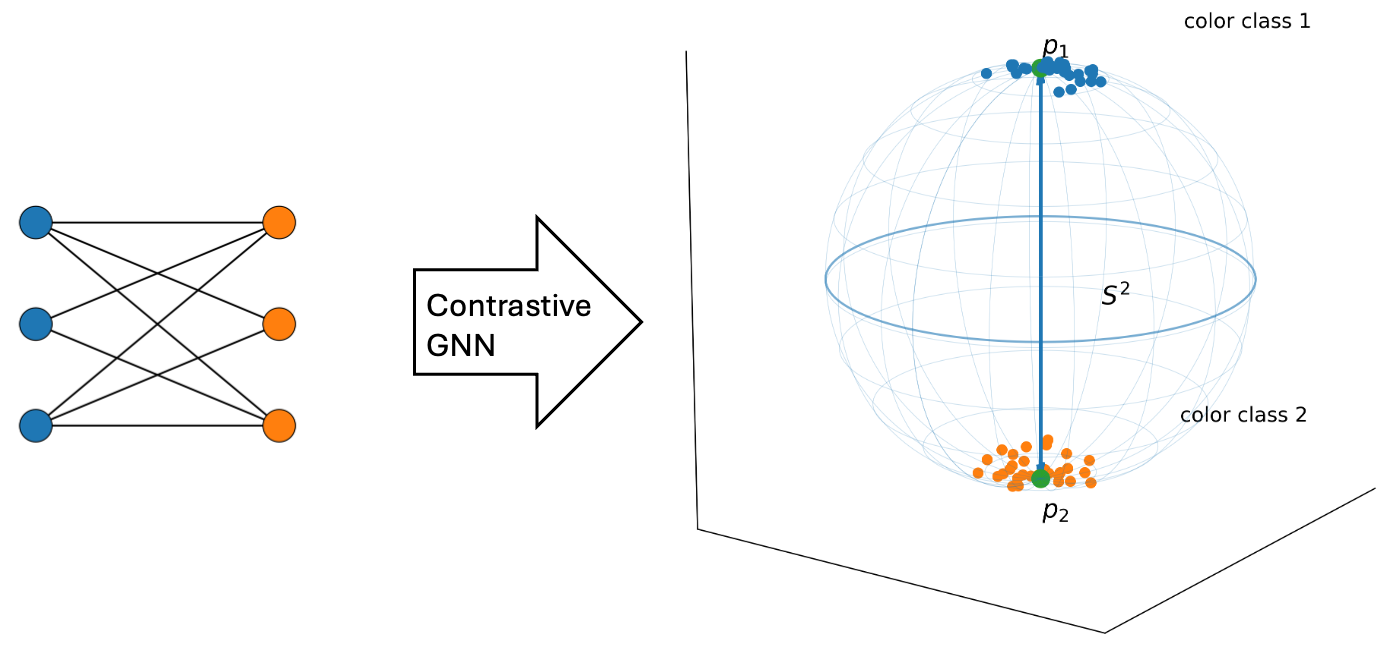}
    \caption{Supervised contrastive coloring framework. A graph encoder maps graph vertices to unit-norm embeddings. At training, the InfoNCE loss (or its variants) learns to push neighboring vertices' representation away from each other and attracts same-color vertices' representation. At testing, a clustering algorithm is run to cluster test graph vertices' embeddings.}
    \label{fig:schematic}
    \vspace{-1.0em}
\end{wrapfigure}

\subsection{Related works}

Graph coloring has long served as a central test case for the limits of efficient optimization. This history motivates relaxations and certificates, not exact recovery alone. Our work is closest in spirit to geometric relaxations of coloring, including vector colorings, semidefinite programming relaxations, and Lov\'asz's theta function \citep{karger98colorsdp,lovasz1979shannon}. We differ in the starting point: instead of solving an SDP, we ask when a learned contrastive representation produces the relevant orthogonality geometry.

Although our focus is on supervised learning of graph coloring, there exists a prolific line of SAT solvers \citep{selman1992newmethod} and neural algorithmic reasoning to tackle unsupervised coloring \citep{schuetz2022pignn,vanderbush2026warmstart}. While the performances of the unsupervised approaches are not affected by the choice of training datasets, at inference time, they run much more slowly as they attempt to solve the coloring problem from scratch without any data priors.  


The line-collapse phenomenon is related to neural collapse in supervised classification, where within-class features collapse and class means often approach a simplex ETF geometry \citep{papyan2020prevalence}. Our setting differs in that stationary equations depend on the weighted {quotient graph} rather than only on class balance. 

%% file: sections/setup.tex
\section{Setup}
In this work we study graph coloring from the perspective of representation geometry rather than combinatorial recovery. We now describe our approach more formally by fixing the graph distribution, the coloring rule that supplies positives, and the two contrastive objectives whose learned representation geometry we compare.

We write $\mathrm{Im}$ for the image of a function. Let $\mathcal G_{\le n_{\max}}$ be a finite family of simple graphs, each written as $G=([n_G],E_G)$ with $1\le n_G\le n_{\max}$. Let $\mathcal D_{\mathcal G}$ be a distribution supported on $\mathcal G_{\le n_{\max}}$. A proper coloring rule is a family $c=\{c_G\}_{G\in \mathrm{supp}(\mathcal D_{\mathcal G})}$ where each
$c_G:[n_G]\to [k_{\max}]
$ is proper on $G$, and
$k_{\max}:=\max_{G\in \mathrm{supp}(\mathcal D_{\mathcal G})} |{\mathrm{Im}}(c_G)|
$.
An embedding function is a family $f=\{f(G)\}_{G\in \mathrm{supp}(\mathcal D_{\mathcal G})}$ with
$f(G)=([f(G)]_1,\dots,[f(G)]_{n_G})\in (\sphere)^{n_G}
$
unless we explicitly drop the unit-norm constraint. For a graph $G$, we write $V_i(G):=c_G^{-1}(i)$ and $\N_G(v)$ for the neighbor set of $v$ in $G$.

\textbf{(Color-)quotient graph.} The coloring rule induces a quotient graph: for each $G$, 
    $Q_c(G) = \bigl(\mathrm{Im}(c_G), E_{Q_c(G)}\bigr)$, 
where distinct colors $i,j$ are adjacent in $Q_c(G)$ iff there exists an edge $(u,v)\in E_G$ with $c_G(u)=i$ and $c_G(v)=j$. \\

\textbf{Signed InfoNCE.}
For a fixed graph $G$ and coloring $c_G$, the local signed contrastive objective is:
\begin{equation}
\ell_{\mathrm{InfoNCE},\tau}^{(G)}(h,c_G)
:=
\frac{1}{n_G}\sum_{v=1}^{n_G}\frac{1}{|V_{c_G(v)}(G)|}
\sum_{w\in V_{c_G(v)}(G)}
\left[
-
\frac{e^{\ip{h_v}{h_w}/\tau}}
{e^{\ip{h_v}{h_w}/\tau}+\sum_{u\in \N_G(v)} e^{\ip{h_v}{h_u}/\tau}}
\right].
\end{equation}

In effect, this is the standard InfoNCE loss with the vertices of $G$ as data points. Positive pairs, which should be mapped to nearby points in representation space, are vertices assigned the same color by $c_G$, while negative pairs, which are mapped to distant points, correspond to edges in the graph, intuitively, making each color forms its own cluster. The population objective is
$\ell_{\mathrm{InfoNCE},\tau}(f,c):=
\E_{G\sim \mathcal D_{\mathcal G}}
\left[
\ell_{\mathrm{InfoNCE},\tau}^{(G)}(f(G),c_G)
\right]$.

\textbf{Absolute-value InfoNCE.}
It turns out that a slight modification of the standard InfoNCE loss can give us both a clean theoretical understanding and a certificate of colorability. We introduce the absolute-value variant of the InfoNCE loss replacing $\ip{h_v}{h_w}$ by $\abs{\ip{h_v}{h_w}}$ for all positive and negative pairs $(v,w)$. For a fixed graph $G$, define:
\begin{align}
\ell_{\mathrm{abs},\tau}^{(G)}(h,c_G)
&:=
\frac{1}{n_G}\sum_{v=1}^{n_G}\frac{1}{|V_{c_G(v)}(G)|}
\sum_{w\in V_{c_G(v)}(G)}
\left[
-
\frac{e^{\abs{\ip{h_v}{h_w}}/\tau}}
{e^{\abs{\ip{h_v}{h_w}}/\tau}+\sum_{u\in \N_G(v)} e^{\abs{\ip{h_v}{h_u}}/\tau}}
\right],\\
\ell_{\mathrm{abs},\tau}(f,c) &:=
\E_{G\sim \mathcal D_{\mathcal G}}
\left[
\ell_{\mathrm{abs},\tau}^{(G)}(f(G),c_G)
\right].
\end{align}
\begin{remark}
 Compared to the standard InfoNCE loss, this objective does not encourage neighboring colors to be antipodal. Instead, it favors embeddings in which same-color vertices lie on a common unoriented line, while edge-adjacent vertices lie in orthogonal directions. We will later use this form of color separation to obtain colorability certificates via Lov\'asz $\vartheta$ bounds~\citep{lovasz1979shannon}.
\end{remark}

A simple fact serves as sanity check for our candidate loss functions:
\begin{proposition}\label{prop:existence-signed}
Fix a proper coloring rule $c$ on $\mathrm{supp}(\mathcal D_{\mathcal G})$. Under the unit-norm constraint, the minimum of $\ell_{\mathrm{InfoNCE},\tau}(f,c)$, $\ell_{\mathrm{abs},\tau}(f,c)$ over all embedding functions $f$ is attained for any $\tau > 0$.
\end{proposition}

\textbf{Learning framework.}
Under this loss, our supervised learning framework can be stated as follows.
Given a training dataset together with a proper or approximate coloring rule, we first train an encoder $f_\theta(G) = (h_1,\ldots,h_{n_G})$, $h_v\in \mathbb{S}^{d-1}$,
by minimizing $\ell_{\mathrm{InfoNCE},\tau}(f,c)$ or $\ell_{\mathrm{abs},\tau}(f,c)$. The objective encourages vertices with the same color label to collapse onto a common unoriented line, while adjacent vertices are pushed toward orthogonal lines. At inference time, first, we perform a \emph{line canonicalization} step (if using the absolute-value variant), e.g. by forcing the first non-zero entry to be positive. This makes the embeddings compatible with ordinary Euclidean clustering while leaving the loss value unchanged. Second, we cluster the canonicalized embeddings, typically by sweeping $k$ in $k$-medoids or $k$-means, and select the smallest $k$ whose decoded coloring satisfies the desired monochromatic-edge threshold.

The framework is modular. The contrastive encoder can be combined with additional graph-learning signals. For example, one may attach a readout head that predicts color logits
    $q_v = g_\theta(h_v)\in\mathbb{R}^K$
and train it with an auxiliary edge-conflict or soft-coloring loss (which we use in one of our variants in the experiments). 

%% file: sections/absolute_value_contrastive_loss.tex
\section{Global minima of the Absolute-value InfoNCE loss}
We first ask what the absolute-value loss enforces when the coloring rule is given during training. We show that the objective admits a universal lower bound, with equality attained exactly by line collapse together with edgewise orthogonality. The full proofs are in Appendix, \Cref{app:absolute_value_minima}.

\begin{theorem}\label{thm:abs-lower-bound}
For every proper coloring rule $c$ and every unit-norm embedding function $f$,
\begin{equation}
\ell_{\mathrm{abs},\tau}(f,c)
\ge
\E_{G\sim \mathcal D_{\mathcal G}}
\left[
\frac{1}{n_G}\sum_{v=1}^{n_G}
\left(
-\frac{e^{1/\tau}}{e^{1/\tau}+|\N_G(v)|}
\right)
\right].
\end{equation}
\end{theorem}

\begin{theorem}\label{thm:abs-characterization}
Assume $d\ge k_{\max}$. Then a collapsed orthogonal prototype embedding function attains the lower bound in \Cref{thm:abs-lower-bound}. More precisely, there exist orthonormal vectors $q_1,\dots,q_{k_{\max}}\in \R^d$ such that the embedding function defined by
\begin{equation}
[f^c(G)]_v=q_{c_G(v)}
\qquad \text{for every } G\in \mathrm{supp}(\mathcal D_{\mathcal G}) \text{ and } v\in [n_G]
\end{equation}
achieves equality. Moreover, for any unit-norm embedding function $f$, equality in \Cref{thm:abs-lower-bound} holds if and only if
\begin{equation}
\abs{\ip{[f(G)]_v}{[f(G)]_w}}=1
\quad \text{whenever } G\in \mathrm{supp}(\mathcal D_{\mathcal G}) \text{ and } c_G(v)=c_G(w),
\end{equation}
and
\begin{equation}
\abs{\ip{[f(G)]_v}{[f(G)]_u}}=0
\quad \text{whenever } G\in \mathrm{supp}(\mathcal D_{\mathcal G}) \text{ and } uv\in E_G.
\end{equation}
Equivalently, every color class lies on a one-dimensional subspace and every edge joins orthogonal subspaces on every graph in the support.
\end{theorem} %

In particular, the global minima characterization of \Cref{thm:abs-characterization} exhibit a kind of regularization for the dimensionality of the final representation. The fact that $\abs{\ip{[f(G)]_v}{[f(G)]_w}}=1$ when $w,v$ have the same color means that there is no redundancy in the number of colors. As a result, when the training coloring rule is optimal, the absolute-value InfoNCE loss optimizes also the number of colors.   
\begin{corollary}\label{cor:optimal-coloring-collapse}
Let $c^\star=\{c_G^\star\}_{G\in \mathrm{supp}(\mathcal D_{\mathcal G})}$ be an optimal coloring rule, i.e. each $c_G^\star$ uses exactly $\chi(G)$ colors, and assume $d\ge \chi_{\max}:=\max_{G\in \mathrm{supp}(\mathcal D_{\mathcal G})}\chi(G)$. Then every global optimizer $f^\star$ of $\ell_{\mathrm{abs},\tau}(\cdot,c^\star)$ is line-collapsed graphwise: for every $G\in \mathrm{supp}(\mathcal D_{\mathcal G})$ there exist pairwise orthogonal unit vectors $q_{G,1},\dots,q_{G,\chi(G)}$ and signs $\sigma_{G,v}\in \{\pm 1\}$ such that
$[f^\star(G)]_v=\sigma_{G,v} q_{G,c_G^\star(v)}$.
\end{corollary}


%% file: sections/supervised_training_with_optimal_colorings.tex
\section{Supervised training with optimal colorings}
\Cref{thm:abs-characterization} and \Cref{cor:optimal-coloring-collapse} show what happens at global minima. It turns out that a precise characterization can be obtained even at stationary points of the per-graph loss function. To justify analyzing the per-graph loss rather than the population loss, we use the fact that sufficiently expressive graph networks can approximate broad classes of invariant or equivariant graph functions provided that higher-order tensorization, node identifiers, random initialization, or sufficient depth and width are allowed \citep{maron2019universality,keriven2019universal,loukas2020depthwidth,azizian2021expressive,abboud2021random}.
We have: 


\begin{lemma}\label{lem:graphwise-decomposition}
For a large enough GNN $f$, the supervised population objective decomposes graphwise:
\begin{equation}
\inf_f \mathcal L_{\mathrm{abs},\tau}(f;c^\star)
=
\sum_{G\in \mathrm{supp}(\mathcal D_{\mathcal G})}
\mathcal D_{\mathcal G}(G)\,
\inf_{h\in (\sphere)^{n_G}}
\ell_{\mathrm{abs},\tau}^{(G)}(h,c_G^\star).
\end{equation}
Consequently, one can obtain a population minimizer by choosing, for each graph $G$, a graphwise minimizer of the corresponding local supervised objective. In particular, under $d\ge k_{\max}$ every global minimizer of $\mathcal L_{\mathrm{abs},\tau}^{\mathrm{sup}}(\cdot;c^\star)$ is graphwise line-collapsed by color as in \Cref{cor:optimal-coloring-collapse}.
\end{lemma}

While the above lemma allows us to focus on line-collapsed representations, understanding which configurations are stationary requires computing the stationarity equations. They show that the loss depends on the graph structure only in a specific way. We state only the informal version, deferring the formal statement (\Cref{thm:prototype-stationary-restatement}) and proof to Appendix, \Cref{app:proof_supervised}. 

\begin{theorem}[Informal]\label{thm:prototype-stationary}
Fix a graph $G=([n],E)\in \mathrm{supp}(\mathcal D_{\mathcal G})$ and write $k_G:=\chi(G)$. Let $V_i:=\bigl(c_G^\star\bigr)^{-1}(i)$ and define the per-color neighbor counts of $v \in V_i$ for some $i$ to be $d_{v,j}:=\abs{\N_G(v)\cap V_j}$, for any $j\neq i$.
For a prototype tuple $q=(q_1,\dots,q_{k_G})\in (\sphere)^{k_G}$, define the collapsed embedding $h(q)_v:=q_{c_G^\star(v)}$. If we restrict the supervised local objective to only the collapsed prototype class (optimizing over $q$), then the condition for $q$ to be a {Clarke-Riemannian stationary point\footnote{See Appendix \Cref{app:nonsmooth}} for a refresher} of the objective on $(\sphere)^{k_G}$ depends on the structure of $G$ only through the $d_{v,j}$'s.
\end{theorem}

The proof follows from first principle by writing down the stationary equations. Intuitively, \Cref{thm:prototype-stationary} says that once we restrict to embeddings that collapse each optimal color class to a single line prototype, the supervised absolute-value loss becomes a finite weighted spherical-code problem over those prototypes. The weights are not uniform: they depend on how many neighbors each node has in the other color classes and on the corresponding InfoNCE denominators. A collapsed prototype configuration is stationary exactly when, for every color, the weighted ``force'' exerted by all other color lines has no component tangent to the sphere.

More importantly, local optima of our objective only interact with the graph structures via the per-color counts $d_{v,j}$. This dependence is exploited in the next sections to identify graph balancedness conditions that ensure gradient-based algorithms land in a favorable regime.


%% file: sections/implicit_bias_toward_orthogonal_line_prototypes.tex
\section{Gradient Dynamics and Max-Margin Prototype Geometry}
The previous section identifies the collapsed prototype geometry and the stationary equations within it. The next question is dynamical: if optimization starts on this geometry, does it remain there, and which prototype configurations does it favor? Under an equitable optimal-coloring assumption, gradient-based optimization is compatible with the collapsed line-prototype picture. Moreover, gradient-based algorithms exhibit an implicit bias toward max-margin solutions, following the standard implicit-bias framework for separable exponential-tail losses in linear models and homogeneous networks \citep{soudry2018implicit,lyu2019gradient,ji2020directional}. We state the theorem informally and defer the formal statement, (\Cref{thm:invariance-restatement}), and its proof to Appendix, \Cref{app:proof_dynamics}.

\begin{theorem}[informal]\label{thm:invariance}
Fix a graph $G=([n],E)\in \mathrm{supp}(\mathcal D_{\mathcal G})$ and let $c_G^\star:[n]\to [k_G]$ be an optimal coloring. Write $V_i:=\bigl(c_G^\star\bigr)^{-1}(i)$. Assume $c_G^\star$ is equitable in the sense that for every pair $i\neq j$ there is an integer $r_{ij}\ge 0$ such that
$\abs{\N_G(v)\cap V_j}=r_{ij}$, for every 
$v\in V_i$.
Then the collapsed manifold,
\begin{equation}
\mathcal M_{c_G^\star}:=\{h\in (\sphere)^n:\ \exists q_1,\dots,q_{k_G}\in \sphere \text{ with } h_v=q_{c_G^\star(v)} \ \forall v\}
\end{equation}
is invariant under GD dynamics, i.e., GD stays in this manifold once it reaches the manifold. 
\end{theorem}
Several standard graph families including complete $k$-partite, balanced Tur\'an, every $(r,s)$-biregular (even cycles, hypercubes) graphs satisfy the equitability hypothesis. One can also build equitable colorings from a $k$-chromatic quotient graph on color classes $V_1,\dots,V_k$. For each active pair $(i,j)$, replace the quotient edge by an $r_{ij}$-regular bipartite graph between $V_i$ and $V_j$. By construction, every node in $V_i$ then has the same number of neighbors in every other color class.


Furthermore, when dropping the unit-norm constraint, gradient descent accumulates at KKT points of the max-margin problem for this optimzation:
\begin{theorem}\label{thm:abs-max-margin}
Fix a graph $G=([n],E)\in \mathrm{supp}(\mathcal D_{\mathcal G})$ and an optimal equitable coloring $c_G^\star$ as in \Cref{thm:invariance}.
Assume the prototype maps $w\mapsto q_i(w)$ are positively $1$-homogeneous, and write $w=\rho \bar w$ with $\rho=\norm{w}_2$ and $\norm{\bar w}_2=1$. During training, $\rho$ and ${\bar w}$ evolve with time $t$ (e.g. steps of GD) and we are interested the dynamics $\rho(t), {\bar w}(t)$. Define the margin:
\begin{equation}
\mu_{ij}(\bar w):=\norm{q_i(\bar w)}^2-\abs{\ip{q_i(\bar w)}{q_j(\bar w)}}
\qquad \text{for each pair }(i,j)\text{ with }r_{ij}>0.
\end{equation}
Assume separability, i.e. there exists $\bar w$ such that $\mu_{ij}(\bar w)>0$ for every pair with $r_{ij}>0$, and assume standard step-size and nondegeneracy conditions for gradient descent on separable exponential-tail objectives. Then
the parameter norm {$\rho(t)$} diverges and
every accumulation point of {$\bar w(t)$} is a first-order stationary/KKT point of
\begin{equation}
\max_{\norm{w}_2=1}\ \min_{(i,j):\,r_{ij}>0}\mu_{ij}(w);
\end{equation}
Consequently, without unit normalization the absolute-value contrastive objective has implicit bias toward maximizing the minimum absolute contrastive margin across active quotient-graph pairs.
\end{theorem}

%% file: sections/lovasz_theta_certificates.tex
\section{Lov\'asz \texorpdfstring{$\vartheta$}{theta} certificates from approximate orthogonality in complement graph}
The preceding sections show that the loss and the dynamics both favor line collapse within colors and orthogonality across edges. We now ask what this learned geometry certifies about the original graph. Approximate edgewise orthogonality can be lifted to an exact orthogonal representation, which in turn gives a graph-theoretic upper-bound certificate through Lov\'asz's $\vartheta$ function \citep{lovasz1979shannon}: if $u_1,\dots,u_n$ is an orthogonal representation of $G$ and $c$ is a unit handle, then
\begin{equation}
\vartheta(\overline G)\le \max_{i\in[n]} 1/\ip{c}{u_i}^2.    
\end{equation}

The next lemma shows that a uniformly approximately orthogonal learned representation already yields such a certificate after an explicit lift. Full proofs are in Appendix \Cref{app:lovasz}.

\begin{lemma}[Approximate orthogonality implies a $\vartheta$ certificate]\label{lemma:lovasz-certificate}
Let $G=([n],E)$ be a graph, let $v_1,\dots,v_n\in \sphere$, and assume that $
\abs{\ip{v_i}{v_j}}\le \varepsilon$ for every $\{i,j\}\in E$. Let $\alpha:=\min_{i\in[n]} \abs{\ip{r}{v_i}}$ for any $r\in \sphere$. If $\alpha>0$, then
\begin{equation}
\vartheta(\overline G)
\le
\frac{1+\Delta(G)\varepsilon}{\alpha^2},
\end{equation}
where $\Delta(G):=\max_{i\in[n]} |\N_G(i)|$ is the maximum degree and $\overline{G}$ is the complement of $G$.
\end{lemma}

\begin{corollary}[Near-prototype certificate]\label{cor:lovasz-prototype-certificate}
Let $G=([n],E)$ admit a proper $k$-coloring $c:[n]\to [k]$, and let $q_1,\dots,q_k\in \sphere$ be orthonormal. Assume that for each node $i\in[n]$ there is a sign $\sigma_i\in\{\pm 1\}$ such that
$\norm{v_i-\sigma_i q_{c(i)}}\le \delta,
$
and assume also that
$\abs{\ip{v_i}{v_j}}\le \varepsilon$
for every $\{i,j\}\in E$.
If $\delta<k^{-1/2}$, then
\begin{equation}
\vartheta(\overline G)
\le
\frac{1+\Delta(G)\varepsilon}{(k^{-1/2}-\delta)^2}.
\end{equation}

\end{corollary}

\noindent
\Cref{cor:lovasz-prototype-certificate} shows that a learned representation close to the line-collapsed orthogonal prototype picture of \Cref{cor:optimal-coloring-collapse} certifies $\vartheta(\overline G)\lesssim k$. The degradation factor is explicit and depends on the edgewise orthogonality error $\varepsilon$ and the prototype misalignment $\delta$. 

%% file: sections/experiments.tex
\section{Experiments}\label{sec:citation-experiments}

\textbf{Experimental protocol.}
All experiments use the same high-level pipeline. We train a GNN encoder with the coloring-aware contrastive objective, normalize the node embeddings, and decode them by sweeping $k$ in $k$-medoids until the monochromatic-edge fraction falls below a prescribed threshold. We denote the selected value by $k^\star$ and write \textsc{Mono} for the resulting monochromatic-edge fraction. A run is counted as a `Hit' if \textsc{Mono} is less than some target threshold, by default $0.05$. Throughout the empirical section, we focus on the strongest encoder configurations identified in our ablation study. When exact chromatic numbers are not uniformly available, we normalize by the greedy-coloring proxy $\chi_{\mathrm{greedy}}$ and report $\rho = k^\star / \chi_{\mathrm{greedy}}$. When testing our performance on graph instances (as opposed to testing generalization), we compare our methods against a greedy coloring baseline and two prior unsupervised learning methods: PI-GNN \citep{schuetz2022pignn} and full-GCN \citep{vanderbush2026warmstart}. In all experiments, we have a hard per-graph runtime budget of $900$ seconds.

\textbf{Model configuration.} We perform extensive ablation across different loss functions, feature constructions, and architectures spanning both MPNNs and attention-based models. Ablation studies are performed on the CORA citation graph dataset. We subsample $200$ subgraphs of CORA for training and a separate set of $50$ subgraphs for testing. We evaluate different loss functions, including vanilla (signed) InfoNCE (denoted by `v1'), absolute-value InfoNCE (`v1abs'), InfoNCE with a read-out head that also compute a soft-loss term for edge conflict (`v2') (See \Cref{tab:cora-v1-v5}) and its absolute-value version (`v2abs'). We also tried different node feature constructions, such as using raw CORA bag-of-words (v1-v2), structural node features (`v3') or random Erdos-Renyi features (`v4') (See \Cref{tab:cora-v1-v5}). Versions `v1'/`v1abs' and `v2'/`v2abs' are selected as the strongest pipelines and further tuned. We then perform encoder ablations over seven message-passing encoders: GCN \citep{kipf2017semisupervised}, ResGCN \citep{li2019deepgcns}, GAT \citep{velickovic2018graph}, GIN \citep{xu2019powerful}, GatedGCN \citep{bresson2017residual}, GraphSAGE \citep{hamilton2017inductive}, and Unitary MP \citep{kiani2024unitary}; and three attention-based graph networks: GPS \citep{rampasek2022recipe}, Graph ViT \citep{he2023generalization}, and Exphormer \citep{shirzad2023exphormer} (see \Cref{tab:full-encoder-sweep} for details). 

We select the top strongest encoders  from both styles: GatedGCN, Unitary MP, gps-sage, gps-gcn for further hyperparameter fine-tuning 
and used for downstream experiments. Due to space constraint, we only report the best configurations for each experiments, leaving the full report to the Appendix.

\subsection{In-distribution (ID) and Out-of-distribution (OOD) generalization}
\subsubsection{COLOR benchmark}

The benchmark is constructed as an in-family generalization study. For each family, training graphs are drawn with bounded size. Performance is then measured on ID graphs (within size-bound) and larger OOD graphs (outside size-bound) from that family. \Cref{tab:color-family-splits} in the Appendix reports the metadata of this dataset. 
All runs use 80 training epochs, common input dimension $d=64$.


Two node feature constructions are used for the attribute-free COLOR graphs.
(1) \textbf{Random:} random unit vectors $x_i \in \mathbb{R}^{64}$,
(2) \textbf{Struct:} deterministic structural descriptors 
starts normalized node statistics such as degree, local clustering coefficient, core number, and PageRank, and then expands them with a fixed nonlinear basis (see Appendix, \Cref{app:color} for details) until the feature dimension reaches 64. This expansion provides a richer deterministic input basis by introducing non-linearity.

\Cref{tab:color-best-by-family} summarizes results per family. 
Across ten seeds, Book and Queen graphs are best handled by random-feature recipes, while Myciel graphs are best handled by structural features. More details and family-level deep dives are in the Appendix (e.g. \Cref{tab:color-book-detail}). Abs variants are competitive but are not the best configurations, while the direct unsupervised baselines often fail to finish in time (App. \Cref{app:color}). 
 
\begin{table}[H]
\centering
\footnotesize
\caption{Best configuration per family, ranked by OOD hit-rate, then lower OOD $\rho$, then lower OOD \textsc{Mono}, then the analogous ID criteria. {Entries are mean $\pm$ standard deviation over ten random seeds.}}
\label{tab:color-best-by-family}
\setlength{\tabcolsep}{3.5pt}
\begin{tabular}{llllrccccc}
\toprule
Family & Feature mode & Pipeline & Encoder & $\rho_{\mathrm{ID}}$ & \textsc{Mono}$_{\mathrm{ID}}$ & Hit$_{\mathrm{ID}}$ & $\rho_{\mathrm{OOD}}$ & \textsc{Mono}$_{\mathrm{OOD}}$ & Hit$_{\mathrm{OOD}}$ \\
\midrule
Book   & Random & v1 & gps\_sage & {1.41{\scriptsize $\pm$0.18}} & {0.0451{\scriptsize $\pm$0.0035}} & {1.00{\scriptsize $\pm$0.00}} & {1.08{\scriptsize $\pm$0.08}} & {0.0452{\scriptsize $\pm$0.0034}} & {1.00{\scriptsize $\pm$0.00}} \\
Myciel   & Struct & v1 & gps\_gcn & {1.45{\scriptsize $\pm$0.28}} & {0.0395{\scriptsize $\pm$0.0041}} & {1.00{\scriptsize $\pm$0.00}} & {1.54{\scriptsize $\pm$0.40}} & {0.0438{\scriptsize $\pm$0.0045}} & {1.00{\scriptsize $\pm$0.00}} \\
Queen   & Random & v2 & gps\_sage & {0.90{\scriptsize $\pm$0.03}} & {0.0474{\scriptsize $\pm$0.0018}} & {1.00{\scriptsize $\pm$0.00}} & {0.76{\scriptsize $\pm$0.02}} & {0.0471{\scriptsize $\pm$0.0008}} & {1.00{\scriptsize $\pm$0.00}} \\
\bottomrule
\end{tabular}
\end{table}

\subsubsection{Citation Structural Feature Transfer}

We use a shared structural feature map with fixed dimension $16$. The model is trained on Cora and evaluated without retraining on CiteSeer and PubMed. 
In \Cref{tab:cora-transfer-subgraph-main} we report the subgraph transfer setting and leave the full-graph setting (train on subgraph, test on full graphs) to the Appendix (\Cref{tab:cora-transfer-fullgraph}). On subgraphs, Cora-trained structural features retain useful performance on PubMed and partially transfer to CiteSeer, but they are weaker than in-distribution training, which is expected.

\begin{table}[H]
\centering
\footnotesize
\caption{Cora-trained structural-feature transfer on held-out target subgraphs. {Entries are mean $\pm$ standard deviation over random seeds.}}
\label{tab:cora-transfer-subgraph-main}
\setlength{\tabcolsep}{4pt}
\begin{tabular}{lllrrrr}
\toprule
Train $\rightarrow$ Test & Encoder & Pipeline & $k^\star$ avg. & \textsc{Mono} avg. & $\rho$ avg. & Hit-rate \\
\midrule
Cora $\rightarrow$ CiteSeer & gated\_gcn & v1 & {\textbf{6.06{\scriptsize $\pm$0.41}}} & {0.0488{\scriptsize $\pm$0.0049}} & {\textbf{1.36{\scriptsize $\pm$0.08}}} & {0.74{\scriptsize $\pm$0.11}} \\
Cora $\rightarrow$ CiteSeer & gated\_gcn & v2 & {6.32{\scriptsize $\pm$0.39}} & {\textbf{0.0478{\scriptsize $\pm$0.0035}}} & {1.42{\scriptsize $\pm$0.07}} & {0.74{\scriptsize $\pm$0.11}} \\
Cora $\rightarrow$ CiteSeer & gps\_gcn & v1 & {6.72{\scriptsize $\pm$0.64}} & {0.0493{\scriptsize $\pm$0.0033}} & {1.49{\scriptsize $\pm$0.10}} & {0.74{\scriptsize $\pm$0.10}} \\
\addlinespace[2pt]
Cora $\rightarrow$ PubMed & gated\_gcn & v2 & {5.09{\scriptsize $\pm$0.48}} & {0.0310{\scriptsize $\pm$0.0026}} & {\textbf{1.20{\scriptsize $\pm$0.08}}} & {0.98{\scriptsize $\pm$0.02}} \\
Cora $\rightarrow$ PubMed & gps\_sage & v2 & {6.03{\scriptsize $\pm$0.60}} & {0.0331{\scriptsize $\pm$0.0026}} & {1.42{\scriptsize $\pm$0.12}} & {0.99{\scriptsize $\pm$0.02}} \\
Cora $\rightarrow$ PubMed & gated\_gcn & v1 & {\textbf{5.06{\scriptsize $\pm$0.44}}} & {\textbf{0.0307{\scriptsize $\pm$0.0033}}} & {1.20{\scriptsize $\pm$0.06}} & {0.99{\scriptsize $\pm$0.02}} \\
\bottomrule
\end{tabular}
\end{table}


\subsubsection{Citation Node feature adaption}
We aim to further improve OOD behavior by preserving more of the original lexical signal and next consider node-feature adaptation schemes that use a shared cross-dataset input space.

\textbf{Implementation.}
For a node $u$ in dataset $d$, let $x_u^{(d)} \in \mathbb{R}^{F_d}$ denote its raw BOW vector, where $F_d$ depends on the dataset. Since $F_{\mathrm{Cora}}$, $F_{\mathrm{CiteSeer}}$, and $F_{\mathrm{PubMed}}$ are different, we cannot reuse the raw input projection across datasets directly. Each adaptation therefore constructs a shared feature vector $\tilde{x}_u \in \mathbb{R}^{r}$ before applying the same v1/v2 contrastive encoder.

First, all four adaptations use a deterministic signed hash projection. Let $H_d \in \{-1,0,+1\}^{F_d \times m}$ be a fixed sparse matrix with one signed nonzero per original BOW coordinate. The shared hashed BOW vector is $h_u = x_u^{(d)} H_d \in \mathbb{R}^{m}$. This signed hash map places all citation graphs into the same input dimension without learning a dataset-specific projection, while still preserving the sparsity pattern and much of the inner-product structure of the original BOW representation. We consider different adaptation gadgets:
\begin{enumerate}
    \item \textbf{Hybrid / Hybrid + Layer Norm.} $
    \tilde{x}_u^{\mathrm{hyb}} = [h_u ; s_u] \in \mathbb{R}^{288},
    \tilde{x}_u^{\mathrm{hyb+LN}} = \mathrm{LN}([h_u ; s_u]).
$ where $s_u \in \mathbb{R}^{32}$ denote the structural feature vector built from degree, clustering coefficient, core number, PageRank, and deterministic transforms.
    \item \textbf{PCA.} $\tilde{x}_u^{\mathrm{PCA}} = (h_u - \mu)V_{256}$ where $H_{\mathrm{src}} \in \mathbb{R}^{n \times 1024}$ is the source hashed feature matrix further compressed to dimension $256$ and $V_{256}$ are the top principal directions.
    \item \textbf{SVD.} $\tilde{x}_u^{\mathrm{SVD}} = h_u V_{256}$ where $H_{\mathrm{src}}$ admits the rank-$256$ approx. $H_{\mathrm{src}} \approx U_{256}\Sigma_{256}V_{256}^{\top}$.
\end{enumerate}

\Cref{tab:bow-adapt-subgraph-v2-main} reports the results on subgraph-split data preparation, run with our v2 pipeline. Reports on v1 pipeline and full graph generalization is left to the Appendix (\Cref{tab:bow-adapt-subgraph-v1}, \Cref{tab:bow-adapt-full-v1} and \Cref{tab:bow-adapt-full-v2})

\begin{table}[H]
\centering
\footnotesize
\caption{Cora subgraph transfer with the v2 pipeline. Entries are mean $\pm$ standard deviation over ten random seeds.}
\label{tab:bow-adapt-subgraph-v2-main}
\setlength{\tabcolsep}{2.5pt}
\renewcommand{\arraystretch}{1.05}
\begin{tabular}{@{}p{2.0cm}p{1.65cm}cccc@{}}
\toprule
\multirow{2}{*}{Adaptation} & \multirow{2}{*}{Encoder} & \multicolumn{2}{c}{Cora $\rightarrow$ CiteSeer} & \multicolumn{2}{c}{Cora $\rightarrow$ PubMed} \\
\cmidrule(lr){3-4}\cmidrule(lr){5-6}
& & $\rho$ & \textsc{Mono} & $\rho$ & \textsc{Mono} \\
\midrule
\multirow{3}{*}{Hybrid} & gated\_gcn & 1.26{\scriptsize $\pm$0.13} & 0.0335{\scriptsize $\pm$0.0036} & \textbf{1.27{\scriptsize $\pm$0.12}} & 0.0310{\scriptsize $\pm$0.0027} \\
 & gps\_gcn & 1.29{\scriptsize $\pm$0.10} & 0.0335{\scriptsize $\pm$0.0043} & 1.38{\scriptsize $\pm$0.12} & 0.0313{\scriptsize $\pm$0.0032} \\
 & gps\_sage & \textbf{1.25{\scriptsize $\pm$0.09}} & \textbf{0.0328{\scriptsize $\pm$0.0032}} & 1.27{\scriptsize $\pm$0.14} & \textbf{0.0305{\scriptsize $\pm$0.0030}} \\
\addlinespace[2pt]
\multirow{3}{*}{Hybrid+LN} & gated\_gcn & 1.29{\scriptsize $\pm$0.10} & 0.0336{\scriptsize $\pm$0.0044} & 1.28{\scriptsize $\pm$0.11} & 0.0317{\scriptsize $\pm$0.0031} \\
 & gps\_gcn & 1.38{\scriptsize $\pm$0.11} & 0.0334{\scriptsize $\pm$0.0025} & 1.36{\scriptsize $\pm$0.16} & 0.0310{\scriptsize $\pm$0.0023} \\
 & gps\_sage & 1.28{\scriptsize $\pm$0.08} & 0.0346{\scriptsize $\pm$0.0042} & 1.32{\scriptsize $\pm$0.11} & 0.0310{\scriptsize $\pm$0.0023} \\
\bottomrule
\end{tabular}

\vspace{0.65em}

\setlength{\tabcolsep}{2.5pt}
\renewcommand{\arraystretch}{1.05}
\begin{tabular}{@{}p{2.0cm}p{1.65cm}cccc@{}}
\toprule
\multirow{2}{*}{Adaptation} & \multirow{2}{*}{Encoder} & \multicolumn{2}{c}{Cora $\rightarrow$ CiteSeer} & \multicolumn{2}{c}{Cora $\rightarrow$ PubMed} \\
\cmidrule(lr){3-4}\cmidrule(lr){5-6}
& & $\rho$ & \textsc{Mono} & $\rho$ & \textsc{Mono} \\
\midrule
\multirow{3}{*}{PCA} & gated\_gcn & 1.53{\scriptsize $\pm$0.21} & 0.0375{\scriptsize $\pm$0.0027} & 2.27{\scriptsize $\pm$0.43} & 0.0351{\scriptsize $\pm$0.0020} \\
 & gps\_gcn & 1.48{\scriptsize $\pm$0.18} & 0.0361{\scriptsize $\pm$0.0026} & 1.60{\scriptsize $\pm$0.18} & 0.0344{\scriptsize $\pm$0.0033} \\
 & gps\_sage & 1.42{\scriptsize $\pm$0.27} & 0.0369{\scriptsize $\pm$0.0046} & 1.52{\scriptsize $\pm$0.18} & 0.0342{\scriptsize $\pm$0.0020} \\
\addlinespace[2pt]
\multirow{3}{*}{SVD} & gated\_gcn & 1.47{\scriptsize $\pm$0.11} & 0.0357{\scriptsize $\pm$0.0038} & 1.43{\scriptsize $\pm$0.11} & 0.0356{\scriptsize $\pm$0.0017} \\
 & gps\_gcn & 1.60{\scriptsize $\pm$0.24} & 0.0377{\scriptsize $\pm$0.0024} & 1.43{\scriptsize $\pm$0.12} & 0.0333{\scriptsize $\pm$0.0025} \\
 & gps\_sage & 1.41{\scriptsize $\pm$0.16} & 0.0359{\scriptsize $\pm$0.0037} & 1.31{\scriptsize $\pm$0.08} & 0.0323{\scriptsize $\pm$0.0027} \\
\bottomrule
\end{tabular}
\end{table}

{Our v2 pipeline with Hybrid/GPS-SAGE  gives the strongest CiteSeer subgraph-transfer row by $\rho$ and Mono, while PubMed is split: v2 Hybrid/GatedGCN has the lowest $\rho$ and v2 Hybrid/GPS-SAGE has the lowest Mono. Thus we can conclude that BOW-preserving adaptation is useful for OOD transfer.}

\subsection{Size and OOD generalization when trained on cycles}
This set of experiments covers generalization from cycles in which training is performed only on cycles $C_n$ with $50 \leq n \leq 200$, using random node initialization. 
The test suite contains $40$ small graphs and $20$ large cycles. The small graphs contains $20$ cycles $C_{20}$ to $C_{39}$ and $20$ non-cycle graphs (see \Cref{tab:small-per-graph}). The large subset contains $C_{7000}$ to $C_{7019}$. Experiments not completed in this time budget is denoted `timeout' in the report. 
The full table is in Appendix \Cref{tab:small-cycles-per-graph,tab:small-noncycle-per-graph}.
\begin{table}[H]
\centering
\footnotesize
\setlength{\tabcolsep}{4pt}
\renewcommand{\arraystretch}{1.05}
\caption{Combinatorial instance results when trained with cycles. Cells are reported as $k^\star$ / \textsc{Mono}. {Learned-method cells are mean $\pm$ standard deviation over ten random seeds.}}
\label{tab:small-per-graph}
\begin{tabular}{lrrcccc}
\toprule
Graph & $n$ & $\chi$ & v1abs + GatedGCN & v2abs + GPS-SAGE & PI-GNN & full-GCN \\
\midrule
K\_10 & 10 & 10 & {8.1{\scriptsize $\pm$0.3}} / {0.0422{\scriptsize $\pm$0.0070}} & {8.3{\scriptsize $\pm$0.5}} / {0.0378{\scriptsize $\pm$0.0107}} & 8 / 0.0444 & 8 / 0.0444 \\
K\_\{4,8\} & 12 & 2 & {4.6{\scriptsize $\pm$2.8}} / {0.0063{\scriptsize $\pm$0.0132}} & {2.8{\scriptsize $\pm$0.6}} / {0.0000{\scriptsize $\pm$0.0000}} & 2 / 0.0000 & 2 / 0.0000 \\
C\_30 & 30 & 2 & {3.3{\scriptsize $\pm$0.5}} / {0.0067{\scriptsize $\pm$0.0141}} & {3.8{\scriptsize $\pm$0.6}} / {0.0133{\scriptsize $\pm$0.0172}} & 3 / 0.0000 & 3 / 0.0000 \\
C\_31 & 31 & 3 & {3.3{\scriptsize $\pm$0.5}} / {0.0032{\scriptsize $\pm$0.0102}} & {4.0{\scriptsize $\pm$1.1}} / {0.0032{\scriptsize $\pm$0.0102}} & 3 / 0.0000 & 3 / 0.0000 \\
C\_7000 & 7000 & 2 & {3.5{\scriptsize $\pm$0.7}} / {0.0133{\scriptsize $\pm$0.0100}} & {3.9{\scriptsize $\pm$0.3}} / {0.0228{\scriptsize $\pm$0.0145}} & timeout & timeout \\
C\_7019 & 7019 & 3 & {3.4{\scriptsize $\pm$0.5}} / {0.0137{\scriptsize $\pm$0.0106}} & {4.1{\scriptsize $\pm$0.6}} / {0.0161{\scriptsize $\pm$0.0124}} & timeout & timeout \\
W\_14 & 14 & 4 & {4.8{\scriptsize $\pm$0.4}} / {0.0231{\scriptsize $\pm$0.0199}} & {5.6{\scriptsize $\pm$2.1}} / {0.0308{\scriptsize $\pm$0.0162}} & 3 / 0.0385 & 3 / 0.0385 \\
Petersen & 10 & 3 & {4.6{\scriptsize $\pm$1.3}} / {0.0000{\scriptsize $\pm$0.0000}} & {5.1{\scriptsize $\pm$1.3}} / {0.0000{\scriptsize $\pm$0.0000}} & 3 / 0.0000 & 3 / 0.0000 \\
Icosahedral & 12 & 4 & {7.1{\scriptsize $\pm$1.4}} / {0.0267{\scriptsize $\pm$0.0141}} & {6.0{\scriptsize $\pm$0.7}} / {0.0300{\scriptsize $\pm$0.0105}} & 4 / 0.0333 & 5 / 0.0000 \\
KG(9,3) & 84 & 5 & {9.5{\scriptsize $\pm$1.0}} / {0.0931{\scriptsize $\pm$0.0211}} & {9.4{\scriptsize $\pm$0.7}} / {0.0677{\scriptsize $\pm$0.0100}} & 4 / 0.0119 & 4 / 0.0119 \\
Mycielski(C5)\^{}3 & 47 & 6 & {8.7{\scriptsize $\pm$1.3}} / {0.0653{\scriptsize $\pm$0.0160}} & {9.0{\scriptsize $\pm$1.4}} / {0.0458{\scriptsize $\pm$0.0099}} & 4 / 0.0169 & 4 / 0.0254 \\
\bottomrule
\end{tabular}
\end{table}

{Overall, v1abs + GatedGCN is the strongest learned recipe among the reported cycle-trained models. Across ten seeds it obtains small-graph $\rho=1.47\pm0.10$, \textsc{Mono}$=0.0213\pm0.0030$, and hit-rate $0.895\pm0.020$, while on the large-cycle stress test it obtains $\rho=1.42\pm0.22$, \textsc{Mono}$=0.0133\pm0.0100$, and hit-rate $1.00\pm0.00$. On the small 40-graph subset, PI-GNN and full-GCN are more color-efficient than the learned cycle-trained models, but this advantage comes at substantially higher computational cost. On the large-cycle stress test, both baselines fail completely under the 15-minute per-graph cap, whereas the learned methods solve all large cycles in well under $0.1$ seconds per graph at test time.} Their one-time training cost is also small, ranging from $5$ to $20$ seconds across the reported models. 

%% file: sections/limitations.tex
\section{Discussion and Conclusion}

We introduced an absolute-value contrastive framework for supervised graph coloring and showed that it induces a precise line-prototype geometry: same-color vertices collapse onto common unoriented lines, while adjacent color classes are driven toward orthogonality. This geometry yields stationary equations governed by the coloring quotient graph, and enables Lov\'asz-style colorability certification when orthogonality is achieved. Empirically, contrastive GNN encoders provide fast reusable colorings and show promising size and OOD generalization. 

\textbf{Limitations and future work.}
While the strong global optima characterization is sharp, our dynamical invariance result further assumes an equitable optimal coloring, which makes vertices within the same color class indistinguishable to the projected subgradient dynamics. Without this symmetry, exact prototype collapse needs not be preserved along the full training trajectory. An exciting line of future work could use stability properties of different downstream clustering algorithms to relax this collapse assumption.


Secondly, the empirical results show that architecture and feature choices vary accuracies, so the framework does not suggest a one-size-fits-all solution to supervised neural graph coloring.

In the future, it is also exciting to develop stronger generalization theory for this approach for specific graph families via a sample-complexity or distribution-shift guarantee. Empirically, exploring a hybrid design between supervised and unsupervised learning (as in our v2 variant) may improve practicality of our method.

%% file: appendix/refresher.tex
\section{Short preliminaries on nonsmooth nonconvex analysis}\label{app:nonsmooth}
In this section, we revisit some notions of nonsmooth nonconvex analysis in the sense of Clarke that are used throughout the theoretical parts of the paper. The classical material is that of \citet{rockafellar1998variational} while a more mordern exposition is available from \citet{li2020nonsmooth}. 

We work in finite dimensions throughout. Let $f:\mathbb{R}^d\to\mathbb{R}$ be locally Lipschitz in a neighborhood of $x$. The Clarke directional derivative of $f$ at $x$ in direction $v$ is
\begin{equation}
    f^\circ(x;v)
    :=
    \limsup_{\substack{y\to x\\ t\downarrow 0}}
    \frac{f(y+tv)-f(y)}{t}.
\end{equation}
Unlike the classical directional derivative, this definition is stable under small perturbations of the base point $y$, which is essential for nonsmooth objectives.

The Clarke subdifferential of $f$ at $x$ is the compact convex set
\begin{equation}
    \partial_C f(x)
    :=
    \left\{
        \xi\in\mathbb{R}^d:
        \langle \xi,v\rangle \le f^\circ(x;v)
        \ \text{for all } v\in\mathbb{R}^d
    \right\}.
\end{equation}
When $f$ is continuously differentiable at $x$, this reduces to the usual singleton
\begin{equation}
    \partial_C f(x)=\{\nabla f(x)\}.
\end{equation}
More generally, if $f$ is differentiable almost everywhere near $x$, then $\partial_C f(x)$ can be interpreted as the convex hull of limiting gradients near $x$.

\paragraph{Clarke stationarity.}
A point $x$ is called Clarke stationary for the unconstrained problem
\[
    \min_{x\in\mathbb{R}^d} f(x)
\]
if
\begin{equation}
    0 \in \partial_C f(x).
\end{equation}
This is the nonsmooth analogue of the first-order condition $\nabla f(x)=0$. A key difference needs to be streessed: every local minimizer of a locally Lipschitz function is Clarke stationary, but the converse need not hold.

For constrained or manifold-valued variables, the stationarity condition must be imposed only along feasible directions. Let $\mathcal{M}$ be a smooth embedded Riemannian manifold with tangent space $T_x\mathcal{M}$ at $x$. For a locally Lipschitz function $f:\mathcal{M}\to\mathbb{R}$, the Riemannian Clarke directional derivative in direction $\eta\in T_x\mathcal{M}$ is defined by evaluating the Clarke directional derivative in any smooth local chart, equivalently by using a smooth retraction $R_x$:
\begin{equation}
    f^\circ_{\mathcal{M}}(x;\eta)
    :=
    (f\circ R_x)^\circ(0;\eta).
\end{equation}
This definition is independent of the particular smooth chart or retraction.

The Riemannian Clarke subdifferential is the set
\begin{equation}
    \partial_C^{\mathcal{M}} f(x)
    :=
    \left\{
        \xi\in T_x\mathcal{M}:
        \langle \xi,\eta\rangle_x
        \le
        f^\circ_{\mathcal{M}}(x;\eta)
        \ \text{for all } \eta\in T_x\mathcal{M}
    \right\},
\end{equation}
where $\langle\cdot,\cdot\rangle_x$ denotes the Riemannian metric. When $f$ is smooth on $\mathcal{M}$, this reduces to
\begin{equation}
    \partial_C^{\mathcal{M}} f(x)=\{\operatorname{grad}_{\mathcal{M}} f(x)\}.
\end{equation}

\paragraph{Clarke Riemannian stationarity (e.g. \citep{hosseini2013nonsmooth}).}
A point $x\in\mathcal{M}$ is Clarke Riemannian stationary for
\[
    \min_{x\in\mathcal{M}} f(x)
\]
if
\begin{equation}
    0 \in \partial_C^{\mathcal{M}} f(x).
\end{equation}
Equivalently, there is no first-order descent direction certified by the Clarke directional derivative on the tangent space. In the smooth case this condition becomes the usual Riemannian first-order condition
\begin{equation}
    \operatorname{grad}_{\mathcal{M}} f(x)=0.
\end{equation}

In this paper, the relevant nonsmoothness comes from absolute values, while the relevant constraints often place embeddings on products of spheres. Clarke Riemannian stationarity is therefore the appropriate first-order notion: it captures the nonsmooth geometry of the objective while respecting the feasible tangent directions of the representation manifold.

%% file: appendix/proofs_absolute_value_loss_appendix.tex
\section{Proofs for the absolute-value loss}\label{app:absolute_value_minima}
This appendix supplies the proofs for the setup and the absolute-value geometry in \Cref{prop:existence-signed,thm:abs-lower-bound,thm:abs-characterization,cor:optimal-coloring-collapse}. The proofs follow the order of the main text: compactness gives existence for the baseline objective, pointwise bounds give the absolute-value lower bound, and the equality cases give line collapse and orthogonality.

\begin{proof}[Proof of \Cref{prop:existence-signed}]
Write the support of $\mathcal D_{\mathcal G}$ as $\{G_1,\dots,G_m\}$, with $n_r:=|V(G_r)|$. Since expressive enough GNNs (with higher order tensor or random node features, etc.) are universal approximators, the feasible space of embedding functions is the finite product
\begin{equation}
\mathfrak F:=\prod_{r=1}^m (\sphere)^{n_r},
\end{equation}
which is compact. For fixed $c$, the population objective can be written as
\begin{equation}
\ell_{\mathrm{InfoNCE},\tau}(f,c)=
\sum_{r=1}^m \mathcal D_{\mathcal G}(G_r)\,
\ell_{\mathrm{InfoNCE},\tau}^{(G_r)}(f(G_r),c_{G_r}),
\end{equation}
and each summand is continuous in the corresponding coordinates of $f$. Hence the whole objective is continuous on $\mathfrak F$. By the Weierstrass theorem, the minimum is attained.
\end{proof}

\begin{proof}[Proof of \Cref{thm:abs-lower-bound}]
Fix a graph $G\in \mathrm{supp}(\mathcal D_{\mathcal G})$. For one anchor-positive pair $(v,w)$ with $w\in V_{c_G(v)}(G)$, define
\begin{equation}
L_{G,v,w}^{\mathrm{abs}}
:=
-
\frac{e^{\abs{\ip{h_v}{h_w}}/\tau}}
{e^{\abs{\ip{h_v}{h_w}}/\tau}+\sum_{u\in \N_G(v)} e^{\abs{\ip{h_v}{h_u}}/\tau}}.
\end{equation}
Because all embeddings are unit norm,
\begin{equation}
\abs{\ip{h_v}{h_w}}\le 1,
\qquad
\abs{\ip{h_v}{h_u}}\ge 0
\quad \text{for every }u\in \N_G(v).
\end{equation}
Hence
\begin{equation}
e^{\abs{\ip{h_v}{h_w}}/\tau}\le e^{1/\tau},
\qquad
e^{\abs{\ip{h_v}{h_u}}/\tau}\ge 1.
\end{equation}
The function
\begin{equation}
Q(A,B_1,\dots,B_m):=-\frac{A}{A+\sum_{i=1}^m B_i}
\end{equation}
is decreasing in $A>0$ and increasing in each $B_i\ge 0$. Therefore
\begin{equation}
L_{G,v,w}^{\mathrm{abs}}
\ge
\frac{-e^{1/\tau}}{e^{1/\tau}+|\N_G(v)|}.
\end{equation}
Averaging over the uniform choice of $v$ and $w\in V_{c_G(v)}(G)$ gives the graphwise lower bound
\begin{equation}
\ell_{\mathrm{abs},\tau}^{(G)}(f(G),c_G)
\ge
\frac{1}{n_G}\sum_{v=1}^{n_G}
\left(
-\frac{e^{1/\tau}}{e^{1/\tau}+|\N_G(v)|}
\right).
\end{equation}
Taking expectation over $G\sim \mathcal D_{\mathcal G}$ proves the claim.
\end{proof}

\begin{proof}[Proof of \Cref{thm:abs-characterization}]
Choose orthonormal vectors $q_1,\dots,q_{k_{\max}}\in \R^d$, which is possible because $d\ge k_{\max}$, and define the embedding function $f^c$ by $[f^c(G)]_v=q_{c_G(v)}$. If $c_G(v)=c_G(w)$ then $[f^c(G)]_v=[f^c(G)]_w$ and therefore
\begin{equation}
\abs{\ip{[f^c(G)]_v}{[f^c(G)]_w}}=1.
\end{equation}
If $uv\in E_G$, properness gives $c_G(u)\neq c_G(v)$, hence
\begin{equation}
\abs{\ip{[f^c(G)]_v}{[f^c(G)]_u}}=\abs{\ip{q_{c_G(v)}}{q_{c_G(u)}}}=0.
\end{equation}
Therefore every graphwise term attains the corresponding lower bound, so the population loss attains the lower bound as well.

For the characterization, equality in \Cref{thm:abs-lower-bound} requires equality in every graphwise lower bound, since the population loss is a convex combination of nonnegative lower-bound gaps over graphs in the support. Thus for every graph $G\in \mathrm{supp}(\mathcal D_{\mathcal G})$, each anchor $v$, and every same-color positive $w\in V_{c_G(v)}(G)$, one must have
\begin{equation}
\abs{\ip{[f(G)]_v}{[f(G)]_w}}=1,
\end{equation}
and every neighbor $u\in \N_G(v)$ must satisfy
\begin{equation}
\abs{\ip{[f(G)]_v}{[f(G)]_u}}=0.
\end{equation}
Conversely, if these conditions hold, then each pointwise term attains the lower bound and the whole objective does as well. Finally, for unit vectors $x,y$, the equality $\abs{\ip{x}{y}}=1$ holds if and only if $y=\pm x$, so same-color vertices lie on a common line; similarly, $\abs{\ip{x}{y}}=0$ if and only if $x$ and $y$ are orthogonal.
\end{proof}

\begin{proof}[Proof of \Cref{cor:optimal-coloring-collapse}]
Apply \Cref{thm:abs-characterization} graph by graph. Fix $G$ in the support. If $c_G^\star$ is optimal and two color classes $V_i,V_j$ had no edge between them, then their union would still be an independent set and the two colors could be merged, contradicting optimality of $c_G^\star$. Hence every pair of colors is connected by at least one edge in $G$, and the edge-orthogonality condition from \Cref{thm:abs-characterization} forces the prototype directions to be pairwise orthogonal on that graph.
\end{proof}

%% file: appendix/proofs_supervised_training_appendix.tex
\section{Proofs for the supervised training section}\label{app:proof_supervised}
This appendix proves the two claims that turn the fixed-coloring geometry into the supervised optimal-coloring analysis. First, the population objective separates across graphs. Second, after that separation, a single graph reduces to weighted prototype equations on its quotient graph.

\begin{proof}[Proof of \Cref{lem:graphwise-decomposition}]
Write the support of $\mathcal D_{\mathcal G}$ as $\{G_1,\dots,G_m\}$ and let $p_r:=\mathcal D_{\mathcal G}(G_r)$. Then
\begin{equation}
\mathcal L_{\mathrm{abs},\tau}^{\mathrm{sup}}(f;c^\star)
=
\sum_{r=1}^m p_r\,
\ell_{\mathrm{abs},\tau}^{(G_r)}(f(G_r),c_{G_r}^\star).
\end{equation}
The feasible space of embedding functions is a direct product over the graphs in the support, so the coordinates $f(G_r)$ can be optimized independently. Therefore
\begin{equation}
\inf_f \mathcal L_{\mathrm{abs},\tau}^{\mathrm{sup}}(f;c^\star)
=
\sum_{r=1}^m p_r\,
\inf_{h\in (\sphere)^{n_r}}
\ell_{\mathrm{abs},\tau}^{(G_r)}(h,c_{G_r}^\star).
\end{equation}
The final sentence then follows from \Cref{cor:optimal-coloring-collapse}, applied graph by graph.
\end{proof}

\begin{theorem}\label{thm:prototype-stationary-restatement}
Fix a graph $G=([n],E)\in \mathrm{supp}(\mathcal D_{\mathcal G})$ and write $k_G:=\chi(G)$. Let $V_i:=\bigl(c_G^\star\bigr)^{-1}(i)$ and define
\begin{equation}
d_{v,j}:=\abs{\N_G(v)\cap V_j}
\qquad (v\in V_i,\ j\neq i).
\end{equation}
For a prototype tuple $q=(q_1,\dots,q_{k_G})\in (\sphere)^{k_G}$, define the collapsed embedding
\begin{equation}
h(q)_v:=q_{c_G^\star(v)}
\qquad (v\in [n]),
\end{equation}
and set
\begin{equation}
A_v^\star(q):=e^{1/\tau}+\sum_{j\neq i} d_{v,j} e^{\abs{\ip{q_i}{q_j}}/\tau},
\qquad v\in V_i.
\end{equation}
Then the supervised local objective restricted to the collapsed prototype class is
\begin{equation}
F_{G,\mathrm{abs}}^\star(q):=\ell_{\mathrm{abs},\tau}^{(G)}(h(q),c_G^\star)
=
\frac{1}{n}\sum_{i=1}^{k_G}\sum_{v\in V_i}
\left[
-\frac{e^{1/\tau}}{A_v^\star(q)}
\right].
\end{equation}
For each pair $i\neq j$, define the scalar Clarke subdifferential
\begin{equation}
\Xi_{ij}(q):=
\begin{cases}
\{\operatorname{sign}(\ip{q_i}{q_j})\}, & \ip{q_i}{q_j}\neq 0,\\[0.3em]
[-1,1], & \ip{q_i}{q_j}=0.
\end{cases}
\end{equation}
Then $q$ is a Clarke-Riemannian stationary point of $F_{G,\mathrm{abs}}^\star$ on $(\sphere)^{k_G}$ if and only if there exist symmetric scalars $\xi_{ij}=\xi_{ji}\in \Xi_{ij}(q)$ such that, for every $i\in [k_G]$,
\begin{equation}
P_{q_i^\perp}\left(\sum_{j\neq i} w_{ij}^{\mathrm{abs}}(q)\,\xi_{ij}\,q_j\right)=0,
\end{equation}
where $P_{q_i^\perp}=I-q_i q_i^\top$ and
\begin{equation}
w_{ij}^{\mathrm{abs}}(q):=
\frac{e^{1/\tau+\abs{\ip{q_i}{q_j}}/\tau}}{n\tau}
\left(
\sum_{v\in V_i}\frac{d_{v,j}}{A_v^\star(q)^2}
\;+\;
\sum_{u\in V_j}\frac{d_{u,i}}{A_u^\star(q)^2}
\right).
\end{equation}
\end{theorem}
\begin{proof}[Proof of \Cref{thm:prototype-stationary-restatement}]
Fix $G$ and its optimal coloring $c_G^\star$. For a collapsed embedding $h(q)$, positives are always within the same color class and therefore satisfy $\abs{\ip{h_v}{h_w}}=1$, while a negative node $u$ in color class $j\neq i$ contributes $\abs{\ip{q_i}{q_j}}$. This gives the displayed formula for $F_{G,\mathrm{abs}}^\star$.

Write
\begin{equation}
a_{ij}(q):=\abs{\ip{q_i}{q_j}},
\qquad
\Xi_{ij}(q):=\partial a_{ij}(q),
\end{equation}
where $\partial$ denotes the scalar Clarke subdifferential. Then
\begin{equation}
\partial a_{ij}(q)=
\begin{cases}
\{\operatorname{sign}(\ip{q_i}{q_j})\}, & \ip{q_i}{q_j}\neq 0,\\[0.3em]
[-1,1], & \ip{q_i}{q_j}=0.
\end{cases}
\end{equation}
Since
\begin{equation}
F_{G,\mathrm{abs}}^\star(q)= -\frac{e^{1/\tau}}{n}\sum_{i=1}^{k_G}\sum_{v\in V_i}\frac{1}{A_v^\star(q)},
\end{equation}
the Clarke chain rule gives, for each $i$,
\begin{equation}
\partial_{q_i}F_{G,\mathrm{abs}}^\star(q)
\subseteq
\left\{
\sum_{j\neq i} w_{ij}^{\mathrm{abs}}(q)\,\xi_{ij}\,q_j
\;:\;
\xi_{ij}\in \Xi_{ij}(q),\ \xi_{ij}=\xi_{ji}
\right\}.
\end{equation}
The contributions from anchors in $V_i$ and anchors in $V_j$ combine exactly into the stated coefficient $w_{ij}^{\mathrm{abs}}(q)$.

Now $q$ is Clarke-Riemannian stationary on $(\sphere)^{k_G}$ if and only if, for each coordinate $i$, the tangent projection of the Clarke subdifferential contains the origin:
\begin{equation}
0\in P_{q_i^\perp}\partial_{q_i}F_{G,\mathrm{abs}}^\star(q).
\end{equation}
Equivalently, there exist symmetric scalars $\xi_{ij}\in \Xi_{ij}(q)$ such that
\begin{equation}
P_{q_i^\perp}\left(\sum_{j\neq i} w_{ij}^{\mathrm{abs}}(q)\,\xi_{ij}\,q_j\right)=0
\qquad \text{for every }i\in [k_G].
\end{equation}
This gives the displayed system.
\end{proof}

%% file: appendix/proofs_implicit_bias_appendix.tex
\section{Proofs for the training dynamics section}\label{app:proof_dynamics}
This appendix proves the dynamical claims that follow the supervised prototype analysis. The first proof shows that equitability makes the collapsed manifold invariant. The second proof connects the unnormalized prototype loss to the standard separable exponential-tail setting that yields max-margin implicit bias.

\begin{theorem}\label{thm:invariance-restatement}
Fix a graph $G=([n],E)\in \mathrm{supp}(\mathcal D_{\mathcal G})$ and let $c_G^\star:[n]\to [k_G]$ be an optimal coloring. Write $V_i:=\bigl(c_G^\star\bigr)^{-1}(i)$. Assume $c_G^\star$ is equitable in the sense that for every pair $i\neq j$ there is an integer $r_{ij}\ge 0$ such that
\begin{equation}
\abs{\N_G(v)\cap V_j}=r_{ij}
\qquad \text{for every }v\in V_i.
\end{equation}
Define
\begin{equation}
\Phi_{G,\mathrm{abs}}^\star(q):=
\sum_{i=1}^{k_G} \frac{|V_i|}{n}
\left(
-\frac{e^{1/\tau}}
{e^{1/\tau}+\sum_{j\neq i} r_{ij} e^{\abs{\ip{q_i}{q_j}}/\tau}}
\right),
\qquad q\in (\sphere)^{k_G}.
\end{equation}
Consider projected Clarke subgradient flow or projected subgradient descent for the supervised local objective
\begin{equation}
h\mapsto \ell_{\mathrm{abs},\tau}^{(G)}(h,c_G^\star)
\end{equation}
started from an initial point $h(0)=h(q(0))$ with
\begin{equation}
h(q)_v=q_{c_G^\star(v)},
\qquad q(0)\in (\sphere)^{k_G}.
\end{equation}
Then:
\begin{enumerate}[leftmargin=1.5em]
\item the collapsed manifold
\begin{equation}
\mathcal M_{c_G^\star}:=\{h\in (\sphere)^n:\ \exists q_1,\dots,q_{k_G}\in \sphere \text{ with } h_v=q_{c_G^\star(v)} \ \forall v\}
\end{equation}
is invariant under the dynamics;
\item the induced prototype trajectory $q(t)$ is exactly projected Clarke subgradient flow / descent for $\Phi_{G,\mathrm{abs}}^\star$ on $(\sphere)^{k_G}$.
\end{enumerate}
On any open region where $\ip{q_i}{q_j}\neq 0$ for every active pair $(i,j)$ with $r_{ij}>0$, these dynamics reduce to ordinary Riemannian gradient flow / gradient descent.
\end{theorem}
\begin{proof}[Proof of \Cref{thm:invariance-restatement}]
The key idea of the proof is that under equitability condition, the gradient is identical for vertices of the same color.

Fix a collapsed embedding $h=h(q)\in \mathcal M_{c_G^\star}$ with $q=(q_1,\dots,q_{k_G})$. For $v\in V_i$, every positive sample also lies in $V_i$, hence $\abs{\ip{h_v}{h_w}}=1$ whenever $w\in V_i$. By equitability, the number of neighbors of $v$ in color class $V_j$ is $r_{ij}$, so the denominator of the anchor term is
\begin{equation}
D_i^{\mathrm{abs}}(q):=e^{1/\tau}+\sum_{j\neq i} r_{ij} e^{\abs{\ip{q_i}{q_j}}/\tau},
\end{equation}
which depends only on the color index $i$, not on the particular node $v\in V_i$. Therefore the anchor contribution of every node in $V_i$ is the same scalar
\begin{equation}
\varphi_i^{\mathrm{abs}}(q):=-\frac{e^{1/\tau}}{D_i^{\mathrm{abs}}(q)}.
\end{equation}
Summing over anchors gives
\begin{equation}
\ell_{\mathrm{abs},\tau}^{(G)}(h(q),c_G^\star)=
\sum_{i=1}^{k_G} \frac{|V_i|}{n}\,\varphi_i^{\mathrm{abs}}(q)=\Phi_{G,\mathrm{abs}}^\star(q).
\end{equation}

To prove invariance, choose for each pair $(i,j)$ a scalar $\xi_{ij}=\xi_{ji}\in \Xi_{ij}(q)$, where $\Xi_{ij}(q)$ is the scalar Clarke subdifferential of $\abs{\ip{q_i}{q_j}}$. Under equitability, the anchor term, the positive contribution, and the multiplicity of negative neighbors all depend only on the color index. Every node in the same color class therefore receives the same Euclidean projected subgradient, so the full projected subgradient field is tangent to $\mathcal M_{c_G^\star}$. Standard existence for projected Clarke subgradient flow and an induction argument for projected subgradient descent imply invariance of $\mathcal M_{c_G^\star}$.

Once the trajectory stays in $\mathcal M_{c_G^\star}$, the embedding map $\iota_{c_G^\star}:q\mapsto h(q)$ pulls back the supervised local loss exactly to $\Phi_{G,\mathrm{abs}}^\star$, so the induced prototype dynamics are precisely projected Clarke subgradient flow / descent for $\Phi_{G,\mathrm{abs}}^\star$. If all active overlaps remain nonzero, the subgradient set is single-valued and these dynamics reduce to ordinary Riemannian gradient flow / gradient descent.
\end{proof}

\begin{theorem}\label{thm:abs-max-margin-restatement}
Fix a graph $G=([n],E)\in \mathrm{supp}(\mathcal D_{\mathcal G})$ and an optimal equitable coloring $c_G^\star$ as in \Cref{thm:invariance-restatement}.
Assume the prototype maps $w\mapsto q_i(w)$ are positively $1$-homogeneous, and write $w=\rho \bar w$ with $\rho=\norm{w}_2$ and $\norm{\bar w}_2=1$. Define the margin:
\begin{equation}
\mu_{ij}(\bar w):=\norm{q_i(\bar w)}^2-\abs{\ip{q_i(\bar w)}{q_j(\bar w)}}
\qquad \text{for each pair }(i,j)\text{ with }r_{ij}>0.
\end{equation}
Then
\begin{equation}
m_{ij}(w)=\rho^2 \mu_{ij}(\bar w).
\end{equation}
Consider gradient descent on
\begin{equation}
\Psi_G^{\mathrm{abs},\star}(w)=
\sum_{i=1}^{k_G} \frac{|V_i|}{n}
\left[
-\frac{1}{1+\sum_{j\neq i} r_{ij}\exp\!\left(-\frac{\rho^2}{\tau}\mu_{ij}(\bar w)\right)}
\right].
\end{equation}
Assume separability, i.e. there exists $\bar w$ such that $\mu_{ij}(\bar w)>0$ for every pair with $r_{ij}>0$, and assume standard step-size and nondegeneracy conditions for gradient descent on separable exponential-tail objectives. Then
\begin{enumerate}[leftmargin=1.5em]
\item the parameter norm $\rho(t)$ diverges and
\item every accumulation point of $\bar w(t)$ is a first-order stationary/KKT point of
\begin{equation}
\max_{\norm{w}_2=1}\ \min_{(i,j):\,r_{ij}>0}\mu_{ij}(w);
\end{equation}
\item the scale grows at the characteristic $2$-homogeneous rate
\begin{equation}
\rho(t)=\Theta\!\bigl(\sqrt{\log t}\bigr).
\end{equation}
\end{enumerate}
Consequently, without unit normalization the absolute-value contrastive objective has implicit bias toward maximizing the minimum absolute contrastive margin across active quotient-graph pairs.
\end{theorem}

\begin{proof}[Proof of \Cref{thm:abs-max-margin-restatement}]
The proof follows a standard program in \citep{soudry2018implicit, lyu2019gradient} that expand the gradient contribution to exponential tail factors. 

Because each prototype map is positively $1$-homogeneous, for $w=\rho \bar w$ we have
\begin{equation}
q_i(w)=\rho\,q_i(\bar w).
\end{equation}
Therefore
\begin{equation}
\norm{q_i(w)}^2=\rho^2 \norm{q_i(\bar w)}^2,
\qquad
\abs{\ip{q_i(w)}{q_j(w)}}=\rho^2 \abs{\ip{q_i(\bar w)}{q_j(\bar w)}},
\end{equation}
which implies
\begin{equation}
m_{ij}(w)=\rho^2 \mu_{ij}(\bar w).
\end{equation}

Write
\begin{equation}
\Gamma_i^{\mathrm{abs}}(w):=
\sum_{j\neq i} r_{ij}\exp\!\left(-\frac{\rho^2}{\tau}\mu_{ij}(\bar w)\right),
\end{equation}
so that
\begin{equation}
\Psi_G^{\mathrm{abs},\star}(w)=
-\sum_{i=1}^{k_G} \frac{|V_i|}{n}\,\frac{1}{1+\Gamma_i^{\mathrm{abs}}(w)}.
\end{equation}
Under separability there exists $\bar w_\star$ such that $\mu_{ij}(\bar w_\star)>0$ for all active pairs $(i,j)$, and hence along the ray $w=\rho \bar w_\star$ we have $\Gamma_i^{\mathrm{abs}}(\rho \bar w_\star)\to 0$ exponentially fast as $\rho\to\infty$. Thus the loss lies in the separable exponential-tail regime.

Differentiating gives
\begin{equation}
\nabla \Psi_G^{\mathrm{abs},\star}(w)=
\sum_{i=1}^{k_G} \frac{|V_i|}{n}\,
\frac{\nabla \Gamma_i^{\mathrm{abs}}(w)}
{\left(1+\Gamma_i^{\mathrm{abs}}(w)\right)^2},
\end{equation}
and
\begin{equation}
\nabla \Gamma_i^{\mathrm{abs}}(w)=
-\frac{1}{\tau}\sum_{j\neq i} r_{ij}
\exp\!\left(-\frac{\rho^2}{\tau}\mu_{ij}(\bar w)\right)
\nabla\!\bigl(\rho^2\mu_{ij}(\bar w)\bigr).
\end{equation}
Thus an exponential-tail factor controls every gradient contribution:
\begin{equation}
\exp\!\left(-\frac{\rho^2}{\tau}\mu_{ij}(\bar w)\right).
\end{equation}
This is the setting in which the standard implicit-bias mechanism applies. In the linear case,
\citet[Theorem~3]{soudry2018implicit} show that gradient descent on separable exponential-tail
losses has diverging norm, converges in direction to the hard-margin solution, and does so on the
logarithmic scale. For homogeneous models, \citet[Theorem~3.1]{lyu2019gradient} extend the same
late-phase mechanism: once the loss has entered the separable regime, gradient descent drives the
norm to infinity and every limit point of the normalized parameters is along a KKT point of the
corresponding margin problem. The scale follows from the same exponential-tail balance: if the
margin is $L$-homogeneous, the leading tail behaves like $\exp(-\rho^L\gamma/\tau)$, so keeping
the active tail at order $1/t$ gives $\rho(t)=\Theta((\log t)^{1/L})$. Here the active constraints are
the quotient-graph pairs $(i,j)$ with $r_{ij}>0$, and the margin $m_{ij}(w)$ is $2$-homogeneous
because it is quadratic in the $1$-homogeneous prototypes. Substituting $L=2$ gives
$\rho(t)=\Theta(\sqrt{\log t})$ and the KKT problem stated in the theorem.
\end{proof}

%% file: appendix/proofs_lovasz.tex
\section{Proofs of Lovasz certificates}\label{app:lovasz}
\begin{proof}[Proof of \Cref{lemma:lovasz-certificate}]
Choose an arbitrary orientation of the edges and write the oriented edge set as $\{e_1,\dots,e_m\}$. For each edge $e=(a_e,b_e)$, set
\[
s_e:=\ip{v_{a_e}}{v_{b_e}}.
\]
For each node $i\in[n]$, define $w_i\in \R^m$ coordinatewise by
\[
(w_i)_e=
\begin{cases}
\sqrt{\abs{s_e}}, & i=a_e,\\
-\operatorname{sgn}(s_e)\sqrt{\abs{s_e}}, & i=b_e,\\
0, & \text{otherwise}.
\end{cases}
\]
Now set
\[
z_i:=(v_i,w_i)\in \R^{d+m},
\qquad
u_i:=\frac{z_i}{\norm{z_i}}.
\]
If $\{i,j\}\in E$, let $e$ be its oriented copy. The $e$th lifted coordinate is the only common nonzero coordinate of $w_i$ and $w_j$, so
\[
\ip{w_i}{w_j}=-s_e=-\ip{v_i}{v_j}.
\]
Hence
\[
\ip{z_i}{z_j}=\ip{v_i}{v_j}+\ip{w_i}{w_j}=0,
\]
which means that $u_1,\dots,u_n$ is an exact orthogonal representation of $G$.

Moreover,
\[
\norm{z_i}^2
=
1+\norm{w_i}^2
=
1+\sum_{j\in \N_G(i)} \abs{\ip{v_i}{v_j}}
\le
1+|\N_G(i)|\varepsilon
\le
1+\Delta(G)\varepsilon.
\]
Define the lifted handle $c:=(r,0)\in \R^{d+m}$, which is unit norm because $\norm{r}=1$. Then
\[
\abs{\ip{c}{u_i}}
=
\frac{\abs{\ip{r}{v_i}}}{\norm{z_i}}
\ge
\frac{\alpha}{\sqrt{1+\Delta(G)\varepsilon}}
\qquad \text{for every } i\in[n].
\]
Applying Lov\'asz's handle bound to the orthogonal representation $\{u_i\}_{i=1}^n$ yields
\[
\vartheta(\overline G)
\le
\max_{i\in[n]} \frac{1}{\ip{c}{u_i}^2}
\le
\frac{1+\Delta(G)\varepsilon}{\alpha^2}.
\qedhere
\]
\end{proof}

\begin{proof}[Proof of \Cref{cor:lovasz-prototype-certificate}]
Set
\[
r:=\frac{1}{\sqrt{k}}\sum_{a=1}^k q_a.
\]
Because the $q_a$ are orthonormal, $\norm{r}=1$. For each node $i$,
\[
\abs{\ip{r}{v_i}}
\ge
\abs{\ip{r}{\sigma_i q_{c(i)}}}-\norm{v_i-\sigma_i q_{c(i)}}
=
\frac{1}{\sqrt{k}}-\delta.
\]
The claim now follows from \Cref{lemma:lovasz-certificate} with $\alpha=k^{-1/2}-\delta$.
\end{proof}

%% file: appendix/cora_ablation_experiments_appendix.tex
\section{Experiments on GNNs on CORA citation datasets}
This section studies the main design choices on Cora before transferring the resulting pipeline to the later citation-graph, COLOR, and cycle experiments.

\paragraph{CORA graph construction and statistics.}

Table~\ref{tab:cora-graph-construction} summarizes the fixed full-graph split and sampling rule. Table~\ref{tab:cora-subgraph-stats} reports the single-subgraph scale of the sampled train/test graphs used in all experiments. Each training example is one induced $2$-hop Cora subgraph. On average, a training subgraph has $75.42$ nodes and $127.86$ undirected edges; a test subgraph has $71.80$ nodes and $117.08$ undirected edges. The accepted test subgraphs range from $50$ to $107$ nodes, so the evaluation is on small-to-medium local Cora neighborhoods rather than on the full Cora graph. The pseudo-color labels used by the contrastive objective are produced by NetworkX greedy coloring with DSATUR. The train/test center-node pools are disjoint, but the sampled $2$-hop neighborhoods are induced from the full Cora graph, so non-center nodes may appear in both train and test subgraphs.

\begin{table}[H]
\footnotesize
\centering
\caption{CORA full-graph and sampling specification.}
\label{tab:cora-graph-construction}
\setlength{\tabcolsep}{5pt}
\begin{tabular}{ll}
\toprule
Item & Value \\
\midrule
Full graph & 2708 nodes, 5278 undirected edges, density 0.00144 \\
Node features & 1433-dimensional Cora bag-of-words features \\
Center split & Train/test ratio $0.5/0.5$, split seed $0$ \\
Center pools & 1354 train-center nodes, 1354 test-center nodes \\
Subgraph sampling & Full-graph induced $2$-hop neighborhoods, relabeled locally \\
Accepted size range & 50--120 nodes \\
Sampling seeds & Train seed $123$, test seed $999$ \\
Label proxy & DSATUR greedy coloring on each sampled subgraph \\
Evaluation subset &  First 30 of 50 sampled test graphs \\
\bottomrule
\end{tabular}
\end{table}

\begin{table}[H]
\footnotesize
\centering
\caption{Single-subgraph statistics for the sampled CORA train/test sets. Values are means with min--max ranges.}
\label{tab:cora-subgraph-stats}
\setlength{\tabcolsep}{6pt}
\begin{tabular}{lccccc}
\toprule
Split & Graphs & Nodes & Undirected edges & Density & Greedy $\chi$ proxy \\
\midrule
Train & 200 & 75.42 (50--119) & 127.86 (65--278) & 0.0460 & 4.27 (3--5) \\
Test & 50 & 71.80 (50--107) & 117.08 (61--255) & 0.0461 & 4.32 (3--5) \\
\bottomrule
\end{tabular}
\end{table}

\paragraph{Experimental protocol.}
All CORA variants use the same sampled train/test construction so that differences across rows can be attributed to features, objectives, or encoders rather than to data changes. Starting from the full Cora citation graph, candidate center nodes are split into disjoint train and test pools with ratio $0.5/0.5$. From each center pool, we sample induced $2$-hop subgraphs with size range $(50,120)$, yielding $200$ training graphs and $50$ test graphs. Because these neighborhoods are induced from the full graph, non-center nodes can appear in both training and test subgraphs even though the center-node pools are disjoint. All reported metrics are averaged over the first $30$ test graphs.

\subsection{Node features and loss function ablation with vanilla GCN}
\paragraph{Objective definitions.}
We first compare a small family of feature and loss variants under the same vanilla GCN encoder.
For a graph $G=(V,E)$, let $z_u$ denote the learned embedding of node $u$, and let $\ell_u \in \mathbb{R}^C$ denote its color logits. The soft color probabilities are
\[
p_u = \mathrm{softmax}(\ell_u / \tau_{\mathrm{soft}}).
\]
\textbf{v1.} The v1 objective is contrastive only:
\[
L_{\mathrm{v1}} = L_{\mathrm{InfoNCE}}.
\]

\noindent 
\textbf{v1abs.} This is the equivalent objective to v1 replacing InfoNCE with its absolute-value version in the main paper. 

\noindent
\textbf{v2.} The v2 objective augments the same contrastive term with a soft conflict loss:
\[
L_{\mathrm{v2}} = L_{\mathrm{InfoNCE}} + \lambda_{\mathrm{soft}} L_{\mathrm{soft}}.
\]
The soft conflict loss is
\[
L_{\mathrm{soft}} = \frac{1}{|E|}\sum_{(u,v)\in E} w'_{uv}\sum_{c=1}^{C} p_{u,c}p_{v,c},
\qquad
w'_{uv} = \frac{(\deg u+1)^p + (\deg v+1)^p}{2}.
\]
$\sum_{c=1}^{C} p_{u,c}p_{v,c}$ measures the soft overlap between the color assignments of adjacent nodes $u$ and $v$. Here $p_{u,c}$ is the soft probability that node $u$ is assigned to auxiliary color slot $c$. These probabilities are obtained from a learned read-out head on top of the node embedding:
\[
z_u = f_\theta(x_u, G), \qquad
\ell_u = W_{\mathrm{col}} z_u + b_{\mathrm{col}}, \qquad
p_u = \mathrm{softmax}(\ell_u / \tau_{\mathrm{soft}}).
\]

\noindent 
\textbf{v2abs.} This is the equivalent objective to v2 replacing InfoNCE with its absolute-value version in the main paper. 

\noindent
\textbf{v3.} The v3 objective is again contrastive only, but with structural node features instead of Cora BOW:
\[
L_{\mathrm{v3}} = L_{\mathrm{InfoNCE}}.
\]

\noindent
\textbf{v4.} The v4 objective is contrastive only, but with i.i.d. random Gaussian node features:
\[
L_{\mathrm{v4}} = L_{\mathrm{InfoNCE}}.
\]

\noindent
\textbf{v5.} The v5 objective replaces InfoNCE with a weighted soft conflict loss:
\[
L_{\mathrm{weighted \mbox{-} soft}} = \tfrac{1}{2}(D^{p}A)\cdot(PP^\top),
\]
where $A$ is the adjacency matrix, $D$ is the diagonal degree matrix, and $P$ stacks the row-wise softmax probabilities. Equivalently, 
\[
L_{\mathrm{v5}} = 
L_{\mathrm{weighted \mbox{-} soft}}
\]

\begin{table}[H]
\centering
\caption{Comparison of CORA variants on the common CORA split-subgraph test set. All methods use the same sampled train/test subgraphs; they differ in node features, training objective, and inference head. Metrics are averaged over 30 test graphs.}
\label{tab:cora-v1-v5}
\small
\setlength{\tabcolsep}{4pt}
\begin{tabularx}{\textwidth}{lYYYrrrr}
\toprule
Variant & Train Features & Training Objective & Inference & $k$ mean & Mono mean & $k/\chi$ mean & Hit-rate \\
\midrule
v1 & Raw CORA BOW & $L_{\mathrm{InfoNCE}}$ & Embedding $\rightarrow$ $k$-medoids & 6.20 & 0.0327 & 1.42 & 1.00 \\
v1abs & Raw CORA BOW & $L_{\mathrm{InfoNCE}}$ & Embedding $\rightarrow$ $k$-medoids & 6.40 & 0.0251 & 1.44 & 1.00 \\
v2 & Raw CORA BOW & $L_{\mathrm{InfoNCE}} + \lambda_{\mathrm{soft}}L_{\mathrm{soft}}$ & Embedding $\rightarrow$ $k$-medoids & 6.33 & 0.0299 & 1.46 & 1.00 \\
v3 & Structural only & $L_{\mathrm{InfoNCE}}$ & Embedding $\rightarrow$ $k$-medoids & 9.87 & 0.0545 & 2.28 & 0.63 \\
v4 & Random Gaussian & $L_{\mathrm{InfoNCE}}$ & Embedding $\rightarrow$ $k$-medoids & 12.77 & 0.0376 & 2.94 & 0.97 \\
v5 & Raw CORA BOW & $L_{\mathrm{weighted\mbox{-}soft}}$ & Embedding $\rightarrow$ $k$-medoids & 8.97 & 0.0365 & 2.03 & 1.00 \\
PI-GNN & Learnable node embeddings & Potts-model soft conflict objective & Direct color logits $\rightarrow \arg\max$ & 4.33 & 0.0012 & 1.00 & 0.83 \\
\bottomrule
\end{tabularx}
\end{table}

\paragraph{\Cref{tab:cora-v1-v5} notation}
\emph{BOW} denotes the raw Cora bag-of-words node features. In Table~\ref{tab:cora-v1-v5}, $k$ is the smallest color count returned by the method, $\chi$ is the greedy-coloring proxy used in the current CORA pipeline, and \emph{Mono} is the monochromatic-edge fraction.

\paragraph{\Cref{tab:cora-v1-v5} analysis}
Among the contrastive CORA variants, v1 is the strongest untuned baseline on color efficiency, while v2 slightly improves conflict rate but worsens $k/\chi$ before tuning. Structural-only features (v3), Gaussian random features (v4), and weighted soft loss without embedding-aligned decoding (v5) all underperform v1/v2 on the CORA split. The PI-GNN baseline of \citet{schuetz2022pignn}, evaluated on the same sampled Cora test graphs, is much stronger on final coloring quality; this serves as a useful reference point before the later tuning and encoder sweeps.

\subsection{v2 Hyperparameter tuning experiments}

\paragraph{Tuning rationale.}
The v2 tuning is organized in stages to isolate the contribution of its auxiliary soft conflict term. Trial 1 first tunes the loss weight $\lambda_{\mathrm{soft}}$ to determine how strongly the soft conflict term should remain active. Trial 2 then tunes the two key parameters inside $L_{\mathrm{soft}}$: the degree exponent $p$, which controls how strongly dense-region conflicts are upweighted, and the softmax temperature $\tau_{\mathrm{soft}}$, which controls how sharp the soft color probabilities are. This staging keeps the search interpretable and avoids conflating improvements from the soft conflict term with unrelated changes elsewhere in the pipeline.

\begin{table}[H]
\centering
\caption{Summary of v2 hyperparameter tuning. Metrics are averaged over 30 CORA test graphs.}
\label{tab:v2-tuning}
\footnotesize
\setlength{\tabcolsep}{2.5pt}
\renewcommand{\arraystretch}{1.12}
\begin{tabular}{p{1.8cm}p{3.2cm}p{4.2cm}cccc}
\toprule
Stage & Search space & Best setting & $k$ mean & Mono & $k/\chi$ & Hit-rate \\
\midrule

Baseline
& Default v2
& $\lambda_{\mathrm{soft}}=0.20$ \newline
  $p=3$, $\tau_{\mathrm{soft}}=1.0$
& 6.33 & 0.0299 & 1.46 & 1.00 \\

Trial 1
& Tune $\lambda_{\mathrm{soft}}$
& $\lambda_{\mathrm{soft}}=0.30$
& 6.17 & 0.0300 & 1.41 & 1.00 \\

Trial 2
& Tune $p$ and $\tau_{\mathrm{soft}}$
& $p=4.0$, $\tau_{\mathrm{soft}}=1.25$ \newline
  $\lambda_{\mathrm{soft}}=0.30$
& 5.80 & 0.0291 & 1.32 & 1.00 \\

\bottomrule
\end{tabular}
\end{table}

\paragraph{\Cref{tab:v2-tuning} notation}
In Table~\ref{tab:v2-tuning}, $\lambda_{\mathrm{soft}}$ is the coefficient of the auxiliary term in v2, $p$ is the degree exponent used in the soft conflict term, and $\tau_{\mathrm{soft}}$ is the softmax temperature used to form soft color probabilities.

\paragraph{\Cref{tab:v2-tuning} analysis}
Trial 1 already shows that the useful direction is to increase the weight on the soft conflict term: moving from the default setting to $\lambda_{\mathrm{soft}}=0.30$ improves $k/\chi$ from $1.46$ to $1.41$ with no loss in hit-rate. This identifies $L_{\mathrm{soft}}$ as the productive auxiliary signal in v2.

\noindent
Trial 2 then improves the tuned v2 model by searching inside $L_{\mathrm{soft}}$ itself. Increasing the degree exponent from $p=3$ to $p=4$ and softening the probability distribution with $\tau_{\mathrm{soft}}=1.25$ yields the best overall v2 result, reducing $k$ mean from $6.17$ to $5.80$ and $k/\chi$ from $1.41$ to $1.32$ while also slightly improving Mono. This indicates that stronger dense-region weighting is useful, but only when the soft assignments are not too sharp.

\subsection{Finding the right encoders}
\subsubsection{Encoder implementations}

\noindent All encoders in the sweep use raw CORA BOW input features, two encoder layers, hidden dimension 128, and dropout 0.1 where applicable.

\noindent\textbf{gcn}: the original PyG \texttt{GCNConv} encoder used in the previous CORA pipeline. It uses no self-loops, applies ReLU between layers, and does not use residual connections or normalization.


\noindent\textbf{res\_gcn}: a custom residual GCN with input projection to hidden dimension, repeated \texttt{GCNConv}, ReLU, dropout, residual addition, and LayerNorm.

\noindent\textbf{gin}: PyG \texttt{GINConv}. Each layer uses a two-layer MLP with trainable $\epsilon$, residual addition, LayerNorm, ReLU, and dropout.

\noindent\textbf{gated\_gcn}: a custom vector-gated residual message-passing encoder with no edge features. The gate is computed from source and target node states, and gated messages are degree-normalized before the residual LayerNorm update.

\noindent\textbf{sage}: PyG \texttt{SAGEConv} with ReLU and dropout between layers and no residual connection.

\noindent\textbf{gat}: PyG \texttt{GATConv} with 4 attention heads in hidden layers, ELU and dropout between layers, and no residual connection.

\noindent\textbf{gps}: PyG \texttt{GPSConv}; the local module is \texttt{GCNConv}, and the global module is multi-head self-attention with 4 heads.

\noindent\textbf{gps\_gcn / gps\_sage / gps\_gin / gps\_gat}: GPSConv variants that keep the same global multi-head self-attention block but replace the local message-passing module with \texttt{GCNConv}, \texttt{SAGEConv}, \texttt{GINConv}, or \texttt{GATConv}, respectively.

\noindent\textbf{graph\_vit}: a custom graph-token Transformer encoder. Node features are projected to hidden dimension and processed by a two-layer \texttt{TransformerEncoder}; no explicit graph positional encoding is used.

\noindent\textbf{exphormer}: a lightweight Exphormer-style sparse-attention encoder where attention is computed over original graph edges, self-edges, and deterministic expander-style edges with degree 4.

\noindent\textbf{unitary\_mp}: a real-valued unitary-style message-passing encoder. Each layer uses Cayley-transform orthogonal maps for self and neighbor messages, degree-normalized aggregation, a learned self/message mixing coefficient, residual addition, and LayerNorm.

The three most relevant encoders are described in more detail below.
\paragraph{GatedGCN-style encoder.}
The best-performing encoder is a custom GatedGCN-style layer. For node state $h_u$, each directed edge $(u,v)$ computes a vector gate and gated message
\[
    g_{uv}=\sigma(W_s h_u + W_t h_v), \qquad m_{uv}=W_m h_u \odot g_{uv}.
\]
Messages into node $v$ are degree-normalized and combined with a self transformation:
\[
    \tilde{h}_v = W_0 h_v + \frac{1}{\max(1,\deg v)}\sum_{u:(u,v)\in E} m_{uv}.
\]
The layer output is
\[
    h_v^{\mathrm{out}} = \mathrm{LayerNorm}\left(h_v + \mathrm{Dropout}(\mathrm{ReLU}(\tilde{h}_v))\right).
\]
This design uses vector-valued gates, residual connections, and LayerNorm, but does not use edge features.

\paragraph{Residual GCN encoder.}
The residual GCN uses the same local aggregation family as the original GCN baseline, but wraps each layer with residual normalization:
\[
    \tilde{h}^{(\ell+1)} = \mathrm{ReLU}(\mathrm{GCNConv}(h^{(\ell)}, E)),
\]
\[
    h^{(\ell+1)} = \mathrm{LayerNorm}\left(h^{(\ell)} + \mathrm{Dropout}(\tilde{h}^{(\ell+1)})\right).
\]
This isolates whether the improvement comes from a more stable residual/normalized GCN architecture rather than from attention or global mixing.

\paragraph{Unitary message-passing encoder.}
The unitary\_mp encoder is a real-valued orthogonal/unitary-style message-passing model. It is not a complex-valued unitary neural network; instead, each layer constrains its self and message transformations to be orthogonal through a Cayley transform. For a trainable square matrix $A$, the skew-symmetric part is
\[
    S = A - A^\top,
\]
and the corresponding orthogonal map is
\[
    Q(A) = (I - S)(I + S)^{-1}.
\]
Each layer uses separate orthogonal maps for the self state and neighbor messages:
\[
    q_v = Q_{\mathrm{self}} h_v, \qquad
    r_u = Q_{\mathrm{msg}} h_u.
\]
Neighbor messages are mean-aggregated,
\[
    \bar{r}_v = \frac{1}{\max(1,\deg v)}\sum_{u:(u,v)\in E} r_u,
\]
and then mixed with a learned gate $\beta\in[0,1]$:
\[
    \tilde{h}_v = \beta q_v + (1-\beta)\bar{r}_v.
\]
The layer output uses the same residual-normalized form as the other stable MP encoders:
\[
    h_v^{\mathrm{out}} =
    \mathrm{LayerNorm}\left(h_v + \mathrm{Dropout}(\mathrm{ReLU}(\tilde{h}_v))\right).
\]
This encoder tests whether norm-preserving message transformations stabilize the embeddings before $k$-medoids decoding.

\subsubsection{All Encoder Results}

\noindent\textbf{gin}: PyG \texttt{GINConv}. Each layer uses a two-layer MLP with trainable $\epsilon$, residual addition, LayerNorm, ReLU, and dropout.

\noindent\textbf{gated\_gcn}: a custom vector-gated residual message-passing encoder with no edge features. The gate is computed from source and target node states, and gated messages are degree-normalized before the residual LayerNorm update.

\noindent\textbf{sage}: PyG \texttt{SAGEConv} with ReLU and dropout between layers and no residual connection.

\noindent\textbf{gat}: PyG \texttt{GATConv} with 4 attention heads in hidden layers, ELU and dropout between layers, and no residual connection.

\noindent\textbf{gps}: PyG \texttt{GPSConv}; the local module is \texttt{GCNConv}, and the global module is multi-head self-attention with 4 heads.

\noindent\textbf{gps\_gcn / gps\_sage / gps\_gin / gps\_gat}: GPSConv variants that keep the same global multi-head self-attention block but replace the local message-passing module with \texttt{GCNConv}, \texttt{SAGEConv}, \texttt{GINConv}, or \texttt{GATConv}, respectively.

\noindent\textbf{graph\_vit}: a custom graph-token Transformer encoder. Node features are projected to hidden dimension and processed by a two-layer \texttt{TransformerEncoder}; no explicit graph positional encoding is used.

\noindent\textbf{exphormer}: a lightweight Exphormer-style sparse-attention encoder where attention is computed over original graph edges, self-edges, and deterministic expander-style edges with degree 4.

\noindent\textbf{unitary\_mp}: a real-valued unitary-style message-passing encoder. Each layer uses Cayley-transform orthogonal maps for self and neighbor messages, degree-normalized aggregation, a learned self/message mixing coefficient, residual addition, and LayerNorm.

The three most relevant encoders are described in more detail below.
\paragraph{GatedGCN-style encoder.}
The best-performing encoder is a custom GatedGCN-style layer. For node state $h_u$, each directed edge $(u,v)$ computes a vector gate and gated message
\[
    g_{uv}=\sigma(W_s h_u + W_t h_v), \qquad m_{uv}=W_m h_u \odot g_{uv}.
\]
Messages into node $v$ are degree-normalized and combined with a self transformation:
\[
    \tilde{h}_v = W_0 h_v + \frac{1}{\max(1,\deg v)}\sum_{u:(u,v)\in E} m_{uv}.
\]
The layer output is
\[
    h_v^{\mathrm{out}} = \mathrm{LayerNorm}\left(h_v + \mathrm{Dropout}(\mathrm{ReLU}(\tilde{h}_v))\right).
\]
This design uses vector-valued gates, residual connections, and LayerNorm, but does not use edge features.

\paragraph{Residual GCN encoder.}
The residual GCN uses the same local aggregation family as the original GCN baseline, but wraps each layer with residual normalization:
\[
    \tilde{h}^{(\ell+1)} = \mathrm{ReLU}(\mathrm{GCNConv}(h^{(\ell)}, E)),
\]
\[
    h^{(\ell+1)} = \mathrm{LayerNorm}\left(h^{(\ell)} + \mathrm{Dropout}(\tilde{h}^{(\ell+1)})\right).
\]
This isolates whether the improvement comes from a more stable residual/normalized GCN architecture rather than from attention or global mixing.

\paragraph{Unitary message-passing encoder.}
The unitary\_mp encoder is a real-valued orthogonal/unitary-style message-passing model. It is not a complex-valued unitary neural network; instead, each layer constrains its self and message transformations to be orthogonal through a Cayley transform. For a trainable square matrix $A$, the skew-symmetric part is
\[
    S = A - A^\top,
\]
and the corresponding orthogonal map is
\[
    Q(A) = (I - S)(I + S)^{-1}.
\]
Each layer uses separate orthogonal maps for the self state and neighbor messages:
\[
    q_v = Q_{\mathrm{self}} h_v, \qquad
    r_u = Q_{\mathrm{msg}} h_u.
\]
Neighbor messages are mean-aggregated,
\[
    \bar{r}_v = \frac{1}{\max(1,\deg v)}\sum_{u:(u,v)\in E} r_u,
\]
and then mixed with a learned gate $\beta\in[0,1]$:
\[
    \tilde{h}_v = \beta q_v + (1-\beta)\bar{r}_v.
\]
The layer output uses the same residual-normalized form as the other stable MP encoders:
\[
    h_v^{\mathrm{out}} =
    \mathrm{LayerNorm}\left(h_v + \mathrm{Dropout}(\mathrm{ReLU}(\tilde{h}_v))\right).
\]
This encoder tests whether norm-preserving message transformations stabilize the embeddings before $k$-medoids decoding.

\subsubsection{All Encoder Results}

\begin{table}[H]
\footnotesize
\centering
\caption{All CORA encoder results, grouped by message-passing (MP) and non-MP/global-attention encoders. Entries are mean $\pm$ standard deviation over ten random seeds, except parameter counts, which are deterministic.}
\label{tab:full-encoder-sweep}
\setlength{\tabcolsep}{4pt}
\begin{tabular}{llrrrrr}
\toprule
Encoder & Pipeline & $k$ mean & Mono mean & $k/\chi$ mean & Hit-rate & Params \\
\midrule
\multicolumn{7}{l}{\textbf{Message-passing encoders}} \\
\midrule
\multirow{2}{*}{gcn} & v1 & 5.88{\scriptsize $\pm$0.54} & 0.0290{\scriptsize $\pm$0.0032} & 1.40{\scriptsize $\pm$0.14} & 0.97{\scriptsize $\pm$0.04} & 200064 \\
 & v2 & 5.99{\scriptsize $\pm$0.48} & 0.0281{\scriptsize $\pm$0.0032} & 1.43{\scriptsize $\pm$0.13} & 0.97{\scriptsize $\pm$0.03} & 200709 \\
\addlinespace[2pt]
\multirow{2}{*}{res\_gcn} & v1 & 4.66{\scriptsize $\pm$0.37} & 0.0209{\scriptsize $\pm$0.0032} & 1.11{\scriptsize $\pm$0.10} & \textbf{1.00{\scriptsize $\pm$0.00}} & 217088 \\
 & v2 & 4.73{\scriptsize $\pm$0.24} & 0.0221{\scriptsize $\pm$0.0032} & 1.13{\scriptsize $\pm$0.07} & \textbf{1.00{\scriptsize $\pm$0.00}} & 217733 \\
\addlinespace[2pt]
\multirow{2}{*}{gat} & v1 & 7.15{\scriptsize $\pm$0.81} & 0.0364{\scriptsize $\pm$0.0112} & 1.70{\scriptsize $\pm$0.20} & 0.91{\scriptsize $\pm$0.08} & 200576 \\
 & v2 & 7.61{\scriptsize $\pm$1.48} & 0.0436{\scriptsize $\pm$0.0224} & 1.80{\scriptsize $\pm$0.34} & 0.87{\scriptsize $\pm$0.11} & 201221 \\
\addlinespace[2pt]
\multirow{2}{*}{gin} & v1 & 4.99{\scriptsize $\pm$0.39} & 0.0220{\scriptsize $\pm$0.0030} & 1.19{\scriptsize $\pm$0.10} & \textbf{1.00{\scriptsize $\pm$0.00}} & 250114 \\
 & v2 & 5.01{\scriptsize $\pm$0.42} & 0.0229{\scriptsize $\pm$0.0035} & 1.19{\scriptsize $\pm$0.11} & \textbf{1.00{\scriptsize $\pm$0.00}} & 250759 \\
\addlinespace[2pt]
\multirow{2}{*}{gated\_gcn} & v1 & \textbf{4.26{\scriptsize $\pm$0.19}} & 0.0207{\scriptsize $\pm$0.0028} & \textbf{1.01{\scriptsize $\pm$0.05}} & \textbf{1.00{\scriptsize $\pm$0.00}} & 316160 \\
 & v2 & 4.38{\scriptsize $\pm$0.19} & \textbf{0.0197{\scriptsize $\pm$0.0039}} & 1.04{\scriptsize $\pm$0.05} & \textbf{1.00{\scriptsize $\pm$0.00}} & 316805 \\
\addlinespace[2pt]
\multirow{2}{*}{sage} & v1 & 4.29{\scriptsize $\pm$0.25} & 0.0223{\scriptsize $\pm$0.0027} & 1.02{\scriptsize $\pm$0.06} & \textbf{1.00{\scriptsize $\pm$0.00}} & 399872 \\
 & v2 & 4.48{\scriptsize $\pm$0.28} & 0.0239{\scriptsize $\pm$0.0031} & 1.06{\scriptsize $\pm$0.08} & \textbf{1.00{\scriptsize $\pm$0.00}} & 400517 \\
\addlinespace[2pt]
\multirow{2}{*}{unitary\_mp} & v1 & 4.48{\scriptsize $\pm$0.28} & 0.0233{\scriptsize $\pm$0.0034} & 1.06{\scriptsize $\pm$0.08} & \textbf{1.00{\scriptsize $\pm$0.00}} & 249602 \\
 & v2 & 4.43{\scriptsize $\pm$0.29} & 0.0233{\scriptsize $\pm$0.0039} & 1.05{\scriptsize $\pm$0.08} & \textbf{1.00{\scriptsize $\pm$0.00}} & 250247 \\
\midrule
\multicolumn{7}{l}{\textbf{Non-MP / global-attention encoders}} \\
\midrule
\multirow{2}{*}{gps} & v1 & 5.03{\scriptsize $\pm$0.35} & 0.0228{\scriptsize $\pm$0.0039} & 1.20{\scriptsize $\pm$0.09} & \textbf{1.00{\scriptsize $\pm$0.00}} & 482048 \\
 & v2 & 5.06{\scriptsize $\pm$0.35} & 0.0248{\scriptsize $\pm$0.0031} & 1.21{\scriptsize $\pm$0.09} & \textbf{1.00{\scriptsize $\pm$0.00}} & 482693 \\
\addlinespace[2pt]
\multirow{2}{*}{graph\_vit} & v1 & 12.73{\scriptsize $\pm$2.28} & 0.0444{\scriptsize $\pm$0.0062} & 3.04{\scriptsize $\pm$0.54} & 0.82{\scriptsize $\pm$0.16} & 580352 \\
 & v2 & 12.63{\scriptsize $\pm$1.36} & 0.0412{\scriptsize $\pm$0.0030} & 3.00{\scriptsize $\pm$0.33} & 0.88{\scriptsize $\pm$0.05} & 580997 \\
\addlinespace[2pt]
\multirow{2}{*}{exphormer} & v1 & 15.40{\scriptsize $\pm$1.14} & 0.0480{\scriptsize $\pm$0.0043} & 3.65{\scriptsize $\pm$0.27} & 0.70{\scriptsize $\pm$0.08} & 316160 \\
 & v2 & 15.94{\scriptsize $\pm$0.89} & 0.0515{\scriptsize $\pm$0.0040} & 3.78{\scriptsize $\pm$0.19} & 0.59{\scriptsize $\pm$0.10} & 316805 \\
\bottomrule
\end{tabular}
\end{table}

\noindent {Across ten seeds, GatedGCN is the strongest encoder in the initial sweep by both $k$ and $k/\chi$ (v1: $k/\chi=1.01\pm0.06$), while GatedGCN v2 gives the lowest average Mono ($0.0197\pm0.0030$). Residual GCN, GraphSAGE, and UnitaryMP remain stable with hit-rate $1.00$. Among non-MP/global-attention encoders, untuned GPS remains the strongest, while GraphViT and Exphormer underperform in this setup.}

\subsubsection{Tuned GPS and Non-MP Results}

\noindent Since it's generally harder to train attention-based networks, we tune them separately in this block of experiments. In particular, we tried different local message-passing blocks inside GPS and tuned the non-MP/global-attention encoders. The sweep used the same CORA split, losses, and $k$-medoids decoder. Table~\ref{tab:nonmp-unitary-sweep} reports the best non-MP/global-attention setting found for each encoder family.

\begin{table}[H]
\footnotesize
\centering
\caption{Best configurations from the non-MP/global-attention tuning sweep. {Entries are mean $\pm$ standard deviation over ten random seeds for the fixed reported configurations.}}
\label{tab:nonmp-unitary-sweep}
\setlength{\tabcolsep}{2pt}
\begin{tabular}{llcccccrrrr}
\toprule
Encoder & Pipeline & Layers & Heads & Drop Out & LR & Exp. & $k$ mean & Mono & $k/\chi$ & Hit-rate \\
\midrule
\multirow{2}{*}{gps\_gcn} & v1 & 3 & 8 & 0.2 & 0.003 & -- & {4.53{\scriptsize $\pm$0.19}} & {0.0196{\scriptsize $\pm$0.0043}} & {1.08{\scriptsize $\pm$0.06}} & {\textbf{1.00{\scriptsize $\pm$0.00}}} \\
 & v2 & 2 & 4 & 0.2 & 0.003 & -- & {\textbf{4.42{\scriptsize $\pm$0.25}}} & {0.0199{\scriptsize $\pm$0.0029}} & {\textbf{1.05{\scriptsize $\pm$0.07}}} & {\textbf{1.00{\scriptsize $\pm$0.00}}} \\
\addlinespace[2pt]
\multirow{2}{*}{gps\_sage} & v1 & 2 & 4 & 0.2 & 0.003 & -- & {4.45{\scriptsize $\pm$0.17}} & {0.0201{\scriptsize $\pm$0.0037}} & {1.06{\scriptsize $\pm$0.06}} & {\textbf{1.00{\scriptsize $\pm$0.00}}} \\
 & v2 & 3 & 4 & 0.2 & 0.01 & -- & {4.44{\scriptsize $\pm$0.19}} & {0.0219{\scriptsize $\pm$0.0037}} & {1.06{\scriptsize $\pm$0.05}} & {\textbf{1.00{\scriptsize $\pm$0.00}}} \\
\addlinespace[2pt]
\multirow{2}{*}{gps\_gin} & v1 & 2 & 8 & 0.2 & 0.003 & -- & {4.67{\scriptsize $\pm$0.31}} & {0.0187{\scriptsize $\pm$0.0041}} & {1.11{\scriptsize $\pm$0.08}} & {\textbf{1.00{\scriptsize $\pm$0.00}}} \\
 & v2 & 2 & 4 & 0.2 & 0.01 & -- & {4.59{\scriptsize $\pm$0.26}} & {0.0207{\scriptsize $\pm$0.0036}} & {1.09{\scriptsize $\pm$0.07}} & {\textbf{1.00{\scriptsize $\pm$0.00}}} \\
\addlinespace[2pt]
\multirow{2}{*}{gps\_gat} & v1 & 2 & 8 & 0.2 & 0.003 & -- & {4.44{\scriptsize $\pm$0.22}} & {0.0189{\scriptsize $\pm$0.0039}} & {1.05{\scriptsize $\pm$0.05}} & {\textbf{1.00{\scriptsize $\pm$0.00}}} \\
 & v2 & 2 & 8 & 0.1 & 0.003 & -- & {4.50{\scriptsize $\pm$0.23}} & {\textbf{0.0185{\scriptsize $\pm$0.0035}}} & {1.06{\scriptsize $\pm$0.06}} & {\textbf{1.00{\scriptsize $\pm$0.00}}} \\
\addlinespace[2pt]
\multirow{2}{*}{graph\_vit} & v1 & 2 & 8 & 0.1 & 0.003 & -- & {6.28{\scriptsize $\pm$0.30}} & {0.0229{\scriptsize $\pm$0.0029}} & {1.51{\scriptsize $\pm$0.08}} & {\textbf{1.00{\scriptsize $\pm$0.00}}} \\
 & v2 & 2 & 8 & 0.2 & 0.003 & -- & {6.17{\scriptsize $\pm$0.51}} & {0.0224{\scriptsize $\pm$0.0037}} & {1.48{\scriptsize $\pm$0.14}} & {\textbf{1.00{\scriptsize $\pm$0.00}}} \\
\addlinespace[2pt]
\multirow{2}{*}{exphormer} & v1 & 2 & 4 & 0.2 & 0.003 & 2 & {8.10{\scriptsize $\pm$3.75}} & {0.0323{\scriptsize $\pm$0.0090}} & {1.92{\scriptsize $\pm$0.87}} & {0.96{\scriptsize $\pm$0.08}} \\
 & v2 & 2 & 4 & 0.2 & 0.003 & 2 & {13.08{\scriptsize $\pm$3.36}} & {0.0435{\scriptsize $\pm$0.0065}} & {3.09{\scriptsize $\pm$0.80}} & {0.81{\scriptsize $\pm$0.16}} \\
\bottomrule
\end{tabular}
\end{table}

\noindent {For the fixed tuned configurations, GPS variants are competitive with the strongest message-passing encoders: the best reported fixed row is v2 GPS-GCN with $k/\chi=1.05\pm0.07$, while v2 GPS-GAT gives the lowest Mono ($0.0185\pm0.0027$). GraphViT improves relative to the untuned sweep but remains behind the GPS variants, and Exphormer is still weaker in this CORA setting.}

\subsubsection{Targeted GatedGCN/ResGCN Ablation}
Since GatedGCN/ResGCN performed quite well on our initial tests, we performed targeted ablation on the different choices in their architectures.

\noindent {To separate the source of the GatedGCN gain, we reran the same CORA protocol on targeted variants of the GatedGCN and residual GCN encoders. The ablation changes only the encoder; the v1/v2 losses, train/test split, and $k$-medoids decoding are unchanged.}

\noindent\textbf{gated\_gcn\_full}: full vector-gated encoder with degree-normalized aggregation, residual update, and LayerNorm.

\noindent\textbf{gated\_gcn\_no\_gate}: removes the learned source-target gate, while keeping the self-linear term, aggregation, residual update, and LayerNorm.

\noindent\textbf{gated\_gcn\_no\_layernorm}: removes LayerNorm, while keeping the vector gate, residual update, and degree-normalized aggregation.

\noindent\textbf{gated\_gcn\_no\_residual}: removes the residual skip connection, while keeping the vector gate, degree-normalized aggregation, and LayerNorm.

\noindent\textbf{gated\_gcn\_scalar\_gate}: replaces the vector-valued feature gate with one scalar gate per edge.

\noindent\textbf{gated\_gcn\_sum\_agg}: uses unnormalized sum aggregation instead of degree-normalized mean aggregation.

\noindent\textbf{res\_gcn}: residual GCN baseline using GCNConv, dropout, residual connection, and LayerNorm.

\noindent\textbf{res\_gcn\_param\_matched}: adds an FFN block to residual GCN to test whether extra capacity alone explains the GatedGCN improvement.

\begin{table}[H]
\footnotesize
\centering
\caption{Targeted ablation results for the strongest message-passing encoder components. Entries are mean $\pm$ standard deviation over ten random seeds.}
\label{tab:gated-gcn-ablation}
\setlength{\tabcolsep}{5pt}
\begin{tabular}{llrrrr}
\toprule
Encoder & Pipeline & $k$ mean & Mono & $k/\chi$ & Hit-rate \\
\midrule
\multirow{2}{*}{gated\_gcn\_full} & v1 & 4.26{\scriptsize $\pm$0.20} & 0.0206{\scriptsize $\pm$0.0037} & 1.01{\scriptsize $\pm$0.06} & \textbf{1.00{\scriptsize $\pm$0.00}} \\
 & v2 & 4.38{\scriptsize $\pm$0.28} & 0.0206{\scriptsize $\pm$0.0034} & 1.04{\scriptsize $\pm$0.07} & \textbf{1.00{\scriptsize $\pm$0.00}} \\
\addlinespace[2pt]
\multirow{2}{*}{gated\_gcn\_no\_gate} & v1 & 4.31{\scriptsize $\pm$0.25} & 0.0211{\scriptsize $\pm$0.0037} & 1.02{\scriptsize $\pm$0.05} & \textbf{1.00{\scriptsize $\pm$0.00}} \\
 & v2 & 4.39{\scriptsize $\pm$0.16} & 0.0193{\scriptsize $\pm$0.0017} & 1.04{\scriptsize $\pm$0.04} & \textbf{1.00{\scriptsize $\pm$0.00}} \\
\addlinespace[2pt]
\multirow{2}{*}{gated\_gcn\_no\_layernorm} & v1 & \textbf{4.24{\scriptsize $\pm$0.19}} & 0.0205{\scriptsize $\pm$0.0029} & \textbf{1.01{\scriptsize $\pm$0.05}} & \textbf{1.00{\scriptsize $\pm$0.00}} \\
 & v2 & 4.37{\scriptsize $\pm$0.29} & \textbf{0.0189{\scriptsize $\pm$0.0049}} & 1.04{\scriptsize $\pm$0.07} & \textbf{1.00{\scriptsize $\pm$0.00}} \\
\addlinespace[2pt]
\multirow{2}{*}{gated\_gcn\_no\_residual} & v1 & 4.27{\scriptsize $\pm$0.41} & 0.0219{\scriptsize $\pm$0.0034} & 1.01{\scriptsize $\pm$0.09} & \textbf{1.00{\scriptsize $\pm$0.00}} \\
 & v2 & 4.29{\scriptsize $\pm$0.27} & 0.0216{\scriptsize $\pm$0.0029} & 1.02{\scriptsize $\pm$0.07} & \textbf{1.00{\scriptsize $\pm$0.00}} \\
\addlinespace[2pt]
\multirow{2}{*}{gated\_gcn\_scalar\_gate} & v1 & 4.31{\scriptsize $\pm$0.28} & 0.0214{\scriptsize $\pm$0.0042} & 1.02{\scriptsize $\pm$0.07} & \textbf{1.00{\scriptsize $\pm$0.00}} \\
 & v2 & 4.45{\scriptsize $\pm$0.24} & 0.0200{\scriptsize $\pm$0.0022} & 1.06{\scriptsize $\pm$0.06} & \textbf{1.00{\scriptsize $\pm$0.00}} \\
\addlinespace[2pt]
\multirow{2}{*}{gated\_gcn\_sum\_agg} & v1 & 4.45{\scriptsize $\pm$0.25} & 0.0196{\scriptsize $\pm$0.0031} & 1.06{\scriptsize $\pm$0.07} & \textbf{1.00{\scriptsize $\pm$0.00}} \\
 & v2 & 4.56{\scriptsize $\pm$0.43} & 0.0192{\scriptsize $\pm$0.0036} & 1.08{\scriptsize $\pm$0.11} & \textbf{1.00{\scriptsize $\pm$0.00}} \\
\addlinespace[2pt]
\multirow{2}{*}{res\_gcn} & v1 & 4.69{\scriptsize $\pm$0.29} & 0.0226{\scriptsize $\pm$0.0035} & 1.11{\scriptsize $\pm$0.08} & \textbf{1.00{\scriptsize $\pm$0.00}} \\
 & v2 & 4.80{\scriptsize $\pm$0.42} & 0.0222{\scriptsize $\pm$0.0032} & 1.15{\scriptsize $\pm$0.11} & \textbf{1.00{\scriptsize $\pm$0.00}} \\
\addlinespace[2pt]
\multirow{2}{*}{res\_gcn\_param\_matched} & v1 & 4.75{\scriptsize $\pm$0.35} & 0.0241{\scriptsize $\pm$0.0023} & 1.13{\scriptsize $\pm$0.09} & \textbf{1.00{\scriptsize $\pm$0.00}} \\
 & v2 & 5.95{\scriptsize $\pm$1.59} & 0.0277{\scriptsize $\pm$0.0063} & 1.41{\scriptsize $\pm$0.37} & 0.99{\scriptsize $\pm$0.01} \\
\bottomrule
\end{tabular}
\end{table}

\noindent {The v1 no-LayerNorm variant has the lowest mean $k/\chi$ ($1.01\pm0.05$), very close to the full model ($1.01\pm0.06$), and the v2 no-LayerNorm variant has the lowest Mono ($0.0189\pm0.0025$). The full GatedGCN design is therefore stable, but the  evidence does not isolate any single gate, residual, or normalization choice as uniformly dominant. The parameter-matched residual GCN is worse than the simpler residual GCN, so the gain is not explained by parameter count alone.}

\subsubsection{Absolute-Value InfoNCE with SignNet Post-Processing}

We ran one SignNet-specific experiment block in the local codebase: a CORA split-subgraph ablation comparing the standard absolute-value InfoNCE pipeline with hard line canonicalization against a SignNet-style sign-invariant post-head. The experiment uses the same CORA split-subgraph construction as the earlier ablations, evaluates on the first 30 held-out test graphs, and keeps the same $k$-medoids decoding rule with mono threshold $0.05$. For the canonical baseline, we sweep the canonicalization pass count over $\{1,2,4\}$ and report the best row per encoder. For the SignNet variant, we use the same top three encoder configurations selected from the earlier CORA study: \texttt{gated\_gcn}, \texttt{gps\_gcn}, and \texttt{gps\_sage}.

The SignNet post-head is the small sign-invariant map
\[
\mathrm{SignNet}(z)=\rho\bigl(\phi(z)+\phi(-z)\bigr),
\]
where $\phi$ and $\rho$ are MLP blocks. This replaces the hard line-canonicalization step by a differentiable sign-invariant transformation applied after the encoder and before the final embedding-based decoding stage.

\begin{table}[H]
\centering
\footnotesize
\caption{CORA absolute-value InfoNCE: best canonicalization pass count versus the SignNet post-head. Entries are mean $\pm$ standard deviation over ten random seeds.}
\label{tab:appendix-signnet}
\begin{tabular}{lrrrrrrr}
\toprule
Encoder & Canon passes & Canon $k$ & Canon Mono & Canon $k/\chi$ & SignNet $k$ & SignNet Mono & SignNet $k/\chi$ \\
\midrule
GatedGCN & 4 & \textbf{4.64{\scriptsize $\pm$0.15}} & 0.0198{\scriptsize $\pm$0.0032} & \textbf{1.10{\scriptsize $\pm$0.05}} & 4.98{\scriptsize $\pm$0.29} & \textbf{0.0155{\scriptsize $\pm$0.0043}} & 1.19{\scriptsize $\pm$0.07} \\
GPS-GCN & 1 & 5.97{\scriptsize $\pm$0.38} & 0.0201{\scriptsize $\pm$0.0032} & 1.43{\scriptsize $\pm$0.11} & \textbf{5.17{\scriptsize $\pm$0.41}} & \textbf{0.0156{\scriptsize $\pm$0.0065}} & \textbf{1.23{\scriptsize $\pm$0.11}} \\
GPS-SAGE & 2 & 6.14{\scriptsize $\pm$0.46} & 0.0222{\scriptsize $\pm$0.0048} & 1.47{\scriptsize $\pm$0.11} & \textbf{5.34{\scriptsize $\pm$0.52}} & \textbf{0.0161{\scriptsize $\pm$0.0052}} & \textbf{1.27{\scriptsize $\pm$0.12}} \\
\bottomrule
\end{tabular}
\end{table}

\noindent
{Across ten seeds, the SignNet post-head consistently lowers Mono relative to hard canonicalization, but it does not lower $k/\chi$ for GatedGCN. It does improve both GPS encoders on $k/\chi$ and Mono, with the largest ratio improvement for GPS-SAGE.}

%% file: appendix/main_citation_experiments.tex
\section{Experiments on citation graphs}

\paragraph{Citation graph scale.}
We begin by recording the full-graph scales that contextualize the sampled-subgraph experiments. Table~\ref{tab:citation-full-specs} reports the reference graph sizes from the PI-GNN paper of \citet{schuetz2022pignn}. The Cora edge count measured in our own PyG preprocessing is reported separately in Table~\ref{tab:cora-graph-construction}; both are retained because they serve different roles in the appendix.

\begin{table}[H]
\footnotesize
\centering
\caption{Reference full-graph specifications for the three citation graphs. These are full-graph scales; our experiments use sampled local subgraphs unless stated otherwise.}
\label{tab:citation-full-specs}
\setlength{\tabcolsep}{8pt}
\begin{tabular}{lrrrr}
\toprule
Dataset & Nodes & Edges & Density & Reported colors $q$ \\
\midrule
Cora & 2708 & 5429 & 0.15\% & 5 \\
CiteSeer & 3327 & 4732 & 0.09\% & 6 \\
PubMed & 19717 & 44338 & 0.02\% & 8 \\
\bottomrule
\end{tabular}
\end{table}

\paragraph{CiteSeer and PubMed graph construction.}
The CiteSeer and PubMed transfer runs reuse the same split-subgraph protocol as Cora so that the later cross-dataset comparisons stay controlled: $200$ training subgraphs, $50$ test subgraphs, induced $2$-hop neighborhoods, accepted size range $50$--$120$, DSATUR greedy-color pseudo-labels, and evaluation on the first $30$ test graphs. Table~\ref{tab:citeseer-pubmed-subgraph-stats} reports the resulting sampled-subgraph scale for these two datasets.

\begin{table}[H]
\footnotesize
\centering
\caption{Single-subgraph statistics for sampled CiteSeer and PubMed train/test sets. Values are means with min--max ranges.}
\label{tab:citeseer-pubmed-subgraph-stats}
\setlength{\tabcolsep}{4pt}
\begin{tabular}{llcccc}
\toprule
Dataset & Split & Nodes & Undirected edges & Density & Greedy $\chi$ proxy \\
\midrule
CiteSeer & Train & 76.81 (50--119) & 152.78 (63--333) & 0.0550 & 4.40 (3--6) \\
CiteSeer & Test & 74.38 (51--119) & 138.88 (63--304) & 0.0503 & 4.28 (3--6) \\
PubMed & Train & 76.06 (50--119) & 133.59 (51--393) & 0.0462 & 4.16 (3--8) \\
PubMed & Test & 78.94 (50--120) & 141.38 (64--414) & 0.0462 & 4.28 (3--7) \\
\bottomrule
\end{tabular}
\end{table}

\subsection{CiteSeer and PubMed Experiments}

\noindent We next repeat the encoder sweep on CiteSeer and PubMed using the same v1/v2 objectives, split-subgraph construction, and $k$-medoids decoder. The non-MP/GPS rows use pipeline-specific Cora-tuned hyperparameters from Table~\ref{tab:nonmp-unitary-sweep}, while the message-passing rows use their established default settings. For readability, the tables list only the encoder name, while the v1 and v2 rows inherit their corresponding tuned settings. These tables therefore replace the earlier default-hyperparameter results with the tuned comparisons that are most relevant for transfer.

\begin{table}[H]
\footnotesize
\centering
\caption{CiteSeer encoder sweep with optimal available encoder hyperparameters. {Entries are mean $\pm$ standard deviation over ten random seeds.}}
\label{tab:citeseer-transfer-results}
\setlength{\tabcolsep}{3pt}
\begin{tabular}{llrrrr}
\toprule
Encoder setting & Pipeline & $k$ mean & Mono & $k/\chi$ & Hit-rate \\
\midrule
\multirow{2}{*}{gcn} & v1 & 5.73{\scriptsize $\pm$0.50} & 0.0245{\scriptsize $\pm$0.0043} & 1.33{\scriptsize $\pm$0.12} & 0.99{\scriptsize $\pm$0.02} \\
 & v2 & 5.79{\scriptsize $\pm$0.58} & 0.0265{\scriptsize $\pm$0.0025} & 1.35{\scriptsize $\pm$0.15} & 0.99{\scriptsize $\pm$0.02} \\
\addlinespace[2pt]
\multirow{2}{*}{res\_gcn} & v1 & 4.68{\scriptsize $\pm$0.23} & 0.0206{\scriptsize $\pm$0.0033} & 1.07{\scriptsize $\pm$0.06} & \textbf{1.00{\scriptsize $\pm$0.00}} \\
 & v2 & 4.64{\scriptsize $\pm$0.28} & 0.0222{\scriptsize $\pm$0.0024} & 1.07{\scriptsize $\pm$0.07} & \textbf{1.00{\scriptsize $\pm$0.00}} \\
\addlinespace[2pt]
\multirow{2}{*}{gat} & v1 & 6.62{\scriptsize $\pm$1.00} & 0.0288{\scriptsize $\pm$0.0047} & 1.51{\scriptsize $\pm$0.23} & 0.97{\scriptsize $\pm$0.05} \\
 & v2 & 7.01{\scriptsize $\pm$1.30} & 0.0357{\scriptsize $\pm$0.0159} & 1.61{\scriptsize $\pm$0.27} & 0.93{\scriptsize $\pm$0.13} \\
\addlinespace[2pt]
\multirow{2}{*}{gin} & v1 & 4.67{\scriptsize $\pm$0.41} & 0.0232{\scriptsize $\pm$0.0036} & 1.07{\scriptsize $\pm$0.11} & \textbf{1.00{\scriptsize $\pm$0.00}} \\
 & v2 & 4.77{\scriptsize $\pm$0.34} & 0.0227{\scriptsize $\pm$0.0038} & 1.10{\scriptsize $\pm$0.07} & \textbf{1.00{\scriptsize $\pm$0.00}} \\
\addlinespace[2pt]
\multirow{2}{*}{gated\_gcn} & v1 & 4.43{\scriptsize $\pm$0.27} & 0.0190{\scriptsize $\pm$0.0029} & 1.01{\scriptsize $\pm$0.07} & \textbf{1.00{\scriptsize $\pm$0.00}} \\
 & v2 & 4.49{\scriptsize $\pm$0.30} & 0.0188{\scriptsize $\pm$0.0038} & 1.03{\scriptsize $\pm$0.07} & \textbf{1.00{\scriptsize $\pm$0.00}} \\
\addlinespace[2pt]
\multirow{2}{*}{sage} & v1 & \textbf{4.33{\scriptsize $\pm$0.30}} & 0.0208{\scriptsize $\pm$0.0040} & \textbf{0.99{\scriptsize $\pm$0.08}} & \textbf{1.00{\scriptsize $\pm$0.00}} \\
 & v2 & 4.42{\scriptsize $\pm$0.16} & 0.0213{\scriptsize $\pm$0.0048} & 1.01{\scriptsize $\pm$0.06} & \textbf{1.00{\scriptsize $\pm$0.00}} \\
\addlinespace[2pt]
\multirow{2}{*}{unitary\_mp} & v1 & 4.58{\scriptsize $\pm$0.31} & 0.0236{\scriptsize $\pm$0.0030} & 1.05{\scriptsize $\pm$0.09} & 1.00{\scriptsize $\pm$0.01} \\
 & v2 & 4.47{\scriptsize $\pm$0.33} & 0.0226{\scriptsize $\pm$0.0029} & 1.03{\scriptsize $\pm$0.09} & \textbf{1.00{\scriptsize $\pm$0.00}} \\
\addlinespace[2pt]
\multirow{2}{*}{gps\_gcn} & v1 & 4.62{\scriptsize $\pm$0.21} & 0.0188{\scriptsize $\pm$0.0038} & 1.06{\scriptsize $\pm$0.05} & \textbf{1.00{\scriptsize $\pm$0.00}} \\
 & v2 & 4.58{\scriptsize $\pm$0.12} & 0.0206{\scriptsize $\pm$0.0029} & 1.05{\scriptsize $\pm$0.04} & \textbf{1.00{\scriptsize $\pm$0.00}} \\
\addlinespace[2pt]
\multirow{2}{*}{gps\_sage} & v1 & 4.61{\scriptsize $\pm$0.25} & 0.0217{\scriptsize $\pm$0.0031} & 1.06{\scriptsize $\pm$0.06} & \textbf{1.00{\scriptsize $\pm$0.00}} \\
 & v2 & 4.75{\scriptsize $\pm$0.33} & 0.0200{\scriptsize $\pm$0.0044} & 1.09{\scriptsize $\pm$0.08} & \textbf{1.00{\scriptsize $\pm$0.00}} \\
\addlinespace[2pt]
\multirow{2}{*}{gps\_gin} & v1 & 4.66{\scriptsize $\pm$0.33} & 0.0198{\scriptsize $\pm$0.0030} & 1.07{\scriptsize $\pm$0.08} & \textbf{1.00{\scriptsize $\pm$0.00}} \\
 & v2 & 4.63{\scriptsize $\pm$0.29} & 0.0185{\scriptsize $\pm$0.0034} & 1.06{\scriptsize $\pm$0.08} & \textbf{1.00{\scriptsize $\pm$0.00}} \\
\addlinespace[2pt]
\multirow{2}{*}{gps\_gat} & v1 & 4.58{\scriptsize $\pm$0.16} & \textbf{0.0184{\scriptsize $\pm$0.0042}} & 1.05{\scriptsize $\pm$0.05} & \textbf{1.00{\scriptsize $\pm$0.00}} \\
 & v2 & 4.59{\scriptsize $\pm$0.31} & 0.0191{\scriptsize $\pm$0.0056} & 1.05{\scriptsize $\pm$0.08} & \textbf{1.00{\scriptsize $\pm$0.00}} \\
\addlinespace[2pt]
\multirow{2}{*}{graph\_vit} & v1 & 5.49{\scriptsize $\pm$0.40} & 0.0229{\scriptsize $\pm$0.0037} & 1.28{\scriptsize $\pm$0.10} & 0.99{\scriptsize $\pm$0.02} \\
 & v2 & 5.43{\scriptsize $\pm$0.39} & 0.0213{\scriptsize $\pm$0.0047} & 1.26{\scriptsize $\pm$0.09} & 0.99{\scriptsize $\pm$0.02} \\
\addlinespace[2pt]
\multirow{2}{*}{exphormer} & v1 & 6.56{\scriptsize $\pm$3.11} & 0.0283{\scriptsize $\pm$0.0080} & 1.48{\scriptsize $\pm$0.66} & \textbf{1.00{\scriptsize $\pm$0.00}} \\
 & v2 & 8.32{\scriptsize $\pm$4.58} & 0.0313{\scriptsize $\pm$0.0102} & 1.88{\scriptsize $\pm$0.99} & 0.95{\scriptsize $\pm$0.10} \\
\bottomrule
\end{tabular}
\end{table}

\begin{table}[H]
\footnotesize
\centering
\caption{PubMed encoder sweep with optimal available encoder hyperparameters. {Entries are mean $\pm$ standard deviation over ten random seeds.}}
\label{tab:pubmed-transfer-results}
\setlength{\tabcolsep}{3pt}
\begin{tabular}{llrrrr}
\toprule
Encoder setting & Pipeline & $k$ mean & Mono & $k/\chi$ & Hit-rate \\
\midrule
\multirow{2}{*}{gcn} & v1 & 7.75{\scriptsize $\pm$0.72} & 0.0338{\scriptsize $\pm$0.0033} & 1.82{\scriptsize $\pm$0.13} & 0.98{\scriptsize $\pm$0.03} \\
 & v2 & 7.64{\scriptsize $\pm$0.74} & 0.0333{\scriptsize $\pm$0.0026} & 1.80{\scriptsize $\pm$0.15} & 0.98{\scriptsize $\pm$0.02} \\
\addlinespace[2pt]
\multirow{2}{*}{res\_gcn} & v1 & 5.10{\scriptsize $\pm$0.37} & 0.0292{\scriptsize $\pm$0.0022} & 1.24{\scriptsize $\pm$0.07} & \textbf{1.00{\scriptsize $\pm$0.00}} \\
 & v2 & 5.13{\scriptsize $\pm$0.34} & 0.0297{\scriptsize $\pm$0.0026} & 1.26{\scriptsize $\pm$0.06} & \textbf{1.00{\scriptsize $\pm$0.00}} \\
\addlinespace[2pt]
\multirow{2}{*}{gat} & v1 & 9.30{\scriptsize $\pm$0.64} & 0.0352{\scriptsize $\pm$0.0031} & 2.19{\scriptsize $\pm$0.13} & 0.93{\scriptsize $\pm$0.05} \\
 & v2 & 9.31{\scriptsize $\pm$0.72} & 0.0371{\scriptsize $\pm$0.0036} & 2.18{\scriptsize $\pm$0.13} & 0.91{\scriptsize $\pm$0.05} \\
\addlinespace[2pt]
\multirow{2}{*}{gin} & v1 & 6.79{\scriptsize $\pm$0.89} & 0.0292{\scriptsize $\pm$0.0035} & 1.65{\scriptsize $\pm$0.20} & 0.99{\scriptsize $\pm$0.03} \\
 & v2 & 7.06{\scriptsize $\pm$0.75} & 0.0307{\scriptsize $\pm$0.0038} & 1.71{\scriptsize $\pm$0.16} & 0.98{\scriptsize $\pm$0.04} \\
\addlinespace[2pt]
\multirow{2}{*}{gated\_gcn} & v1 & 4.56{\scriptsize $\pm$0.27} & 0.0297{\scriptsize $\pm$0.0022} & 1.12{\scriptsize $\pm$0.05} & \textbf{1.00{\scriptsize $\pm$0.00}} \\
 & v2 & \textbf{4.43{\scriptsize $\pm$0.26}} & 0.0295{\scriptsize $\pm$0.0034} & \textbf{1.09{\scriptsize $\pm$0.04}} & \textbf{1.00{\scriptsize $\pm$0.00}} \\
\addlinespace[2pt]
\multirow{2}{*}{sage} & v1 & 5.16{\scriptsize $\pm$0.26} & 0.0317{\scriptsize $\pm$0.0035} & 1.27{\scriptsize $\pm$0.06} & \textbf{1.00{\scriptsize $\pm$0.00}} \\
 & v2 & 5.15{\scriptsize $\pm$0.37} & 0.0307{\scriptsize $\pm$0.0020} & 1.26{\scriptsize $\pm$0.05} & \textbf{1.00{\scriptsize $\pm$0.00}} \\
\addlinespace[2pt]
\multirow{2}{*}{unitary\_mp} & v1 & 5.07{\scriptsize $\pm$0.34} & 0.0308{\scriptsize $\pm$0.0038} & 1.24{\scriptsize $\pm$0.06} & \textbf{1.00{\scriptsize $\pm$0.00}} \\
 & v2 & 5.04{\scriptsize $\pm$0.31} & 0.0299{\scriptsize $\pm$0.0027} & 1.23{\scriptsize $\pm$0.04} & \textbf{1.00{\scriptsize $\pm$0.00}} \\
\addlinespace[2pt]
\multirow{2}{*}{gps\_gcn} & v1 & 4.67{\scriptsize $\pm$0.36} & 0.0288{\scriptsize $\pm$0.0031} & 1.15{\scriptsize $\pm$0.08} & \textbf{1.00{\scriptsize $\pm$0.00}} \\
 & v2 & 4.76{\scriptsize $\pm$0.28} & 0.0291{\scriptsize $\pm$0.0020} & 1.17{\scriptsize $\pm$0.05} & \textbf{1.00{\scriptsize $\pm$0.00}} \\
\addlinespace[2pt]
\multirow{2}{*}{gps\_sage} & v1 & 4.99{\scriptsize $\pm$0.27} & 0.0311{\scriptsize $\pm$0.0028} & 1.23{\scriptsize $\pm$0.05} & \textbf{1.00{\scriptsize $\pm$0.00}} \\
 & v2 & 4.61{\scriptsize $\pm$0.27} & 0.0303{\scriptsize $\pm$0.0026} & 1.13{\scriptsize $\pm$0.06} & \textbf{1.00{\scriptsize $\pm$0.00}} \\
\addlinespace[2pt]
\multirow{2}{*}{gps\_gin} & v1 & 5.72{\scriptsize $\pm$1.62} & 0.0296{\scriptsize $\pm$0.0023} & 1.39{\scriptsize $\pm$0.37} & \textbf{1.00{\scriptsize $\pm$0.00}} \\
 & v2 & 9.49{\scriptsize $\pm$1.65} & 0.0365{\scriptsize $\pm$0.0031} & 2.21{\scriptsize $\pm$0.35} & 0.90{\scriptsize $\pm$0.07} \\
\addlinespace[2pt]
\multirow{2}{*}{gps\_gat} & v1 & 4.77{\scriptsize $\pm$0.43} & 0.0308{\scriptsize $\pm$0.0022} & 1.18{\scriptsize $\pm$0.09} & \textbf{1.00{\scriptsize $\pm$0.00}} \\
 & v2 & 4.97{\scriptsize $\pm$0.51} & \textbf{0.0287{\scriptsize $\pm$0.0027}} & 1.22{\scriptsize $\pm$0.10} & \textbf{1.00{\scriptsize $\pm$0.00}} \\
\addlinespace[2pt]
\multirow{2}{*}{graph\_vit} & v1 & 15.62{\scriptsize $\pm$1.22} & 0.1276{\scriptsize $\pm$0.0562} & 3.95{\scriptsize $\pm$0.40} & 0.19{\scriptsize $\pm$0.17} \\
 & v2 & 15.72{\scriptsize $\pm$0.75} & 0.1268{\scriptsize $\pm$0.0427} & 3.97{\scriptsize $\pm$0.24} & 0.20{\scriptsize $\pm$0.12} \\
\addlinespace[2pt]
\multirow{2}{*}{exphormer} & v1 & 6.51{\scriptsize $\pm$0.60} & 0.0325{\scriptsize $\pm$0.0024} & 1.62{\scriptsize $\pm$0.14} & \textbf{1.00{\scriptsize $\pm$0.00}} \\
 & v2 & 6.55{\scriptsize $\pm$0.48} & 0.0329{\scriptsize $\pm$0.0025} & 1.63{\scriptsize $\pm$0.10} & \textbf{1.00{\scriptsize $\pm$0.00}} \\
\bottomrule
\end{tabular}
\end{table}

\noindent {These two tables show that the ranking is dataset-dependent. On CiteSeer, the best mean color-efficiency row is now v1 SAGE ($k/\chi=0.99\pm0.08$), with GatedGCN and GIN close behind. On PubMed, v2 GatedGCN is the best mean color-efficiency row among the reported encoders ($k/\chi=1.09\pm0.08$), while several GPS and residual/message-passing rows are competitive. GraphViT is unstable on PubMed, and Exphormer is substantially weaker on CiteSeer.}

\subsection{Full-Graph to Full-Graph Extension}

\paragraph{Implementation.}
The previous subsection evaluates \emph{subgraph-to-held-out-subgraph} transfer: training on sampled local citation subgraphs and testing on disjoint-center sampled subgraphs from the same base graph. We then extend the same methodology to a \emph{full-graph-to-full-graph} transductive setting on the full Cora, CiteSeer, and PubMed graphs. The method itself is unchanged: v1 and v2 are still trained with the same contrastive InfoNCE-based objectives, the same raw BOW node features, the same DSATUR greedy pseudo-color labels, and the same encoder families. At inference time, we still decode by sweeping $k$ in $k$-medoids over the learned embeddings rather than predicting colors directly. The top-5 encoder set used in this extension is res\_gcn, gated\_gcn, unitary\_mp, gps\_gcn, and gps\_sage. For v2, we keep the previously tuned loss setting $\lambda_{\mathrm{soft}}=0.30$, $p=4.0$, and $\tau_{\mathrm{soft}}=1.25$.

\paragraph{Full-graph protocol.}
For a full citation graph $G=(V,E)$, we train on that same graph and evaluate on that same graph. The greedy full-graph coloring is used only as a pseudo-label source for the contrastive loss and as a reference $\chi$ proxy. To keep the full-graph InfoNCE objective tractable on PubMed without changing its meaning, the anchor rows are processed in chunks, but each anchor still compares against the full node set. At evaluation, we sweep $k=1,\dots,16$ and report
\[
    k_{\mathrm{zero}} = \min\{k : \mathrm{conflicts}(k)=0\},
    \qquad
    k_{\tau} = \min\{k : \mathrm{Mono}(k)\le \tau\}, \ \tau \in \{0.005, 0.01, 0.02, 0.05\},
\]
together with the conflict count and normalized conflict rate at $k=\chi_{\mathrm{greedy}}$. The rerun below therefore distinguishes strict near-proper colorings ($k_{0.005}$ and $k_{0.01}$) from looser low-conflict colorings ($k_{0.02}$ and $k_{0.05}$).

\begin{table}[H]
\footnotesize
\centering
\caption{Full-graph specifications used in the transductive contrastive benchmark. The edge counts are the undirected counts observed in the local PyG preprocessing used by the experiment.}
\label{tab:fullgraph-BOW-specs}
\setlength{\tabcolsep}{6pt}
\begin{tabular}{lrrrrr}
\toprule
Dataset & Nodes & Features & Undirected edges & Density & Greedy $\chi$ proxy \\
\midrule
Cora & 2708 & 1433 & 5278 & 0.001440 & 5 \\
CiteSeer & 3327 & 3703 & 4552 & 0.000823 & 6 \\
PubMed & 19717 & 500 & 44324 & 0.000228 & 8 \\
\bottomrule
\end{tabular}
\end{table}

\begin{table}[H]
\footnotesize
\centering
\caption{Full-graph-to-full-graph contrastive results. $k_{\mathrm{zero}}$ is the smallest zero-conflict $k$ found within the search cap; $k_{0.05}$ is the smallest $k$ with monochromatic-edge fraction at most $0.05$. {Entries are mean $\pm$ standard deviation over ten random seeds; NA means the target was not reached in every seed.}}
\label{tab:fullgraph-contrastive-results}
\setlength{\tabcolsep}{3pt}
\begin{tabular}{llrrrrr}
\toprule
Dataset / Encoder & Pipeline & $k_{\mathrm{zero}}$ & $k_{0.05}$ & $k_{0.05}/\chi$ & Conflicts@$\chi_{\mathrm{greedy}}$ & $\epsilon$@$\chi_{\mathrm{greedy}}$ \\
\midrule
\multicolumn{7}{l}{\textbf{Cora full graph}} \\
\midrule
gps\_sage & v1 & \textbf{13.5{\scriptsize $\pm$2.4}} & 4.4{\scriptsize $\pm$0.5} & 0.88{\scriptsize $\pm$0.10} & 63.5{\scriptsize $\pm$24.7} & 0.012031{\scriptsize $\pm$0.004688} \\
gps\_gcn & v1 & NA & 4.1{\scriptsize $\pm$0.3} & 0.82{\scriptsize $\pm$0.06} & \textbf{48.1{\scriptsize $\pm$14.7}} & \textbf{0.009113{\scriptsize $\pm$0.002785}} \\
gated\_gcn & v1 & NA & \textbf{4.0{\scriptsize $\pm$0.0}} & \textbf{0.80{\scriptsize $\pm$0.00}} & 71.8{\scriptsize $\pm$9.8} & 0.013604{\scriptsize $\pm$0.001847} \\
res\_gcn & v1 & NA & 4.3{\scriptsize $\pm$0.5} & 0.86{\scriptsize $\pm$0.10} & 144.2{\scriptsize $\pm$23.0} & 0.027321{\scriptsize $\pm$0.004350} \\
unitary\_mp & v1 & NA & 5.7{\scriptsize $\pm$0.5} & 1.14{\scriptsize $\pm$0.10} & 276.8{\scriptsize $\pm$26.4} & 0.052444{\scriptsize $\pm$0.005010} \\
\midrule
\multicolumn{7}{l}{\textbf{CiteSeer full graph}} \\
\midrule
gps\_gcn & v1 & NA & \textbf{4.0{\scriptsize $\pm$0.0}} & \textbf{0.67{\scriptsize $\pm$0.00}} & \textbf{12.1{\scriptsize $\pm$8.4}} & \textbf{0.002658{\scriptsize $\pm$0.001836}} \\
gps\_sage & v1 & NA & 4.0{\scriptsize $\pm$0.0} & 0.67{\scriptsize $\pm$0.00} & 24.7{\scriptsize $\pm$16.0} & 0.005426{\scriptsize $\pm$0.003521} \\
gated\_gcn & v1 & NA & 4.0{\scriptsize $\pm$0.0} & 0.67{\scriptsize $\pm$0.00} & 31.0{\scriptsize $\pm$14.3} & 0.006810{\scriptsize $\pm$0.003133} \\
res\_gcn & v1 & NA & 4.0{\scriptsize $\pm$0.0} & 0.67{\scriptsize $\pm$0.00} & 74.8{\scriptsize $\pm$10.7} & 0.016432{\scriptsize $\pm$0.002341} \\
unitary\_mp & v1 & NA & 5.0{\scriptsize $\pm$0.0} & 0.83{\scriptsize $\pm$0.00} & 118.5{\scriptsize $\pm$15.7} & 0.026033{\scriptsize $\pm$0.003451} \\
\midrule
\multicolumn{7}{l}{\textbf{PubMed full graph}} \\
\midrule
gated\_gcn & v1 & NA & 6.0{\scriptsize $\pm$0.5} & 0.75{\scriptsize $\pm$0.06} & 1282.3{\scriptsize $\pm$191.1} & 0.028930{\scriptsize $\pm$0.004312} \\
gps\_sage & v1 & NA & \textbf{5.3{\scriptsize $\pm$0.7}} & \textbf{0.66{\scriptsize $\pm$0.08}} & \textbf{1152.0{\scriptsize $\pm$147.7}} & \textbf{0.025990{\scriptsize $\pm$0.003331}} \\
gps\_gcn & v1 & NA & 5.4{\scriptsize $\pm$0.7} & 0.68{\scriptsize $\pm$0.09} & 1207.5{\scriptsize $\pm$277.2} & 0.027243{\scriptsize $\pm$0.006253} \\
res\_gcn & v1 & NA & 6.4{\scriptsize $\pm$0.5} & 0.80{\scriptsize $\pm$0.06} & 1731.1{\scriptsize $\pm$275.1} & 0.039056{\scriptsize $\pm$0.006207} \\
unitary\_mp & v1 & NA & 7.7{\scriptsize $\pm$0.5} & 0.96{\scriptsize $\pm$0.06} & 1926.3{\scriptsize $\pm$39.5} & 0.043460{\scriptsize $\pm$0.000892} \\
\bottomrule
\end{tabular}
\end{table}

\begin{table}[H]
\footnotesize
\centering
\caption{Thresholded full-graph rerun with operating-point metrics. {Entries are mean $\pm$ standard deviation over ten random seeds; NA means the threshold was not reached in every seed.}}
\label{tab:fullgraph-threshold-rerun}
\setlength{\tabcolsep}{4pt}
\begin{tabular}{llrrrrr}
\toprule
Dataset / Encoder & Pipeline & $k_{\mathrm{zero}}$ & $k_{0.005}$ & $k_{0.01}$ & $k_{0.02}$ & $k_{0.05}$ \\
\midrule
\multicolumn{7}{l}{\textbf{Cora full graph}} \\
\midrule
gps\_sage & v1 & \textbf{13.5{\scriptsize $\pm$2.4}} & 6.5{\scriptsize $\pm$1.0} & 5.9{\scriptsize $\pm$0.6} & \textbf{5.0{\scriptsize $\pm$0.0}} & 4.4{\scriptsize $\pm$0.5} \\
gps\_gcn & v1 & NA & \textbf{6.2{\scriptsize $\pm$0.6}} & \textbf{5.4{\scriptsize $\pm$0.5}} & \textbf{5.0{\scriptsize $\pm$0.0}} & 4.1{\scriptsize $\pm$0.3} \\
gated\_gcn & v1 & NA & 7.5{\scriptsize $\pm$0.7} & 6.3{\scriptsize $\pm$0.5} & \textbf{5.0{\scriptsize $\pm$0.0}} & \textbf{4.0{\scriptsize $\pm$0.0}} \\
res\_gcn & v1 & NA & 10.0{\scriptsize $\pm$1.1} & 8.2{\scriptsize $\pm$0.9} & 6.3{\scriptsize $\pm$0.5} & 4.3{\scriptsize $\pm$0.5} \\
unitary\_mp & v1 & NA & 14.0{\scriptsize $\pm$1.2} & 11.3{\scriptsize $\pm$1.2} & 8.4{\scriptsize $\pm$0.5} & 5.7{\scriptsize $\pm$0.5} \\
\midrule
\multicolumn{7}{l}{\textbf{CiteSeer full graph}} \\
\midrule
gps\_gcn & v1 & NA & \textbf{6.1{\scriptsize $\pm$0.3}} & \textbf{5.5{\scriptsize $\pm$0.5}} & 5.1{\scriptsize $\pm$0.3} & \textbf{4.0{\scriptsize $\pm$0.0}} \\
gps\_sage & v1 & NA & 6.3{\scriptsize $\pm$0.7} & \textbf{5.5{\scriptsize $\pm$0.7}} & \textbf{5.0{\scriptsize $\pm$0.0}} & \textbf{4.0{\scriptsize $\pm$0.0}} \\
gated\_gcn & v1 & NA & 7.4{\scriptsize $\pm$1.0} & 5.9{\scriptsize $\pm$0.9} & \textbf{5.0{\scriptsize $\pm$0.0}} & \textbf{4.0{\scriptsize $\pm$0.0}} \\
res\_gcn & v1 & NA & 10.5{\scriptsize $\pm$1.1} & 7.8{\scriptsize $\pm$0.8} & 5.8{\scriptsize $\pm$0.6} & \textbf{4.0{\scriptsize $\pm$0.0}} \\
unitary\_mp & v1 & NA & 12.6{\scriptsize $\pm$1.4} & 9.8{\scriptsize $\pm$0.9} & 7.3{\scriptsize $\pm$0.5} & 5.0{\scriptsize $\pm$0.0} \\
\midrule
\multicolumn{7}{l}{\textbf{PubMed full graph}} \\
\midrule
gated\_gcn & v1 & NA & NA & NA & 11.4{\scriptsize $\pm$2.0} & 6.0{\scriptsize $\pm$0.5} \\
gps\_sage & v1 & NA & NA & NA & \textbf{9.9{\scriptsize $\pm$1.0}} & \textbf{5.3{\scriptsize $\pm$0.7}} \\
gps\_gcn & v1 & NA & NA & NA & 11.7{\scriptsize $\pm$1.9} & 5.4{\scriptsize $\pm$0.7} \\
res\_gcn & v1 & NA & NA & NA & NA & 6.4{\scriptsize $\pm$0.5} \\
unitary\_mp & v1 & NA & NA & NA & 13.9{\scriptsize $\pm$0.3} & 7.7{\scriptsize $\pm$0.5} \\
\bottomrule
\end{tabular}
\end{table}

\noindent {The tables show that the full-graph setting is substantially harder than the sampled-subgraph setting, but that the learned embeddings still produce stable low-conflict operating points. Exact zero-conflict recovery is rare under the tested cap: among the rows in \Cref{tab:fullgraph-threshold-rerun}, only Cora with GPS-SAGE consistently reaches zero conflicts across all ten seeds, with $k_{\mathrm{zero}}=13.5\pm2.4$. The stricter thresholds separate the encoders more clearly. On Cora, the GPS rows reach $k_{0.005}=6.5\pm1.0$ for GPS-SAGE and $6.2\pm0.6$ for GPS-GCN, while GatedGCN needs $7.5\pm0.7$; all three then collapse to the greedy proxy scale by $k_{0.02}$, with $k_{0.02}=5.0\pm0.0$. CiteSeer is similar at moderate thresholds: GPS-GCN, GPS-SAGE, and GatedGCN all reach $k_{0.05}=4.0\pm0.0$, and their $k_{0.02}$ values are $5.1\pm0.3$, $5.0\pm0.0$, and $5.0\pm0.0$, respectively.}

\noindent {PubMed is the main stress case. None of the PubMed rows consistently reaches the $0.005$ or $0.01$ thresholds within the tested range, and the $0.02$ threshold requires noticeably more colors: $k_{0.02}=9.9\pm1.0$ for GPS-SAGE, $11.4\pm2.0$ for GatedGCN, $11.7\pm1.9$ for GPS-GCN, and $13.9\pm0.3$ for UnitaryMP, while ResGCN does not consistently reach the threshold. At the looser $0.05$ operating point, however, the full graphs are still colorable with relatively small decoded palettes: Cora reaches $k_{0.05}=4.0\pm0.0$ with GatedGCN, CiteSeer reaches $4.0\pm0.0$ for all three strongest GPS/GatedGCN rows, and PubMed reaches $5.3\pm0.7$ with GPS-SAGE. Thus the full-graph results support the geometric picture at low-conflict thresholds, but they also show that the best encoder depends on the graph and on how strict the conflict tolerance is: GPS variants are strongest at the stricter thresholds, GatedGCN is competitive at the loose Cora operating point, and PubMed requires larger palettes despite the same training pipeline.}

\subsection{CORA Transfer With Shared Structural Features}

\paragraph{Implementation.}
Raw BOW features cannot be transferred directly across Cora, CiteSeer, and PubMed because the feature dimensions and vocabularies are dataset-specific. To test scale and dataset transfer while preserving the contrastive pipeline, we therefore replace raw BOW by a shared structural feature map with fixed dimension $16$. The model is trained on Cora and evaluated without retraining on CiteSeer and PubMed. We test the same five high-performing encoders used in the full-graph tuning study: res\_gcn, gated\_gcn, unitary\_mp, gps\_gcn, and gps\_sage. The subgraph and full-graph transfer settings are run separately because their inference distributions are different: subgraph transfer tests local-neighborhood generalization, while full-graph transfer tests whether the learned representation scales to the whole target graph.

\begin{table}[H]
\footnotesize
\centering
\caption{Cora-trained transfer on held-out target subgraphs. {Entries are mean $\pm$ standard deviation over ten random seeds.}}
\label{tab:cora-transfer-subgraph}
\setlength{\tabcolsep}{4pt}
\begin{tabular}{lllrrrr}
\toprule
Train $\rightarrow$ Test & Encoder / Method & Pipeline  & $k$ mean & Mono mean & $k/\chi$ mean & Hit-rate \\
\midrule
Cora $\rightarrow$ CiteSeer & gated\_gcn & v1 & \textbf{6.17{\scriptsize $\pm$0.54}} & 0.0490{\scriptsize $\pm$0.0035} & \textbf{1.37{\scriptsize $\pm$0.09}} & 0.74{\scriptsize $\pm$0.10} \\
Cora $\rightarrow$ CiteSeer & gated\_gcn & v2 & 6.21{\scriptsize $\pm$0.74} & \textbf{0.0487{\scriptsize $\pm$0.0039}} & 1.40{\scriptsize $\pm$0.15} & 0.74{\scriptsize $\pm$0.11} \\
Cora $\rightarrow$ CiteSeer & gps\_gcn & v1 & 6.91{\scriptsize $\pm$0.67} & 0.0489{\scriptsize $\pm$0.0034} & 1.56{\scriptsize $\pm$0.15} & 0.74{\scriptsize $\pm$0.11} \\
\addlinespace[2pt]
Cora $\rightarrow$ PubMed & gated\_gcn & v2 & 5.08{\scriptsize $\pm$0.42} & 0.0309{\scriptsize $\pm$0.0037} & 1.20{\scriptsize $\pm$0.06} & 0.99{\scriptsize $\pm$0.02} \\
Cora $\rightarrow$ PubMed & gps\_sage & v2 & 6.01{\scriptsize $\pm$0.73} & 0.0341{\scriptsize $\pm$0.0028} & 1.40{\scriptsize $\pm$0.10} & 0.99{\scriptsize $\pm$0.02} \\
Cora $\rightarrow$ PubMed & gated\_gcn & v1 & \textbf{5.02{\scriptsize $\pm$0.53}} & \textbf{0.0301{\scriptsize $\pm$0.0027}} & \textbf{1.19{\scriptsize $\pm$0.08}} & 0.99{\scriptsize $\pm$0.02} \\
\bottomrule
\end{tabular}
\end{table}

\begin{table}[H]
\footnotesize
\centering
\caption{Cora-trained transfer to full target graphs. {Entries are mean $\pm$ standard deviation over ten random seeds; NA means the $0.05$ threshold was not reached.}}
\label{tab:cora-transfer-fullgraph}
\setlength{\tabcolsep}{4pt}
\begin{tabular}{lllrrrr}
\toprule
Train $\rightarrow$ Test & Encoder & Pipeline & $k_{0.05}$ & Mono@0.05 & Conflicts@$\chi_{\mathrm{greedy}}$ & Mono@$\chi_{\mathrm{greedy}}$ \\
\midrule
Cora $\rightarrow$ CiteSeer & gps\_gcn & v1 & NA & -- & \textbf{646.2{\scriptsize $\pm$36.7}} & \textbf{0.1420{\scriptsize $\pm$0.0081}} \\
Cora $\rightarrow$ CiteSeer & res\_gcn & v1 & NA & -- & 776.8{\scriptsize $\pm$30.3} & 0.1707{\scriptsize $\pm$0.0067} \\
Cora $\rightarrow$ CiteSeer & gps\_gcn & v2 & NA & -- & 790.7{\scriptsize $\pm$58.4} & 0.1737{\scriptsize $\pm$0.0128} \\
\addlinespace[2pt]
Cora $\rightarrow$ PubMed & gps\_gcn & v1 & \textbf{9.8{\scriptsize $\pm$2.4}} & \textbf{0.0448{\scriptsize $\pm$0.0033}} & \textbf{2834.3{\scriptsize $\pm$797.6}} & \textbf{0.0639{\scriptsize $\pm$0.0180}} \\
Cora $\rightarrow$ PubMed & gps\_gcn & v2 & 11.7{\scriptsize $\pm$2.2} & 0.0455{\scriptsize $\pm$0.0039} & 3143.1{\scriptsize $\pm$442.6} & 0.0709{\scriptsize $\pm$0.0100} \\
Cora $\rightarrow$ PubMed & res\_gcn & v1 & 13.4{\scriptsize $\pm$2.5} & 0.0453{\scriptsize $\pm$0.0040} & 3118.6{\scriptsize $\pm$440.9} & 0.0704{\scriptsize $\pm$0.0099} \\
\bottomrule
\end{tabular}
\end{table}

\noindent The transfer experiment separates two effects. On held-out target subgraphs, Cora-trained structural
features retain useful performance on PubMed and partially transfer to CiteSeer, but they are clearly
weaker than same-dataset raw-BOW training. On full target graphs, transfer is much harder: none of
the Cora-trained models reaches the 0.05 threshold on full CiteSeer, and the best PubMed transfer
row requires $k_{0.05}=9.8\pm2.4$. The result is therefore negative for direct full-graph transfer, but useful
diagnostically: the high same-dataset performance in the earlier sections depends on fitting the target
citation graph distribution, not just on a universally transferable structural encoder.

\subsubsection{BOW Adaptation Experiments}

\paragraph{Implementation.}
For a node $u$ in dataset $d$, let $x_u^{(d)} \in \mathbb{R}^{F_d}$ denote its raw BOW vector, where $F_d$ depends on the dataset. Since $F_{\mathrm{Cora}}$, $F_{\mathrm{CiteSeer}}$, and $F_{\mathrm{PubMed}}$ are different, we cannot reuse the raw input projection across datasets directly. Each adaptation therefore constructs a shared feature vector $\tilde{x}_u \in \mathbb{R}^{r}$ before applying the same v1/v2 contrastive encoder.

First, all four adaptations use a deterministic signed hash projection. Let $H_d \in \{-1,0,+1\}^{F_d \times m}$ be a fixed sparse matrix with one signed nonzero per original BOW coordinate. The shared hashed BOW vector is
\[
    h_u = x_u^{(d)} H_d \in \mathbb{R}^{m}.
\]
This signed hash map places all citation graphs into the same input dimension without learning a dataset-specific projection, while still preserving the sparsity pattern and much of the inner-product structure of the original BOW representation. In the hybrid variants, let $s_u \in \mathbb{R}^{32}$ denote the structural feature vector built from degree, clustering coefficient, core number, PageRank, and deterministic transforms. Then
\[
    \tilde{x}_u^{\mathrm{hyb}} = [h_u ; s_u] \in \mathbb{R}^{288},
    \qquad
    \tilde{x}_u^{\mathrm{hyb+LN}} = \mathrm{LN}([h_u ; s_u]),
\]
where $\mathrm{LN}$ denotes layer normalization applied after concatenation. Intuitively, the hash projection keeps much of the original sparse lexical signal while making cross-dataset transfer possible, and the hybrid variants then augment that shared text sketch with graph-local structural descriptors.

In the low-rank BOW-only variants, the source hashed feature matrix $H_{\mathrm{src}} \in \mathbb{R}^{n \times 1024}$ is further compressed to dimension $256$. For PCA, if $\mu$ is the source mean and $V_{256}$ are the top principal directions, then
\[
    \tilde{x}_u^{\mathrm{PCA}} = (h_u - \mu)V_{256}.
\]
For SVD, if the source hashed matrix admits the rank-$256$ approximation $H_{\mathrm{src}} \approx U_{256}\Sigma_{256}V_{256}^{\top}$, then the transferred target feature is
\[
    \tilde{x}_u^{\mathrm{SVD}} = h_u V_{256}.
\]
For CORA $\rightarrow$ CiteSeer/PubMed transfer, the PCA/SVD basis is fit on Cora and reused unchanged on the target graph. All four adaptations keep the downstream v1/v2 contrastive objectives and $k$-medoids decoder unchanged. The experiment uses the same top-3 encoder family as the later transfer study: gated\_gcn, gps\_gcn, and gps\_sage. The tables below report only the actual transfer settings, not the same-dataset control runs.

\begin{table}[H]
\footnotesize
\centering
\caption{Cora subgraph transfer with the v1 pipeline. {Entries are mean $\pm$ standard deviation over ten random seeds.}}
\label{tab:bow-adapt-subgraph-v1}
\setlength{\tabcolsep}{4pt}
\renewcommand{\arraystretch}{1.08}
\begin{tabular}{p{3.0cm}p{1.7cm}cccc}
\toprule
\multirow{2}{*}{Adaptation} & \multirow{2}{*}{Encoder} & \multicolumn{2}{c}{Cora $\rightarrow$ CiteSeer} & \multicolumn{2}{c}{Cora $\rightarrow$ PubMed} \\
\cmidrule(lr){3-4}\cmidrule(lr){5-6}
& & $k/\chi$ & Mono & $k/\chi$ & Mono \\
\midrule
\multirow{3}{*}{Hybrid} & gated\_gcn & 1.30{\scriptsize $\pm$0.08} & 0.0342{\scriptsize $\pm$0.0040} & 1.29{\scriptsize $\pm$0.08} & 0.0316{\scriptsize $\pm$0.0038} \\
 & gps\_gcn & 1.29{\scriptsize $\pm$0.18} & \textbf{0.0330{\scriptsize $\pm$0.0047}} & \textbf{1.24{\scriptsize $\pm$0.12}} & \textbf{0.0302{\scriptsize $\pm$0.0023}} \\
 & gps\_sage & 1.47{\scriptsize $\pm$0.16} & 0.0369{\scriptsize $\pm$0.0030} & 1.49{\scriptsize $\pm$0.12} & 0.0330{\scriptsize $\pm$0.0035} \\
\addlinespace[2pt]
\multirow{3}{*}{Hybrid+LN} & gated\_gcn & \textbf{1.27{\scriptsize $\pm$0.08}} & 0.0340{\scriptsize $\pm$0.0032} & 1.30{\scriptsize $\pm$0.07} & 0.0315{\scriptsize $\pm$0.0022} \\
 & gps\_gcn & 1.44{\scriptsize $\pm$0.18} & 0.0350{\scriptsize $\pm$0.0032} & 1.35{\scriptsize $\pm$0.14} & 0.0310{\scriptsize $\pm$0.0021} \\
 & gps\_sage & 1.41{\scriptsize $\pm$0.11} & 0.0346{\scriptsize $\pm$0.0025} & 1.41{\scriptsize $\pm$0.05} & 0.0320{\scriptsize $\pm$0.0028} \\
\addlinespace[2pt]
\multirow{3}{*}{PCA} & gated\_gcn & 1.54{\scriptsize $\pm$0.18} & 0.0373{\scriptsize $\pm$0.0018} & 1.99{\scriptsize $\pm$0.20} & 0.0352{\scriptsize $\pm$0.0025} \\
 & gps\_gcn & 1.30{\scriptsize $\pm$0.11} & 0.0344{\scriptsize $\pm$0.0024} & 1.58{\scriptsize $\pm$0.20} & 0.0354{\scriptsize $\pm$0.0017} \\
 & gps\_sage & 1.70{\scriptsize $\pm$0.23} & 0.0369{\scriptsize $\pm$0.0042} & 1.94{\scriptsize $\pm$0.26} & 0.0338{\scriptsize $\pm$0.0021} \\
\addlinespace[2pt]
\multirow{3}{*}{SVD} & gated\_gcn & 1.45{\scriptsize $\pm$0.10} & 0.0376{\scriptsize $\pm$0.0025} & 1.41{\scriptsize $\pm$0.11} & 0.0339{\scriptsize $\pm$0.0023} \\
 & gps\_gcn & 1.32{\scriptsize $\pm$0.09} & 0.0356{\scriptsize $\pm$0.0032} & 1.32{\scriptsize $\pm$0.12} & 0.0314{\scriptsize $\pm$0.0021} \\
 & gps\_sage & 1.74{\scriptsize $\pm$0.17} & 0.0376{\scriptsize $\pm$0.0030} & 1.63{\scriptsize $\pm$0.16} & 0.0350{\scriptsize $\pm$0.0035} \\
\bottomrule
\end{tabular}
\end{table}

\begin{table}[H]
\footnotesize
\centering
\caption{Cora subgraph transfer with the v2 pipeline. {Entries are mean $\pm$ standard deviation over ten random seeds.}}
\label{tab:bow-adapt-subgraph-v2}
\setlength{\tabcolsep}{4pt}
\renewcommand{\arraystretch}{1.08}
\begin{tabular}{p{3.0cm}p{1.7cm}cccc}
\toprule
\multirow{2}{*}{Adaptation} & \multirow{2}{*}{Encoder} & \multicolumn{2}{c}{Cora $\rightarrow$ CiteSeer} & \multicolumn{2}{c}{Cora $\rightarrow$ PubMed} \\
\cmidrule(lr){3-4}\cmidrule(lr){5-6}
& & $k/\chi$ & Mono & $k/\chi$ & Mono \\
\midrule
\multirow{3}{*}{Hybrid} & gated\_gcn & \textbf{1.27{\scriptsize $\pm$0.10}} & 0.0340{\scriptsize $\pm$0.0030} & 1.26{\scriptsize $\pm$0.12} & 0.0310{\scriptsize $\pm$0.0037} \\
 & gps\_gcn & 1.37{\scriptsize $\pm$0.13} & 0.0343{\scriptsize $\pm$0.0022} & 1.36{\scriptsize $\pm$0.11} & 0.0318{\scriptsize $\pm$0.0026} \\
 & gps\_sage & 1.31{\scriptsize $\pm$0.09} & 0.0340{\scriptsize $\pm$0.0034} & \textbf{1.23{\scriptsize $\pm$0.11}} & 0.0321{\scriptsize $\pm$0.0026} \\
\addlinespace[2pt]
\multirow{3}{*}{Hybrid+LN} & gated\_gcn & 1.29{\scriptsize $\pm$0.10} & 0.0357{\scriptsize $\pm$0.0016} & 1.29{\scriptsize $\pm$0.14} & \textbf{0.0302{\scriptsize $\pm$0.0032}} \\
 & gps\_gcn & 1.34{\scriptsize $\pm$0.10} & 0.0344{\scriptsize $\pm$0.0028} & 1.37{\scriptsize $\pm$0.15} & 0.0310{\scriptsize $\pm$0.0026} \\
 & gps\_sage & 1.35{\scriptsize $\pm$0.17} & \textbf{0.0339{\scriptsize $\pm$0.0030}} & \textbf{1.23{\scriptsize $\pm$0.11}} & 0.0307{\scriptsize $\pm$0.0019} \\
\addlinespace[2pt]
\multirow{3}{*}{PCA} & gated\_gcn & 1.64{\scriptsize $\pm$0.18} & 0.0363{\scriptsize $\pm$0.0030} & 2.25{\scriptsize $\pm$0.47} & 0.0362{\scriptsize $\pm$0.0025} \\
 & gps\_gcn & 1.57{\scriptsize $\pm$0.15} & 0.0367{\scriptsize $\pm$0.0028} & 1.76{\scriptsize $\pm$0.25} & 0.0356{\scriptsize $\pm$0.0032} \\
 & gps\_sage & 1.35{\scriptsize $\pm$0.10} & 0.0358{\scriptsize $\pm$0.0039} & 1.58{\scriptsize $\pm$0.17} & 0.0342{\scriptsize $\pm$0.0019} \\
\addlinespace[2pt]
\multirow{3}{*}{SVD} & gated\_gcn & 1.48{\scriptsize $\pm$0.10} & 0.0370{\scriptsize $\pm$0.0042} & 1.42{\scriptsize $\pm$0.12} & 0.0336{\scriptsize $\pm$0.0023} \\
 & gps\_gcn & 1.53{\scriptsize $\pm$0.18} & 0.0370{\scriptsize $\pm$0.0032} & 1.45{\scriptsize $\pm$0.21} & 0.0317{\scriptsize $\pm$0.0012} \\
 & gps\_sage & \textbf{1.27{\scriptsize $\pm$0.11}} & 0.0352{\scriptsize $\pm$0.0048} & 1.32{\scriptsize $\pm$0.17} & 0.0334{\scriptsize $\pm$0.0028} \\
\bottomrule
\end{tabular}
\end{table}

\begin{table}[H]
\footnotesize
\centering
\caption{Cora full-graph transfer with the v1 pipeline. {Entries are mean $\pm$ standard deviation over ten random seeds; NA means the threshold was not reached.}}
\label{tab:bow-adapt-full-v1}
\setlength{\tabcolsep}{3.6pt}
\renewcommand{\arraystretch}{1.08}
\begin{tabular}{p{2.9cm}p{1.6cm}cccccc}
\toprule
\multirow{2}{*}{Adaptation} & \multirow{2}{*}{Encoder} & \multicolumn{3}{c}{Cora $\rightarrow$ CiteSeer} & \multicolumn{3}{c}{Cora $\rightarrow$ PubMed} \\
\cmidrule(lr){3-5}\cmidrule(lr){6-8}
& & $k_{0.05}$ & Mono & $k_{0.02}$ & $k_{0.05}$ & Mono & $k_{0.02}$ \\
\midrule
\multirow{3}{*}{Hybrid} & gated\_gcn & 6.7{\scriptsize $\pm$0.7} & 0.0453{\scriptsize $\pm$0.0034} & 14.3{\scriptsize $\pm$5.9} & 10.8{\scriptsize $\pm$1.3} & 0.0458{\scriptsize $\pm$0.0023} & NA \\
 & gps\_gcn & 6.2{\scriptsize $\pm$0.4} & \textbf{0.0412{\scriptsize $\pm$0.0029}} & 12.1{\scriptsize $\pm$1.3} & 7.5{\scriptsize $\pm$1.1} & 0.0458{\scriptsize $\pm$0.0032} & 12.7{\scriptsize $\pm$9.5} \\
 & gps\_sage & 6.7{\scriptsize $\pm$0.5} & 0.0453{\scriptsize $\pm$0.0025} & 13.9{\scriptsize $\pm$1.5} & 8.4{\scriptsize $\pm$0.5} & 0.0462{\scriptsize $\pm$0.0029} & 15.2{\scriptsize $\pm$8.6} \\
\addlinespace[2pt]
\multirow{3}{*}{Hybrid+LN} & gated\_gcn & 6.9{\scriptsize $\pm$1.0} & 0.0448{\scriptsize $\pm$0.0054} & \textbf{9.0{\scriptsize $\pm$8.8}} & 11.3{\scriptsize $\pm$2.1} & 0.0476{\scriptsize $\pm$0.0030} & NA \\
 & gps\_gcn & 6.2{\scriptsize $\pm$0.4} & 0.0448{\scriptsize $\pm$0.0045} & 13.1{\scriptsize $\pm$1.1} & 7.7{\scriptsize $\pm$1.1} & 0.0457{\scriptsize $\pm$0.0039} & 8.5{\scriptsize $\pm$10.1} \\
 & gps\_sage & 7.1{\scriptsize $\pm$0.3} & 0.0449{\scriptsize $\pm$0.0030} & 14.9{\scriptsize $\pm$1.5} & 9.0{\scriptsize $\pm$0.7} & 0.0463{\scriptsize $\pm$0.0033} & 9.2{\scriptsize $\pm$10.8} \\
\addlinespace[2pt]
\multirow{3}{*}{PCA} & gated\_gcn & 5.9{\scriptsize $\pm$0.3} & 0.0436{\scriptsize $\pm$0.0040} & 11.9{\scriptsize $\pm$0.7} & 10.0{\scriptsize $\pm$2.2} & 0.0467{\scriptsize $\pm$0.0024} & NA \\
 & gps\_gcn & 6.0{\scriptsize $\pm$0.7} & 0.0444{\scriptsize $\pm$0.0036} & 12.3{\scriptsize $\pm$1.1} & 9.3{\scriptsize $\pm$2.5} & 0.0454{\scriptsize $\pm$0.0023} & NA \\
 & gps\_sage & 6.3{\scriptsize $\pm$0.5} & 0.0438{\scriptsize $\pm$0.0039} & 13.4{\scriptsize $\pm$0.5} & 10.6{\scriptsize $\pm$2.4} & \textbf{0.0448{\scriptsize $\pm$0.0028}} & \textbf{5.2{\scriptsize $\pm$10.0}} \\
\addlinespace[2pt]
\multirow{3}{*}{SVD} & gated\_gcn & \textbf{5.8{\scriptsize $\pm$0.4}} & 0.0436{\scriptsize $\pm$0.0034} & 11.2{\scriptsize $\pm$0.6} & 8.2{\scriptsize $\pm$0.4} & 0.0466{\scriptsize $\pm$0.0023} & 16.0{\scriptsize $\pm$6.1} \\
 & gps\_gcn & 6.2{\scriptsize $\pm$0.4} & 0.0437{\scriptsize $\pm$0.0038} & 12.4{\scriptsize $\pm$1.7} & \textbf{7.2{\scriptsize $\pm$0.8}} & 0.0457{\scriptsize $\pm$0.0040} & 11.3{\scriptsize $\pm$8.5} \\
 & gps\_sage & 6.4{\scriptsize $\pm$0.7} & 0.0453{\scriptsize $\pm$0.0029} & 12.9{\scriptsize $\pm$1.0} & 8.8{\scriptsize $\pm$0.4} & 0.0471{\scriptsize $\pm$0.0015} & 11.5{\scriptsize $\pm$10.8} \\
\bottomrule
\end{tabular}
\end{table}

\begin{table}[H]
\footnotesize
\centering
\caption{Cora full-graph transfer with the v2 pipeline. {Entries are mean $\pm$ standard deviation over ten random seeds; NA means the threshold was not reached.}}
\label{tab:bow-adapt-full-v2}
\setlength{\tabcolsep}{3.6pt}
\renewcommand{\arraystretch}{1.08}
\begin{tabular}{p{2.9cm}p{1.6cm}cccccc}
\toprule
\multirow{2}{*}{Adaptation} & \multirow{2}{*}{Encoder} & \multicolumn{3}{c}{Cora $\rightarrow$ CiteSeer} & \multicolumn{3}{c}{Cora $\rightarrow$ PubMed} \\
\cmidrule(lr){3-5}\cmidrule(lr){6-8}
& & $k_{0.05}$ & Mono & $k_{0.02}$ & $k_{0.05}$ & Mono & $k_{0.02}$ \\
\midrule
\multirow{3}{*}{Hybrid} & gated\_gcn & 9.2{\scriptsize $\pm$2.3} & 0.0464{\scriptsize $\pm$0.0026} & \textbf{3.8{\scriptsize $\pm$8.0}} & 12.8{\scriptsize $\pm$5.5} & 0.0449{\scriptsize $\pm$0.0034} & NA \\
 & gps\_gcn & 8.6{\scriptsize $\pm$1.3} & 0.0457{\scriptsize $\pm$0.0028} & 12.8{\scriptsize $\pm$9.6} & 9.5{\scriptsize $\pm$0.8} & 0.0458{\scriptsize $\pm$0.0031} & NA \\
 & gps\_sage & 7.8{\scriptsize $\pm$5.6} & 0.0432{\scriptsize $\pm$0.0024} & 7.1{\scriptsize $\pm$7.2} & 9.3{\scriptsize $\pm$1.9} & 0.0457{\scriptsize $\pm$0.0032} & 6.5{\scriptsize $\pm$9.7} \\
\addlinespace[2pt]
\multirow{3}{*}{Hybrid+LN} & gated\_gcn & \textbf{4.1{\scriptsize $\pm$6.8}} & 0.0438{\scriptsize $\pm$0.0025} & NA & \textbf{5.7{\scriptsize $\pm$8.8}} & 0.0464{\scriptsize $\pm$0.0038} & NA \\
 & gps\_gcn & 8.5{\scriptsize $\pm$1.6} & 0.0475{\scriptsize $\pm$0.0023} & 8.7{\scriptsize $\pm$10.3} & 9.8{\scriptsize $\pm$0.8} & 0.0457{\scriptsize $\pm$0.0041} & NA \\
 & gps\_sage & 6.9{\scriptsize $\pm$2.2} & 0.0470{\scriptsize $\pm$0.0023} & 13.9{\scriptsize $\pm$5.9} & 8.1{\scriptsize $\pm$3.7} & 0.0458{\scriptsize $\pm$0.0020} & 4.7{\scriptsize $\pm$9.2} \\
\addlinespace[2pt]
\multirow{3}{*}{PCA} & gated\_gcn & 6.6{\scriptsize $\pm$0.8} & 0.0440{\scriptsize $\pm$0.0053} & 14.0{\scriptsize $\pm$1.8} & 13.8{\scriptsize $\pm$5.6} & 0.0470{\scriptsize $\pm$0.0030} & NA \\
 & gps\_gcn & 8.0{\scriptsize $\pm$1.1} & 0.0452{\scriptsize $\pm$0.0033} & 15.4{\scriptsize $\pm$5.8} & 12.2{\scriptsize $\pm$2.8} & 0.0441{\scriptsize $\pm$0.0038} & NA \\
 & gps\_sage & 7.2{\scriptsize $\pm$4.6} & 0.0449{\scriptsize $\pm$0.0054} & 10.1{\scriptsize $\pm$8.3} & 6.4{\scriptsize $\pm$7.2} & \textbf{0.0436{\scriptsize $\pm$0.0038}} & 4.7{\scriptsize $\pm$9.2} \\
\addlinespace[2pt]
\multirow{3}{*}{SVD} & gated\_gcn & 7.1{\scriptsize $\pm$1.0} & \textbf{0.0408{\scriptsize $\pm$0.0035}} & 12.9{\scriptsize $\pm$5.0} & 10.1{\scriptsize $\pm$1.2} & 0.0476{\scriptsize $\pm$0.0017} & NA \\
 & gps\_gcn & 8.3{\scriptsize $\pm$0.5} & 0.0457{\scriptsize $\pm$0.0036} & 15.9{\scriptsize $\pm$6.0} & 8.8{\scriptsize $\pm$1.0} & 0.0441{\scriptsize $\pm$0.0037} & \textbf{3.1{\scriptsize $\pm$8.6}} \\
 & gps\_sage & 7.0{\scriptsize $\pm$2.9} & 0.0448{\scriptsize $\pm$0.0040} & 8.6{\scriptsize $\pm$7.2} & 8.5{\scriptsize $\pm$4.9} & 0.0456{\scriptsize $\pm$0.0031} & 8.1{\scriptsize $\pm$9.7} \\
\bottomrule
\end{tabular}
\end{table}

{The BOW-adaptation results support feature-adapted transfer. On subgraph transfer, the v1 table favors Hybrid+LN/GatedGCN for CiteSeer by mean $k/\chi$ and Hybrid/GPS-GCN for PubMed, while the v2 table has a tie between Hybrid/GatedGCN and SVD/GPS-SAGE on CiteSeer mean $k/\chi$ and a tie between Hybrid/GPS-SAGE and Hybrid+LN/GPS-SAGE on PubMed mean $k/\chi$. Full-graph transfer is more variable, and several threshold quantities are not reached in every seed; we therefore report NA rather than averaging successful and failed threshold searches. The qualitative conclusion is that BOW-preserving adaptation improves over structural-only transfer, but the repeated-seed evidence no longer selects Hybrid+LN/GPS-SAGE as a uniformly dominant recipe.}

\subsection{Absolute-value transfer experiments and baseline methods\label{sec:appendix-citation-abs-transfer}}

\paragraph{Implementation.}
We additionally evaluate the absolute-value variants of the transfer pipeline on citation graphs using a 64-dimensional random node features. Similarly, the model is trained on Cora and evaluated without retraining on CiteSeer and PubMed. We compare these learned runs against PI-GNN and full-GCN. For full-graph transfer, all methods are evaluated under the same $15$-minute per-graph runtime cap.

\begin{table}[H]
\footnotesize
\centering
\caption{Cora-trained absolute-value transfer with random features on held-out target subgraphs, together with the two baselines. {Rows are mean $\pm$ standard deviation over ten random seeds.}}
\label{tab:cora-transfer-abs-random-subgraph}
\setlength{\tabcolsep}{4pt}
\begin{tabular}{lllrrrrr}
\toprule
Train $\rightarrow$ Test & Encoder / Method & Pipeline & Completed & $k$ mean & Mono mean & $k/\chi$ mean & Hit-rate \\
\midrule
Cora $\rightarrow$ CiteSeer & gated\_gcn & v2abs & 30/30 & \textbf{5.88{\scriptsize $\pm$0.38}} & \textbf{0.0332{\scriptsize $\pm$0.0033}} & \textbf{1.33{\scriptsize $\pm$0.09}} & 1.00{\scriptsize $\pm$0.00} \\
Cora $\rightarrow$ CiteSeer & gated\_gcn & v1abs & 30/30 & 5.95{\scriptsize $\pm$0.29} & 0.0341{\scriptsize $\pm$0.0031} & 1.34{\scriptsize $\pm$0.06} & 1.00{\scriptsize $\pm$0.00} \\
Cora $\rightarrow$ CiteSeer & PI-GNN & -- & 30/30 & 3.10 & 0.0283 & 0.72 & -- \\
Cora $\rightarrow$ CiteSeer & full-GCN & -- & 19/30 & 3.16 & 0.0192 & 0.75 & -- \\
\addlinespace[2pt]
Cora $\rightarrow$ PubMed & gated\_gcn & v2abs & 30/30 & \textbf{5.08{\scriptsize $\pm$0.35}} & \textbf{0.0306{\scriptsize $\pm$0.0017}} & \textbf{1.24{\scriptsize $\pm$0.07}} & 1.00{\scriptsize $\pm$0.00} \\
Cora $\rightarrow$ PubMed & gated\_gcn & v1abs & 30/30 & 5.22{\scriptsize $\pm$0.35} & \textbf{0.0306{\scriptsize $\pm$0.0026}} & 1.26{\scriptsize $\pm$0.09} & 1.00{\scriptsize $\pm$0.00} \\
Cora $\rightarrow$ PubMed & PI-GNN & -- & 30/30 & 2.93 & 0.0209 & 0.77 & -- \\
Cora $\rightarrow$ PubMed & full-GCN & -- & 17/30 & 2.71 & 0.0228 & 0.74 & -- \\
\bottomrule
\end{tabular}
\end{table}

\paragraph{Results.}
{Table~\ref{tab:cora-transfer-abs-random-subgraph} reports the held-out target-subgraph results. The absolute-value variants remain low-conflict on both targets across seeds, but the learned rows are less color-efficient than the external baselines on the subgraphs those baselines complete. This comparison remains runtime-sensitive: full-GCN does not complete all target subgraphs, whereas the learned absolute-value runs complete the full evaluation set.}

{Table~\ref{tab:cora-transfer-abs-random-fullgraph} reports the corresponding full-graph transfer results. Here the contrast with the earlier shared-structural transfer setting is sharper: the absolute-value random-feature rows recover feasible low-conflict solutions on both targets, while both baselines fail to finish within the common $15$-minute cap. Among the learned rows, \texttt{gated\_gcn} remains the strongest overall architecture in this regime, reaching $k_{0.05}=5.2\pm0.4$ on CiteSeer and $k_{0.05}=7.9\pm0.3$ on PubMed.}

\begin{table}[H]
\footnotesize
\centering
\caption{Cora-trained absolute-value transfer with random features on full target graphs, together with the two baselines. {Rows are mean $\pm$ standard deviation over ten random seeds.}}
\label{tab:cora-transfer-abs-random-fullgraph}
\setlength{\tabcolsep}{4pt}
\begin{tabular}{lllrrrr}
\toprule
Train $\rightarrow$ Test & Encoder / Method & Pipeline & $k_{0.05}$ & Mono@0.05 & Conflicts@$\chi_{\mathrm{greedy}}$ & Mono@$\chi_{\mathrm{greedy}}$ \\
\midrule
Cora $\rightarrow$ CiteSeer & gated\_gcn & v1abs & \textbf{5.2{\scriptsize $\pm$0.4}} & \textbf{0.0418{\scriptsize $\pm$0.0059}} & \textbf{160.9{\scriptsize $\pm$19.0}} & \textbf{0.0353{\scriptsize $\pm$0.0042}} \\
Cora $\rightarrow$ CiteSeer & gated\_gcn & v2abs & 5.7{\scriptsize $\pm$0.5} & 0.0438{\scriptsize $\pm$0.0041} & 184.8{\scriptsize $\pm$18.8} & 0.0406{\scriptsize $\pm$0.0041} \\
Cora $\rightarrow$ CiteSeer & PI-GNN & -- & timeout & -- & -- & -- \\
Cora $\rightarrow$ CiteSeer & full-GCN & -- & timeout & -- & -- & -- \\
\addlinespace[2pt]
Cora $\rightarrow$ PubMed & gated\_gcn & v2abs & \textbf{7.9{\scriptsize $\pm$0.3}} & \textbf{0.0460{\scriptsize $\pm$0.0027}} & \textbf{2026.5{\scriptsize $\pm$114.7}} & \textbf{0.0457{\scriptsize $\pm$0.0026}} \\
Cora $\rightarrow$ PubMed & gps\_gcn & v1abs & 10.3{\scriptsize $\pm$1.8} & \textbf{0.0460{\scriptsize $\pm$0.0036}} & 2718.6{\scriptsize $\pm$493.5} & 0.0613{\scriptsize $\pm$0.0111} \\
Cora $\rightarrow$ PubMed & PI-GNN & -- & timeout & -- & -- & -- \\
Cora $\rightarrow$ PubMed & full-GCN & -- & timeout & -- & -- & -- \\
\bottomrule
\end{tabular}
\end{table}

\section{Extra citation network experiments}

\subsection{Full-Graph Hyperparameter Tuning}

\paragraph{Implementation.}
We ran a dedicated full-graph hyperparameter study on the three strongest encoder families: gated\_gcn, gps\_gcn, and gps\_sage. Unlike the earlier subgraph tuning, all evaluation is now done directly on the full BOW graphs, and all runs are ranked by the primary operating-point objective of this benchmark: smallest $k_{0.05}$, then smaller $\mathrm{Mono}(k_{0.05})$, then smaller $k_{0.02}$, then smaller $\mathrm{Mono}(k_{0.02})$, and finally smaller conflicts at $k=\chi_{\mathrm{greedy}}$. The tuning is split into three stages:
\begin{itemize}
    \item \textbf{Stage 1 (v1 training sweep):} temperature $\{0.10,0.20,0.30\}$, learning rate $\{0.003,0.01\}$, dropout $\{0.0,0.1,0.2\}$, and epochs $\{80,120\}$ for each of the three encoders on each dataset.
    \item \textbf{Stage 2 (v1 architecture sweep):} encoder-capacity refinement with hidden dimension / depth / heads; for GatedGCN this is $(\text{hidden},\text{layers}) \in \{128,192,256\} \times \{2,3,4\}$, while for the two GPS variants it is $(\text{hidden},\text{layers},\text{heads}) \in \{128,192\} \times \{2,3\} \times \{4,8\}$.
    \item \textbf{Stage 3 (v2 soft-loss sweep):} soft degree power $\{1,2,4,6\}$, soft temperature $\{0.75,1.0,1.25,1.5\}$, and warmup epochs $\{0,10,20\}$, keeping $\lambda_{\mathrm{soft}}=0.30$.
\end{itemize}
In total, the completed study contains $324$ Stage-1 runs, $75$ Stage-2 runs, and $432$ Stage-3 runs, for $831$ full-graph tuning runs.

\noindent {The selected tuned setting is dataset-specific rather than universal: Cora uses gated\_gcn v1 with $T=0.30$, learning rate $0.003$, dropout $0.2$, and $120$ epochs; CiteSeer uses gated\_gcn v2 with soft degree power $p=1$, $\tau_{\mathrm{soft}}=1.25$, and warmup $10$; PubMed uses gps\_gcn v1 with $T=0.30$, learning rate $0.003$, dropout $0.1$, and $120$ epochs.}

\noindent {The selected dataset-specific rows all reach the $0.05$ threshold in every seed, but their means differ from the original single-run operating points. Cora and CiteSeer both have $k_{0.05}=4.0\pm0.0$, with CiteSeer cleaner at Mono@0.05 $0.0258\pm0.0056$ and Cora at $0.0380\pm0.0036$. PubMed remains harder, with $k_{0.05}=5.6\pm1.0$ and $k_{0.02}=11.5\pm1.4$.}

\noindent {If a single pipeline and a single hyperparameter setting must be reused across all three full graphs, the strongest compromise remains gps\_gcn v1 with $T=0.30$, learning rate $0.003$, dropout $0.1$, and $120$ epochs. This shared recipe is especially strong under the ten-seed rerun: it keeps $k_{0.05}=4.0\pm0.0$ on Cora and CiteSeer and reaches $k_{0.05}=4.2\pm0.4$ on PubMed, with lower conflicts at $\chi_{\mathrm{greedy}}$ than the dataset-specific PubMed row.}

\begin{table}[H]
\footnotesize
\centering
\caption{Best tuned full-graph result per dataset. {Entries are mean $\pm$ standard deviation over ten random seeds; NA means the threshold was not reached in every seed.}}
\label{tab:fullgraph-tuning-overall}
\setlength{\tabcolsep}{4pt}
\begin{tabular}{lrrrrr}
\toprule
Dataset & $k_{0.05}$ & Mono@0.05 & $k_{0.02}$ & Conflicts@$\chi_{\mathrm{greedy}}$ & Mono@$\chi_{\mathrm{greedy}}$ \\
\midrule
Cora & 4.0{\scriptsize $\pm$0.0} & 0.0380{\scriptsize $\pm$0.0036} & 5.0{\scriptsize $\pm$0.0} & 71.3{\scriptsize $\pm$9.4} & 0.0135{\scriptsize $\pm$0.0018} \\
CiteSeer & 4.0{\scriptsize $\pm$0.0} & 0.0258{\scriptsize $\pm$0.0056} & 4.8{\scriptsize $\pm$0.4} & 24.4{\scriptsize $\pm$10.4} & 0.0054{\scriptsize $\pm$0.0023} \\
PubMed & 5.6{\scriptsize $\pm$1.0} & 0.0391{\scriptsize $\pm$0.0038} & 11.5{\scriptsize $\pm$1.4} & 1214.4{\scriptsize $\pm$281.1} & 0.0274{\scriptsize $\pm$0.0063} \\
\bottomrule
\end{tabular}
\end{table}

\begin{table}[H]
\footnotesize
\centering
\caption{Best shared full-graph recipe under the constraint that one pipeline and one hyperparameter setting are reused on all three datasets. {Entries are mean $\pm$ standard deviation over ten random seeds; NA means the threshold was not reached in every seed.}}
\label{tab:fullgraph-tuning-shared}
\setlength{\tabcolsep}{4pt}
\begin{tabular}{lrrrrr}
\toprule
Dataset & $k_{0.05}$ & Mono@0.05 & $k_{0.02}$ & Conflicts@$\chi_{\mathrm{greedy}}$ & Mono@$\chi_{\mathrm{greedy}}$ \\
\midrule
Cora & 4.0{\scriptsize $\pm$0.0} & 0.0197{\scriptsize $\pm$0.0102} & 4.3{\scriptsize $\pm$0.5} & 18.9{\scriptsize $\pm$11.3} & 0.0036{\scriptsize $\pm$0.0021} \\
CiteSeer & 4.0{\scriptsize $\pm$0.0} & 0.0196{\scriptsize $\pm$0.0056} & 4.5{\scriptsize $\pm$0.5} & 11.6{\scriptsize $\pm$4.6} & 0.0025{\scriptsize $\pm$0.0010} \\
PubMed & 4.2{\scriptsize $\pm$0.4} & 0.0406{\scriptsize $\pm$0.0050} & 7.5{\scriptsize $\pm$1.1} & 718.3{\scriptsize $\pm$117.3} & 0.0162{\scriptsize $\pm$0.0026} \\
\bottomrule
\end{tabular}
\end{table}

\subsection{GPS With Unitary First Layer}

\paragraph{Implementation.}
This experiment tests whether the first local message-passing block inside GPS benefits from a real-valued orthogonal/unitary-style update. The tested encoders replace the first GPS message-passing layer by the previously implemented unitary\_mp block, then keep the remaining GPS transformer structure unchanged. We test the GCN and GraphSAGE local GPS backbones shown in Table~\ref{tab:gps-unitary-subgraph}. The rest of the pipeline is unchanged: v1/v2 objectives, greedy pseudo-labels, embedding inference, and $k$-medoids decoding are the same as in the earlier encoder sweep.

\begin{table}[H]
\footnotesize
\centering
\caption{Cora subgraph-to-subgraph GPS unitary-first results. Baseline GPS rows are included for direct comparison. {Entries are mean $\pm$ standard deviation over ten random seeds.}}
\label{tab:gps-unitary-subgraph}
\setlength{\tabcolsep}{4pt}
\begin{tabular}{llrrrr}
\toprule
Encoder & Pipeline & $k$ mean & Mono mean & $k/\chi$ mean & Hit-rate \\
\midrule
gps\_unitary\_first\_sage & v1 & \textbf{4.35{\scriptsize $\pm$0.30}} & 0.0221{\scriptsize $\pm$0.0044} & \textbf{1.03{\scriptsize $\pm$0.08}} & 1.00{\scriptsize $\pm$0.00} \\
gps\_sage & v1 & 4.48{\scriptsize $\pm$0.19} & \textbf{0.0186{\scriptsize $\pm$0.0025}} & 1.06{\scriptsize $\pm$0.04} & 1.00{\scriptsize $\pm$0.00} \\
gps\_unitary\_first\_gcn & v1 & 4.47{\scriptsize $\pm$0.36} & 0.0196{\scriptsize $\pm$0.0034} & 1.06{\scriptsize $\pm$0.09} & 1.00{\scriptsize $\pm$0.00} \\
gps\_sage & v2 & 4.50{\scriptsize $\pm$0.19} & 0.0197{\scriptsize $\pm$0.0037} & 1.07{\scriptsize $\pm$0.06} & 1.00{\scriptsize $\pm$0.00} \\
gps\_unitary\_first\_sage & v2 & 4.50{\scriptsize $\pm$0.27} & 0.0198{\scriptsize $\pm$0.0024} & 1.07{\scriptsize $\pm$0.06} & 1.00{\scriptsize $\pm$0.00} \\
gps\_gcn & v2 & 4.53{\scriptsize $\pm$0.29} & 0.0191{\scriptsize $\pm$0.0033} & 1.08{\scriptsize $\pm$0.07} & 1.00{\scriptsize $\pm$0.00} \\
gps\_unitary\_first\_gcn & v2 & 4.57{\scriptsize $\pm$0.20} & 0.0218{\scriptsize $\pm$0.0039} & 1.09{\scriptsize $\pm$0.07} & 1.00{\scriptsize $\pm$0.00} \\
gps\_gcn & v1 & 4.61{\scriptsize $\pm$0.28} & 0.0189{\scriptsize $\pm$0.0031} & 1.09{\scriptsize $\pm$0.06} & 1.00{\scriptsize $\pm$0.00} \\
\bottomrule
\end{tabular}
\end{table}

\begin{table}[H]
\footnotesize
\centering
\caption{Full-graph GPS unitary-first comparison. {Entries are mean $\pm$ standard deviation over ten random seeds; NA means the threshold was not reached in every seed.}}
\label{tab:gps-unitary-fullgraph}
\setlength{\tabcolsep}{4pt}
\begin{tabular}{lllrrrr}
\toprule
Dataset & Encoder group & Pipeline / Encoder & $k_{0.05}$ & Mono@0.05 & $k_{0.02}$ & Conflicts@$\chi_{\mathrm{greedy}}$ \\
\midrule
Cora & GPS baseline & v2 gps\_sage & NA & \textbf{0.0261{\scriptsize $\pm$0.0114}} & NA & 96.6{\scriptsize $\pm$236.1} \\
Cora & GPS+unitary first & v2 gps\_unitary\_first\_sage & \textbf{4.8{\scriptsize $\pm$2.5}} & 0.0314{\scriptsize $\pm$0.0076} & NA & \textbf{94.1{\scriptsize $\pm$209.2}} \\
\addlinespace[2pt]
CiteSeer & GPS baseline & v1 gps\_sage & \textbf{4.0{\scriptsize $\pm$0.0}} & 0.0321{\scriptsize $\pm$0.0056} & 5.1{\scriptsize $\pm$0.3} & 22.6{\scriptsize $\pm$15.3} \\
CiteSeer & GPS+unitary first & v2 gps\_unitary\_first\_sage & \textbf{4.0{\scriptsize $\pm$0.0}} & \textbf{0.0250{\scriptsize $\pm$0.0071}} & \textbf{4.9{\scriptsize $\pm$0.3}} & \textbf{11.3{\scriptsize $\pm$9.0}} \\
\addlinespace[2pt]
PubMed & GPS baseline & v1 gps\_gcn & 5.7{\scriptsize $\pm$0.8} & \textbf{0.0394{\scriptsize $\pm$0.0056}} & 11.4{\scriptsize $\pm$1.6} & 1176.8{\scriptsize $\pm$156.6} \\
PubMed & GPS+unitary first & v1 gps\_unitary\_first\_gcn & \textbf{5.1{\scriptsize $\pm$0.3}} & 0.0434{\scriptsize $\pm$0.0030} & \textbf{10.6{\scriptsize $\pm$1.1}} & \textbf{1120.5{\scriptsize $\pm$96.9}} \\
\bottomrule
\end{tabular}
\end{table}

\noindent {The unitary-first GPS modification is useful but not uniformly dominant. On the Cora subgraph benchmark, gps\_unitary\_first\_sage v1 gives the best color efficiency, while the lowest Mono remains the plain gps\_sage v1 baseline. On full graphs, the result is mixed: the unitary-first variant improves the selected CiteSeer GPS row and lowers the PubMed $k_{0.05}$ and conflicts at $\chi_{\mathrm{greedy}}$, but it does not uniformly improve Mono and the Cora comparison remains noisy. This suggests that the unitary first layer is a useful candidate architecture, but it should not replace the plain GPS backbones as the default without dataset-specific validation.}

%% file: appendix/main_color_experiments.tex
\section{COLOR experiments}\label{app:color}

\paragraph{Implementation.}
The benchmark is constructed as an in-family generalization study on COLOR graphs. For each family, training graphs are drawn from the same family as the evaluation graphs, and performance is measured on held-out in-distribution (ID) graphs together with larger out-of-distribution (OOD) graphs from that family. The three encoders retained here are the strongest Cora-derived choices from the earlier citation experiments: gated\_gcn, gps\_gcn, and gps\_sage, each evaluated under both the v1 and v2 pipelines.

All runs use the same contrastive training and $k$-medoids decoding pipeline as in the Cora experiments. The v1 objective is supervised neighbor InfoNCE only, while v2 adds the tuned soft conflict term inherited from the later Cora study. Across all COLOR runs, we use 80 training epochs, common input dimension $d=64$, and mono-threshold target $0.05$.

\paragraph{Feature modes.}
Because COLOR graphs have no native node attributes, we compare two synthetic feature constructions.
\begin{itemize}
\item \textbf{Random:} random unit vectors $x_i \in \mathbb{R}^{64}$.
\item \textbf{Struct:} deterministic structural descriptors
\[
x_i = \phi_{\mathrm{struct}}(i),
\]
where $\phi_{\mathrm{struct}}(i)$ starts from four normalized node statistics---degree, local clustering coefficient, core number, and PageRank---and then expands them with a fixed nonlinear basis until the feature dimension reaches 64. Concretely, for a base scalar feature $x \in [0,1]$, the expansion cycles through
\[
x,\qquad x^2,\qquad \sin(\pi x),\qquad \cos(\pi x),
\]
repeating these transforms across the four base statistics until 64 coordinates are filled. Since the base four-dimensional structural vector is too small for the encoders used here, the expansion provides a richer deterministic input basis: the quadratic term adds a simple low-order nonlinearity, while the sinusoidal terms provide bounded nonlinear responses across the normalized structural range.

\end{itemize}

\paragraph{Notations.}
The main normalized metric is
\[
\rho = \frac{k}{\chi_{\mathrm{greedy}}},
\]
where $k$ is the smallest color count returned by the embedding-plus-$k$-medoids decoder and $\chi_{\mathrm{greedy}}$ is the greedy-coloring proxy used by the benchmark pipeline.

Table~\ref{tab:color-family-splits} summarizes the family splits used in the completed COLOR benchmark
, while Table~\ref{tab:color-best-by-family-app} reports the best overall configuration per family across all retained feature modes. The three per-family tables give the full breakdown for book, Myciel, and queen graphs. For cleaner reporting, when absolute-value variants are tested, we only report a single best instance of them across v1abs, v2abs, and the different architectures. These variants are competitive but do not appear as the best configuration in our setting. Baseline unsupervised methods are run on the ID test graphs or OOD test graphs.

\begin{table}[H]
\footnotesize
\centering
\caption{Family splits used in the completed COLOR benchmark.}
\label{tab:color-family-splits}
\small
\begin{tabularx}{\textwidth}{lcccccc}
\toprule
Family & Train graphs & ID graphs & OOD graphs & Train node range & ID node range & OOD node range \\
\midrule
Book   & 3 & 1 & 1 & 74--138 & 87 & 561 \\
Myciel & 2 & 1 & 1 & 47--95  & 191 & 383 \\
Queen  & 6 & 2 & 5 & 64--144 & 169--196 & 225--484 \\
\bottomrule
\end{tabularx}
\end{table}


\begin{table}[H]
\footnotesize
\centering
\caption{Best configuration per family, ranked by OOD hit-rate, then lower OOD $\rho$, then lower OOD Mono, then the analogous ID criteria. {Entries are mean $\pm$ standard deviation over ten random seeds.}}
\label{tab:color-best-by-family-app}
\setlength{\tabcolsep}{3.5pt}
\begin{tabular}{llllrccccc}
\toprule
Family & Feature mode & Pipeline & Encoder & $\rho_{\mathrm{ID}}$ & Mono$_{\mathrm{ID}}$ & Hit$_{\mathrm{ID}}$ & $\rho_{\mathrm{OOD}}$ & Mono$_{\mathrm{OOD}}$ & Hit$_{\mathrm{OOD}}$ \\
\midrule
Book & Random & v1 & gps\_sage & \textbf{1.41{\scriptsize $\pm$0.18}} & 0.0451{\scriptsize $\pm$0.0035} & 1.00{\scriptsize $\pm$0.00} & \textbf{1.08{\scriptsize $\pm$0.08}} & \textbf{0.0452{\scriptsize $\pm$0.0034}} & 1.00{\scriptsize $\pm$0.00} \\
Myciel & Struct & v1 & gps\_gcn & \textbf{1.45{\scriptsize $\pm$0.28}} & 0.0395{\scriptsize $\pm$0.0041} & 1.00{\scriptsize $\pm$0.00} & \textbf{1.54{\scriptsize $\pm$0.40}} & 0.0438{\scriptsize $\pm$0.0045} & 1.00{\scriptsize $\pm$0.00} \\
Queen & Random & v2 & gps\_sage & \textbf{0.90{\scriptsize $\pm$0.03}} & 0.0474{\scriptsize $\pm$0.0018} & 1.00{\scriptsize $\pm$0.00} & \textbf{0.76{\scriptsize $\pm$0.02}} & 0.0471{\scriptsize $\pm$0.0008} & 1.00{\scriptsize $\pm$0.00} \\
\bottomrule
\end{tabular}
\end{table}

\begin{table}[H]
\footnotesize
\centering
\caption{Book-family results. All book evaluation graphs have known exact chromatic number, and in this family the greedy proxy equals the exact value for the evaluation split, so $\rho$ is also the exact normalized color ratio. For baseline rows, Hit records the reported completion fraction. {Rows are mean $\pm$ standard deviation over ten random seeds.}}
\label{tab:color-book-detail}
\setlength{\tabcolsep}{4pt}
\renewcommand{\arraystretch}{1.05}
\begin{tabular}{llrrrrrr}
\toprule
Feature mode & Pipeline/Encoder & $\rho_{\mathrm{ID}}$ & Mono$_{\mathrm{ID}}$ & Hit$_{\mathrm{ID}}$ & $\rho_{\mathrm{OOD}}$ & Mono$_{\mathrm{OOD}}$ & Hit$_{\mathrm{OOD}}$ \\
\midrule
\multirow{7}{*}{Random}
& v1 / gps\_sage & 1.39{\scriptsize $\pm$0.20} & 0.0463{\scriptsize $\pm$0.0028} & 1.00{\scriptsize $\pm$0.00} & 1.05{\scriptsize $\pm$0.06} & \textbf{0.0454{\scriptsize $\pm$0.0031}} & 1.00{\scriptsize $\pm$0.00} \\

& v1 / gps\_gcn & 1.38{\scriptsize $\pm$0.18} & 0.0451{\scriptsize $\pm$0.0045} & 1.00{\scriptsize $\pm$0.00} & 1.13{\scriptsize $\pm$0.14} & 0.0459{\scriptsize $\pm$0.0028} & 1.00{\scriptsize $\pm$0.00} \\

& v1 / gated\_gcn & \textbf{1.31{\scriptsize $\pm$0.24}} & 0.0446{\scriptsize $\pm$0.0051} & 1.00{\scriptsize $\pm$0.00} & 1.05{\scriptsize $\pm$0.14} & 0.0469{\scriptsize $\pm$0.0020} & 1.00{\scriptsize $\pm$0.00} \\

& v2 / gps\_sage & 1.40{\scriptsize $\pm$0.24} & 0.0426{\scriptsize $\pm$0.0065} & 1.00{\scriptsize $\pm$0.00} & 1.04{\scriptsize $\pm$0.14} & 0.0456{\scriptsize $\pm$0.0029} & 1.00{\scriptsize $\pm$0.00} \\

& v2 / gps\_gcn & 1.38{\scriptsize $\pm$0.21} & 0.0461{\scriptsize $\pm$0.0031} & 1.00{\scriptsize $\pm$0.00} & 1.21{\scriptsize $\pm$0.10} & 0.0468{\scriptsize $\pm$0.0028} & 1.00{\scriptsize $\pm$0.00} \\

& v2 / gated\_gcn & 1.36{\scriptsize $\pm$0.18} & 0.0416{\scriptsize $\pm$0.0073} & 1.00{\scriptsize $\pm$0.00} & \textbf{1.00{\scriptsize $\pm$0.13}} & 0.0469{\scriptsize $\pm$0.0027} & 1.00{\scriptsize $\pm$0.00} \\

& v2abs / gated\_gcn & 1.48{\scriptsize $\pm$0.24} & \textbf{0.0414{\scriptsize $\pm$0.0054}} & 1.00{\scriptsize $\pm$0.00} & 1.28{\scriptsize $\pm$0.17} & 0.0466{\scriptsize $\pm$0.0018} & 1.00{\scriptsize $\pm$0.00} \\
\addlinespace[2pt]
\multirow{6}{*}{Struct}
& v1 / gated\_gcn & 2.79{\scriptsize $\pm$0.15} & 0.0744{\scriptsize $\pm$0.0034} & 0.00{\scriptsize $\pm$0.00} & 2.05{\scriptsize $\pm$0.31} & 0.1300{\scriptsize $\pm$0.0099} & 0.00{\scriptsize $\pm$0.00} \\

& v1 / gps\_gcn & 2.77{\scriptsize $\pm$0.15} & 0.0808{\scriptsize $\pm$0.0055} & 0.00{\scriptsize $\pm$0.00} & 2.19{\scriptsize $\pm$0.23} & 0.1079{\scriptsize $\pm$0.0116} & 0.00{\scriptsize $\pm$0.00} \\

& v1 / gps\_sage & 2.73{\scriptsize $\pm$0.20} & 0.0791{\scriptsize $\pm$0.0067} & 0.00{\scriptsize $\pm$0.00} & 2.23{\scriptsize $\pm$0.27} & 0.1244{\scriptsize $\pm$0.0095} & 0.00{\scriptsize $\pm$0.00} \\

& v2 / gated\_gcn & 2.62{\scriptsize $\pm$0.38} & 0.0749{\scriptsize $\pm$0.0031} & 0.00{\scriptsize $\pm$0.00} & 2.01{\scriptsize $\pm$0.27} & 0.1270{\scriptsize $\pm$0.0082} & 0.00{\scriptsize $\pm$0.00} \\

& v2 / gps\_gcn & 2.58{\scriptsize $\pm$0.33} & 0.0805{\scriptsize $\pm$0.0060} & 0.00{\scriptsize $\pm$0.00} & 2.10{\scriptsize $\pm$0.25} & 0.1203{\scriptsize $\pm$0.0070} & 0.00{\scriptsize $\pm$0.00} \\

& v2 / gps\_sage & 2.69{\scriptsize $\pm$0.17} & 0.0798{\scriptsize $\pm$0.0073} & 0.00{\scriptsize $\pm$0.00} & 2.20{\scriptsize $\pm$0.16} & 0.1102{\scriptsize $\pm$0.0157} & 0.00{\scriptsize $\pm$0.00} \\
\addlinespace[2pt]
\multirow{2}{*}{Baseline}
& PI-GNN & -- & -- & 0.00 & -- & -- & 0.00 \\
& full-GCN & 0.55 & 0.0394 & 1.00 & -- & -- & 0.00 \\
\bottomrule
\end{tabular}
\end{table}

\begin{table}[H]
\footnotesize
\centering
\caption{Myciel-family results. All evaluation graphs have known exact chromatic number, so $\rho$ is also the exact normalized color ratio. For baseline rows, Hit records the reported completion fraction. {Rows are mean $\pm$ standard deviation over ten random seeds; baseline rows are unchanged.}}
\label{tab:color-myciel-detail}
\setlength{\tabcolsep}{4pt}
\renewcommand{\arraystretch}{1.05}
\begin{tabular}{llrrrrrr}
\toprule
Feature mode & Pipeline/Encoder & $\rho_{\mathrm{ID}}$ & Mono$_{\mathrm{ID}}$ & Hit$_{\mathrm{ID}}$ & $\rho_{\mathrm{OOD}}$ & Mono$_{\mathrm{OOD}}$ & Hit$_{\mathrm{OOD}}$ \\
\midrule
\multirow{7}{*}{Random}
& v1 / gated\_gcn & 2.09{\scriptsize $\pm$0.13} & 0.0457{\scriptsize $\pm$0.0027} & 1.00{\scriptsize $\pm$0.00} & 2.21{\scriptsize $\pm$0.18} & 0.0464{\scriptsize $\pm$0.0036} & 1.00{\scriptsize $\pm$0.00} \\

& v1 / gps\_gcn & 1.62{\scriptsize $\pm$0.16} & 0.0427{\scriptsize $\pm$0.0030} & 1.00{\scriptsize $\pm$0.00} & 1.60{\scriptsize $\pm$0.17} & 0.0458{\scriptsize $\pm$0.0034} & 1.00{\scriptsize $\pm$0.00} \\

& v1 / gps\_sage & 2.26{\scriptsize $\pm$0.30} & 0.0451{\scriptsize $\pm$0.0041} & 1.00{\scriptsize $\pm$0.00} & 2.39{\scriptsize $\pm$0.31} & 0.0455{\scriptsize $\pm$0.0049} & 1.00{\scriptsize $\pm$0.00} \\

& v1abs / gps\_gcn & 2.66{\scriptsize $\pm$0.17} & 0.0456{\scriptsize $\pm$0.0030} & 1.00{\scriptsize $\pm$0.00} & 2.46{\scriptsize $\pm$0.15} & 0.0476{\scriptsize $\pm$0.0013} & 1.00{\scriptsize $\pm$0.00} \\

& v2 / gated\_gcn & 2.11{\scriptsize $\pm$0.34} & 0.0446{\scriptsize $\pm$0.0028} & 1.00{\scriptsize $\pm$0.00} & 2.16{\scriptsize $\pm$0.14} & 0.0448{\scriptsize $\pm$0.0037} & 1.00{\scriptsize $\pm$0.00} \\

& v2 / gps\_gcn & 1.65{\scriptsize $\pm$0.20} & 0.0434{\scriptsize $\pm$0.0036} & 1.00{\scriptsize $\pm$0.00} & 1.66{\scriptsize $\pm$0.17} & 0.0466{\scriptsize $\pm$0.0026} & 1.00{\scriptsize $\pm$0.00} \\

& v2 / gps\_sage & 1.52{\scriptsize $\pm$0.24} & 0.0439{\scriptsize $\pm$0.0060} & 1.00{\scriptsize $\pm$0.00} & 1.80{\scriptsize $\pm$0.22} & 0.0439{\scriptsize $\pm$0.0049} & 1.00{\scriptsize $\pm$0.00} \\
\addlinespace[2pt]
\multirow{6}{*}{Struct}
& v1 / gated\_gcn & 1.65{\scriptsize $\pm$0.16} & 0.0416{\scriptsize $\pm$0.0074} & 1.00{\scriptsize $\pm$0.00} & 1.51{\scriptsize $\pm$0.17} & 0.0438{\scriptsize $\pm$0.0055} & 1.00{\scriptsize $\pm$0.00} \\

& v1 / gps\_gcn & \textbf{1.24{\scriptsize $\pm$0.25}} & 0.0456{\scriptsize $\pm$0.0039} & 1.00{\scriptsize $\pm$0.00} & \textbf{1.49{\scriptsize $\pm$0.27}} & 0.0431{\scriptsize $\pm$0.0060} & 1.00{\scriptsize $\pm$0.00} \\

& v1 / gps\_sage & 1.61{\scriptsize $\pm$0.28} & \textbf{0.0389{\scriptsize $\pm$0.0071}} & 1.00{\scriptsize $\pm$0.00} & 1.81{\scriptsize $\pm$0.29} & 0.0453{\scriptsize $\pm$0.0039} & 1.00{\scriptsize $\pm$0.00} \\

& v2 / gated\_gcn & 1.59{\scriptsize $\pm$0.29} & 0.0397{\scriptsize $\pm$0.0063} & 1.00{\scriptsize $\pm$0.00} & 1.59{\scriptsize $\pm$0.34} & \textbf{0.0411{\scriptsize $\pm$0.0067}} & 1.00{\scriptsize $\pm$0.00} \\

& v2 / gps\_gcn & 1.51{\scriptsize $\pm$0.39} & 0.0444{\scriptsize $\pm$0.0040} & 1.00{\scriptsize $\pm$0.00} & 1.76{\scriptsize $\pm$0.30} & 0.0425{\scriptsize $\pm$0.0058} & 1.00{\scriptsize $\pm$0.00} \\

& v2 / gps\_sage & 1.93{\scriptsize $\pm$0.38} & 0.0425{\scriptsize $\pm$0.0061} & 1.00{\scriptsize $\pm$0.00} & 1.76{\scriptsize $\pm$0.36} & 0.0440{\scriptsize $\pm$0.0050} & 1.00{\scriptsize $\pm$0.00} \\
\addlinespace[2pt]
\multirow{2}{*}{Baseline}
& PI-GNN & -- & -- & 0.00 & -- & -- & 0.00 \\
& full-GCN & 0.50 & 0.0373 & 1.00 & -- & -- & 0.00 \\
\bottomrule
\end{tabular}
\end{table}

\begin{table}[H]
\footnotesize
\centering
\caption{Queen-family results. Here $\rho = k / \chi_{\mathrm{greedy}}$ is a proxy metric, because exact chromatic numbers are unavailable for most larger queen instances. For baseline rows, Hit records the reported completion fraction. {Rows are mean $\pm$ standard deviation over ten random seeds.}}
\label{tab:color-queen-detail}
\setlength{\tabcolsep}{4pt}
\renewcommand{\arraystretch}{1.05}
\begin{tabular}{llrrrrrr}
\toprule
Feature mode & Pipeline/Encoder & $\rho_{\mathrm{ID}}$ & Mono$_{\mathrm{ID}}$ & Hit$_{\mathrm{ID}}$ & $\rho_{\mathrm{OOD}}$ & Mono$_{\mathrm{OOD}}$ & Hit$_{\mathrm{OOD}}$ \\
\midrule
\multirow{7}{*}{Random}
& v1 / gated\_gcn & 1.05{\scriptsize $\pm$0.04} & 0.0469{\scriptsize $\pm$0.0015} & 1.00{\scriptsize $\pm$0.00} & 0.83{\scriptsize $\pm$0.02} & 0.0472{\scriptsize $\pm$0.0013} & 1.00{\scriptsize $\pm$0.00} \\

& v1 / gps\_gcn & 0.98{\scriptsize $\pm$0.05} & 0.0460{\scriptsize $\pm$0.0023} & 1.00{\scriptsize $\pm$0.00} & 0.79{\scriptsize $\pm$0.02} & 0.0478{\scriptsize $\pm$0.0008} & 1.00{\scriptsize $\pm$0.00} \\

& v1 / gps\_sage & 1.04{\scriptsize $\pm$0.05} & 0.0477{\scriptsize $\pm$0.0011} & 1.00{\scriptsize $\pm$0.00} & 0.83{\scriptsize $\pm$0.02} & 0.0471{\scriptsize $\pm$0.0008} & 1.00{\scriptsize $\pm$0.00} \\

& v1abs / gated\_gcn & 1.09{\scriptsize $\pm$0.05} & 0.0475{\scriptsize $\pm$0.0008} & 1.00{\scriptsize $\pm$0.00} & 0.87{\scriptsize $\pm$0.03} & 0.0475{\scriptsize $\pm$0.0007} & 1.00{\scriptsize $\pm$0.00} \\

& v2 / gated\_gcn & 1.03{\scriptsize $\pm$0.04} & 0.0475{\scriptsize $\pm$0.0011} & 1.00{\scriptsize $\pm$0.00} & 0.83{\scriptsize $\pm$0.02} & 0.0472{\scriptsize $\pm$0.0013} & 1.00{\scriptsize $\pm$0.00} \\

& v2 / gps\_gcn & \textbf{0.97{\scriptsize $\pm$0.05}} & 0.0471{\scriptsize $\pm$0.0015} & 1.00{\scriptsize $\pm$0.00} & 0.81{\scriptsize $\pm$0.01} & \textbf{0.0470{\scriptsize $\pm$0.0011}} & 1.00{\scriptsize $\pm$0.00} \\

& v2 / gps\_sage & \textbf{0.97{\scriptsize $\pm$0.05}} & \textbf{0.0459{\scriptsize $\pm$0.0021}} & 1.00{\scriptsize $\pm$0.00} & \textbf{0.77{\scriptsize $\pm$0.02}} & 0.0473{\scriptsize $\pm$0.0011} & 1.00{\scriptsize $\pm$0.00} \\
\addlinespace[2pt]
\multirow{6}{*}{Struct}
& v1 / gated\_gcn & 1.29{\scriptsize $\pm$0.18} & 0.2809{\scriptsize $\pm$0.0471} & 0.00{\scriptsize $\pm$0.00} & 0.88{\scriptsize $\pm$0.10} & 0.2831{\scriptsize $\pm$0.0425} & 0.00{\scriptsize $\pm$0.00} \\

& v1 / gps\_gcn & 1.73{\scriptsize $\pm$0.05} & 0.0885{\scriptsize $\pm$0.0010} & 0.00{\scriptsize $\pm$0.00} & 1.31{\scriptsize $\pm$0.01} & 0.0957{\scriptsize $\pm$0.0010} & 0.00{\scriptsize $\pm$0.00} \\

& v1 / gps\_sage & 1.73{\scriptsize $\pm$0.04} & 0.0903{\scriptsize $\pm$0.0051} & 0.00{\scriptsize $\pm$0.00} & 1.32{\scriptsize $\pm$0.02} & 0.0961{\scriptsize $\pm$0.0029} & 0.00{\scriptsize $\pm$0.00} \\

& v2 / gated\_gcn & 1.12{\scriptsize $\pm$0.28} & 0.2548{\scriptsize $\pm$0.0577} & 0.00{\scriptsize $\pm$0.00} & 0.93{\scriptsize $\pm$0.17} & 0.2620{\scriptsize $\pm$0.0553} & 0.00{\scriptsize $\pm$0.00} \\

& v2 / gps\_gcn & 1.75{\scriptsize $\pm$0.03} & 0.0886{\scriptsize $\pm$0.0011} & 0.00{\scriptsize $\pm$0.00} & 1.32{\scriptsize $\pm$0.01} & 0.0961{\scriptsize $\pm$0.0010} & 0.00{\scriptsize $\pm$0.00} \\

& v2 / gps\_sage & 1.43{\scriptsize $\pm$0.19} & 0.1436{\scriptsize $\pm$0.0527} & 0.00{\scriptsize $\pm$0.00} & 1.10{\scriptsize $\pm$0.13} & 0.1572{\scriptsize $\pm$0.0565} & 0.00{\scriptsize $\pm$0.00} \\
\addlinespace[2pt]
\multirow{2}{*}{Baseline}
& PI-GNN & -- & -- & 0.00 & -- & -- & 0.00 \\
& full-GCN & -- & -- & 0.00 & -- & -- & 0.00 \\
\bottomrule
\end{tabular}
\end{table}

\paragraph{Results.}
{The completed benchmark gives a clear family-dependent picture. For book graphs, random-feature configurations are the viable regime: the selected v1/GPS-SAGE row has OOD $\rho=1.08\pm0.08$ with Mono$_{\mathrm{OOD}}=0.0452\pm0.0034$, while the structural-feature rows miss the hit threshold. For Myciel graphs, structural v1/GPS-GCN is the selected configuration, with ID $\rho=1.45\pm0.28$ and OOD $\rho=1.54\pm0.40$ while keeping hit-rate $1.00$. Queen graphs give the strongest OOD generalization case: v2/GPS-SAGE with random features gives OOD $\rho=0.76\pm0.02$ and Mono$_{\mathrm{OOD}}=0.0471\pm0.0008$. The abs-random rows are feasible low-conflict alternatives, but they are not the best family-level configurations.}

{Taken together, these results support three conclusions. First, the best feature construction is family-dependent: book and queen favor random features, whereas Myciel favors structural features. Second, within the random-initialized family, the regular v1/v2 pipelines are stronger than their abs-random counterparts on the reported family-level comparisons, while the abs-random rows serve as low-conflict robustness checks. Third, the external baselines are only partially informative in this benchmark: when they finish an ID family instance they can be very strong, but their repeated OOD failures mean that they do not overturn the learned-model comparison.}

%% file: appendix/main_cycle_experiments.tex
\section{Cycle experiments}

\paragraph{Implementation.}
This section studies a cycle generalization benchmark in which training is performed only on cycles $C_n$ with $50 \leq n \leq 200$, using random node initialization. The learned methods are GatedGCN, UnitaryMP, GPS-GCN, and GPS-SAGE, each run in a v1abs-style and v2abs-style pipeline. The test set contains $40$ small graphs and $20$ large cycles. The small subset contains $20$ cycles $C_{20},\dots,C_{39}$ and $20$ non-cycle graphs, while the large subset contains $C_{7000},\dots,C_{7019}$. We additionally evaluate the PI-GNN baseline of \citet{schuetz2022pignn} and the warm-start ColoringGNN baseline of \citet{vanderbush2026warmstart}. Both baselines are run on the same test graphs with a hard per-graph runtime budget of $900$ seconds. Experiments not completed within this time budget are denoted as ``timeout'' in the report.

\paragraph{Combined averages for the learned methods.}
{Table~\ref{tab:avg-v1} and Table~\ref{tab:avg-v2} summarize the 40-graph small suite and the 20 large-cycle stress suite under ten random seeds.}

\begin{table}[htbp]
\centering
\footnotesize
\setlength{\tabcolsep}{4pt}
\renewcommand{\arraystretch}{1.12}
\caption{Combined averages for the v1abs-style random-feature benchmark. {Entries are mean $\pm$ standard deviation over ten random seeds.}}
\label{tab:avg-v1}
\begin{tabular}{lcccccc}
\toprule
Encoder & Small $k/\chi$ & Small Mono & Small Hit & Large $k/\chi$ & Large Mono & Large Hit \\
\midrule
GatedGCN & \textbf{1.48{\scriptsize $\pm$0.10}} & 0.0216{\scriptsize $\pm$0.0026} & 0.89{\scriptsize $\pm$0.01} & \textbf{1.42{\scriptsize $\pm$0.22}} & 0.0133{\scriptsize $\pm$0.0100} & 1.00{\scriptsize $\pm$0.00} \\
UnitaryMP & \textbf{1.48{\scriptsize $\pm$0.06}} & 0.0214{\scriptsize $\pm$0.0024} & 0.91{\scriptsize $\pm$0.01} & 1.67{\scriptsize $\pm$0.00} & \textbf{0.0096{\scriptsize $\pm$0.0012}} & 1.00{\scriptsize $\pm$0.00} \\
GPS-GCN & 1.64{\scriptsize $\pm$0.11} & \textbf{0.0213{\scriptsize $\pm$0.0019}} & \textbf{0.92{\scriptsize $\pm$0.02}} & 1.93{\scriptsize $\pm$0.19} & 0.0270{\scriptsize $\pm$0.0082} & 1.00{\scriptsize $\pm$0.00} \\
GPS-SAGE & 1.74{\scriptsize $\pm$0.11} & 0.0238{\scriptsize $\pm$0.0021} & 0.90{\scriptsize $\pm$0.02} & 2.01{\scriptsize $\pm$0.19} & 0.0261{\scriptsize $\pm$0.0110} & 1.00{\scriptsize $\pm$0.00} \\
\bottomrule
\end{tabular}
\end{table}

\begin{table}[htbp]
\centering
\footnotesize
\setlength{\tabcolsep}{4pt}
\renewcommand{\arraystretch}{1.12}
\caption{Combined averages for the v2abs-style random-feature benchmark. {Entries are mean $\pm$ standard deviation over ten random seeds.}}
\label{tab:avg-v2}
\begin{tabular}{lcccccc}
\toprule
Encoder & Small $k/\chi$ & Small Mono & Small Hit & Large $k/\chi$ & Large Mono & Large Hit \\
\midrule
GatedGCN & \textbf{1.44{\scriptsize $\pm$0.06}} & 0.0224{\scriptsize $\pm$0.0020} & 0.90{\scriptsize $\pm$0.01} & \textbf{1.62{\scriptsize $\pm$0.13}} & \textbf{0.0058{\scriptsize $\pm$0.0076}} & 1.00{\scriptsize $\pm$0.00} \\
UnitaryMP & \textbf{1.44{\scriptsize $\pm$0.03}} & 0.0226{\scriptsize $\pm$0.0025} & 0.92{\scriptsize $\pm$0.02} & \textbf{1.62{\scriptsize $\pm$0.13}} & 0.0131{\scriptsize $\pm$0.0073} & 1.00{\scriptsize $\pm$0.00} \\
GPS-GCN & 1.87{\scriptsize $\pm$0.11} & 0.0235{\scriptsize $\pm$0.0021} & 0.90{\scriptsize $\pm$0.02} & 2.20{\scriptsize $\pm$0.28} & 0.0269{\scriptsize $\pm$0.0107} & 1.00{\scriptsize $\pm$0.00} \\
GPS-SAGE & 1.57{\scriptsize $\pm$0.10} & \textbf{0.0196{\scriptsize $\pm$0.0028}} & \textbf{0.94{\scriptsize $\pm$0.03}} & 1.66{\scriptsize $\pm$0.15} & 0.0199{\scriptsize $\pm$0.0101} & 1.00{\scriptsize $\pm$0.00} \\
\bottomrule
\end{tabular}
\end{table}

\paragraph{Comparison with benchmarks.}
{Table~\ref{tab:baseline-summary} compares a balanced learned reference recipe against the two baselines. The learned reference is v1abs + GatedGCN, which has the best large-cycle ratio among the reported learned recipes and competitive small-suite behavior while solving all large cycles under the time cap.}

\begin{table}[htbp]
\centering
\footnotesize
\setlength{\tabcolsep}{4pt}
\renewcommand{\arraystretch}{1.12}
\caption{Aggregate comparison against PI-GNN and the warm-start ColoringGNN baseline. {Learned rows are mean $\pm$ standard deviation over ten random seeds; external baseline rows are unchanged.}}
\label{tab:baseline-summary}
\begin{tabular}{llcccccc}
\toprule
Method & Split & Completed & Timeouts & $k/\chi$ & Mono & Hit & Total runtime (s) \\
\midrule
v1abs + GatedGCN & Small 40 graphs & 40 & 0 & 1.48{\scriptsize $\pm$0.10} & 0.0216{\scriptsize $\pm$0.0026} & 0.89{\scriptsize $\pm$0.01} & 6.82{\scriptsize $\pm$0.87} \\
PI-GNN  & Small 40 graphs & 39 & 1 & 1.009 & 0.0126 & 1.000 & 4701.1 \\
Warm-start ColoringGNN & Small 40 graphs & 40 & 0 & 0.987 & 0.0141 & 1.000 & 2282.7 \\
\midrule
v1abs + GatedGCN & Large 20 cycles & 20 & 0 & 1.42{\scriptsize $\pm$0.22} & 0.0133{\scriptsize $\pm$0.0100} & 1.00{\scriptsize $\pm$0.00} & 6.38{\scriptsize $\pm$0.39} \\
PI-GNN  & Large 20 cycles & 0 & 20 & -- & -- & -- & 18000.0 \\
Warm-start ColoringGNN & Large 20 cycles & 0 & 20 & -- & -- & -- & 18000.0 \\
\bottomrule
\end{tabular}
\end{table}

\paragraph{Runtime.}
For the learned methods, training time is reported once per recipe and test-time runtime is the average decode time per graph. For PI-GNN and the warm-start ColoringGNN baseline, the runtime table reports per-graph solve time because these baselines are evaluated graph-by-graph under the same $k$ sweep.

\begin{table}[htbp]
\centering
\footnotesize
\setlength{\tabcolsep}{4pt}
\renewcommand{\arraystretch}{1.12}
\caption{Runtime summary for the learned random-feature cycle models. {Entries are mean $\pm$ standard deviation over ten random seeds.}}
\label{tab:our-runtime}
\begin{tabular}{llccc}
\toprule
Pipeline & Encoder & Train (s) & Avg small test (s) & Avg large test (s) \\
\midrule
v1abs & GatedGCN & \textbf{5.88{\scriptsize $\pm$0.36}} & 0.0236{\scriptsize $\pm$0.0199} & \textbf{0.0253{\scriptsize $\pm$0.0075}}  \\
v1abs & UnitaryMP & 7.57{\scriptsize $\pm$0.05} & 0.0162{\scriptsize $\pm$0.0008} & 0.0345{\scriptsize $\pm$0.0006} \\
v1abs & GPS-GCN & 12.12{\scriptsize $\pm$0.07} & 0.0187{\scriptsize $\pm$0.0014} & 0.0640{\scriptsize $\pm$0.0076} \\
v1abs & GPS-SAGE & 9.17{\scriptsize $\pm$0.05} & 0.0190{\scriptsize $\pm$0.0014} & 0.0537{\scriptsize $\pm$0.0079} \\
v2abs & GatedGCN & 10.38{\scriptsize $\pm$0.06} & \textbf{0.0147{\scriptsize $\pm$0.0007}} & 0.0333{\scriptsize $\pm$0.0049} \\
v2abs & UnitaryMP & 12.36{\scriptsize $\pm$0.07} & 0.0155{\scriptsize $\pm$0.0004} & 0.0329{\scriptsize $\pm$0.0044} \\
v2abs & GPS-GCN & 13.96{\scriptsize $\pm$0.09} & 0.0200{\scriptsize $\pm$0.0016} & 0.0603{\scriptsize $\pm$0.0107}  \\
v2abs & GPS-SAGE & 17.00{\scriptsize $\pm$0.09} & 0.0179{\scriptsize $\pm$0.0009} & 0.0446{\scriptsize $\pm$0.0056} \\
\bottomrule
\end{tabular}
\end{table}

\begin{table}[htbp]
\centering
\footnotesize
\setlength{\tabcolsep}{4pt}
\renewcommand{\arraystretch}{1.12}
\caption{Runtime summary for PI-GNN and the warm-start ColoringGNN baseline under a 15-minute per-graph budget.}
\label{tab:baseline-runtime}
\begin{tabular}{llcccc}
\toprule
Method & Split & Avg all (s) & Avg completed (s) & Max (s) & Timeouts \\
\midrule
PI-GNN  & Small 40 graphs & 117.53 & 97.47 & 900.0 & 1 \\
PI-GNN  & Large 20 cycles & 900.00 & -- & 900.0 & 20 \\
Warm-start ColoringGNN & Small 40 graphs & 57.07 & 57.07 & 415.5 & 0 \\
Warm-start ColoringGNN & Large 20 cycles & 900.00 & -- & 900.0 & 20 \\
\bottomrule
\end{tabular}
\end{table}

\paragraph{Results.}
{The next per-graph comparison focuses on two representative learned configurations together with the two external baselines. We compare the balanced v1abs + GatedGCN reference recipe, the v2abs + GPS-SAGE recipe used in the main-text per-graph table, PI-GNN, and the warm-start ColoringGNN baseline on each graph. Learned cells report seed means and standard deviations, while external baseline cells are single-run values; timed-out baseline graphs are shown as ``timeout''.}

\begin{table}[p]
\centering
\footnotesize
\setlength{\tabcolsep}{4pt}
\renewcommand{\arraystretch}{1.08}
\caption{Per-graph results on the 20 small cycle graphs. Cells are reported as $k$ / Mono. {Learned-method cells are mean $\pm$ standard deviation over ten random seeds; baseline cells are unchanged.}}
\label{tab:small-cycles-per-graph}
\begin{tabular}{lrrcccc}
\toprule
Graph & $n$ & $\chi$ & v1abs + GatedGCN & v2abs + GPS-SAGE & PI-GNN & Warm-start ColoringGNN \\
\midrule
C\_20 & 20 & 2 & 3.4{\scriptsize $\pm$0.5} / 0.0000{\scriptsize $\pm$0.0000} & 4.0{\scriptsize $\pm$0.8} / 0.0000{\scriptsize $\pm$0.0000} & 3 / 0.0000 & 3 / 0.0000 \\
C\_21 & 21 & 3 & 3.3{\scriptsize $\pm$0.5} / 0.0190{\scriptsize $\pm$0.0246} & 4.3{\scriptsize $\pm$2.1} / 0.0143{\scriptsize $\pm$0.0230} & 3 / 0.0000 & 2 / 0.0476 \\
C\_22 & 22 & 2 & 3.4{\scriptsize $\pm$0.5} / 0.0227{\scriptsize $\pm$0.0240} & 4.3{\scriptsize $\pm$0.7} / 0.0182{\scriptsize $\pm$0.0235} & 2 / 0.0000 & 2 / 0.0000 \\
C\_23 & 23 & 3 & 3.2{\scriptsize $\pm$0.4} / 0.0217{\scriptsize $\pm$0.0229} & 3.8{\scriptsize $\pm$0.8} / 0.0087{\scriptsize $\pm$0.0183} & 2 / 0.0435 & 2 / 0.0435 \\
C\_24 & 24 & 2 & 3.3{\scriptsize $\pm$0.5} / 0.0083{\scriptsize $\pm$0.0176} & 4.1{\scriptsize $\pm$1.1} / 0.0292{\scriptsize $\pm$0.0201} & 3 / 0.0000 & 3 / 0.0000 \\
C\_25 & 25 & 3 & 3.2{\scriptsize $\pm$0.4} / 0.0160{\scriptsize $\pm$0.0207} & 4.1{\scriptsize $\pm$0.6} / 0.0120{\scriptsize $\pm$0.0193} & 2 / 0.0400 & 3 / 0.0000 \\
C\_26 & 26 & 2 & 3.5{\scriptsize $\pm$0.7} / 0.0115{\scriptsize $\pm$0.0186} & 3.7{\scriptsize $\pm$0.9} / 0.0192{\scriptsize $\pm$0.0203} & 3 / 0.0000 & 3 / 0.0000 \\
C\_27 & 27 & 3 & 3.6{\scriptsize $\pm$1.0} / 0.0185{\scriptsize $\pm$0.0195} & 4.5{\scriptsize $\pm$1.3} / 0.0074{\scriptsize $\pm$0.0156} & 3 / 0.0000 & 2 / 0.0370 \\
C\_28 & 28 & 2 & 3.4{\scriptsize $\pm$0.5} / 0.0107{\scriptsize $\pm$0.0173} & 3.9{\scriptsize $\pm$0.6} / 0.0143{\scriptsize $\pm$0.0184} & 3 / 0.0000 & 3 / 0.0000 \\
C\_29 & 29 & 3 & 3.3{\scriptsize $\pm$0.5} / 0.0034{\scriptsize $\pm$0.0109} & 4.1{\scriptsize $\pm$1.2} / 0.0241{\scriptsize $\pm$0.0167} & 2 / 0.0345 & 3 / 0.0000 \\
C\_30 & 30 & 2 & 3.3{\scriptsize $\pm$0.5} / 0.0067{\scriptsize $\pm$0.0141} & 3.8{\scriptsize $\pm$0.6} / 0.0133{\scriptsize $\pm$0.0172} & 3 / 0.0000 & 3 / 0.0000 \\
C\_31 & 31 & 3 & 3.3{\scriptsize $\pm$0.5} / 0.0032{\scriptsize $\pm$0.0102} & 4.0{\scriptsize $\pm$1.1} / 0.0032{\scriptsize $\pm$0.0102} & 3 / 0.0000 & 3 / 0.0000 \\
C\_32 & 32 & 2 & 3.6{\scriptsize $\pm$0.8} / 0.0125{\scriptsize $\pm$0.0161} & 4.2{\scriptsize $\pm$0.8} / 0.0063{\scriptsize $\pm$0.0132} & 3 / 0.0000 & 2 / 0.0000 \\
C\_33 & 33 & 3 & 3.4{\scriptsize $\pm$0.7} / 0.0121{\scriptsize $\pm$0.0156} & 4.1{\scriptsize $\pm$0.9} / 0.0121{\scriptsize $\pm$0.0156} & 2 / 0.0303 & 2 / 0.0303 \\
C\_34 & 34 & 2 & 3.7{\scriptsize $\pm$0.8} / 0.0059{\scriptsize $\pm$0.0124} & 3.7{\scriptsize $\pm$0.5} / 0.0118{\scriptsize $\pm$0.0152} & 2 / 0.0000 & 3 / 0.0000 \\
C\_35 & 35 & 3 & 3.6{\scriptsize $\pm$0.7} / 0.0086{\scriptsize $\pm$0.0138} & 4.0{\scriptsize $\pm$0.8} / 0.0114{\scriptsize $\pm$0.0148} & 3 / 0.0000 & 2 / 0.0286 \\
C\_36 & 36 & 2 & 3.6{\scriptsize $\pm$0.5} / 0.0111{\scriptsize $\pm$0.0143} & 4.0{\scriptsize $\pm$0.5} / 0.0167{\scriptsize $\pm$0.0143} & 3 / 0.0000 & 3 / 0.0000 \\
C\_37 & 37 & 3 & 3.4{\scriptsize $\pm$0.5} / 0.0027{\scriptsize $\pm$0.0085} & 4.3{\scriptsize $\pm$1.1} / 0.0081{\scriptsize $\pm$0.0131} & 3 / 0.0000 & 2 / 0.0270 \\
C\_38 & 38 & 2 & 3.5{\scriptsize $\pm$0.5} / 0.0079{\scriptsize $\pm$0.0127} & 3.9{\scriptsize $\pm$1.0} / 0.0132{\scriptsize $\pm$0.0139} & 3 / 0.0000 & 3 / 0.0000 \\
C\_39 & 39 & 3 & 3.4{\scriptsize $\pm$0.5} / 0.0103{\scriptsize $\pm$0.0132} & 4.5{\scriptsize $\pm$1.1} / 0.0077{\scriptsize $\pm$0.0124} & 3 / 0.0000 & 3 / 0.0000 \\
\bottomrule
\end{tabular}
\end{table}

\begin{table}[p]
\centering
\footnotesize
\setlength{\tabcolsep}{4pt}
\renewcommand{\arraystretch}{1.08}
\caption{Per-graph results on the 20 non-cycle graphs. Cells are reported as $k$ / Mono. A timed-out baseline graph is shown as ``timeout''. {Learned-method cells are mean $\pm$ standard deviation over ten random seeds; baseline cells are unchanged.}}
\label{tab:small-noncycle-per-graph}
\begin{tabular}{lrrcccc}
\toprule
Graph & $n$ & $\chi$ & v1abs + GatedGCN & v2abs + GPS-SAGE & PI-GNN & Warm-start ColoringGNN \\
\midrule
K\_5 & 5 & 5 & 5.0{\scriptsize $\pm$0.0} / 0.0000{\scriptsize $\pm$0.0000} & 5.0{\scriptsize $\pm$0.0} / 0.0000{\scriptsize $\pm$0.0000} & 5 / 0.0000 & 5 / 0.0000 \\
K\_8 & 8 & 8 & 7.0{\scriptsize $\pm$0.0} / 0.0357{\scriptsize $\pm$0.0000} & 7.0{\scriptsize $\pm$0.0} / 0.0357{\scriptsize $\pm$0.0000} & 8 / 0.0357 & 7 / 0.0357 \\
K\_10 & 10 & 10 & 8.1{\scriptsize $\pm$0.3} / 0.0422{\scriptsize $\pm$0.0070} & 8.3{\scriptsize $\pm$0.5} / 0.0378{\scriptsize $\pm$0.0107} & 8 / 0.0444 & 8 / 0.0444 \\
K\_\{5,5\} & 10 & 2 & 3.0{\scriptsize $\pm$1.9} / 0.0040{\scriptsize $\pm$0.0126} & 3.3{\scriptsize $\pm$1.6} / 0.0040{\scriptsize $\pm$0.0126} & 2 / 0.0000 & 2 / 0.0000 \\
K\_\{4,8\} & 12 & 2 & 4.7{\scriptsize $\pm$2.9} / 0.0063{\scriptsize $\pm$0.0132} & 2.8{\scriptsize $\pm$0.6} / 0.0000{\scriptsize $\pm$0.0000} & 2 / 0.0000 & 2 / 0.0000 \\
K\_\{3,10\} & 13 & 2 & 4.2{\scriptsize $\pm$2.4} / 0.0067{\scriptsize $\pm$0.0141} & 3.4{\scriptsize $\pm$0.8} / 0.0033{\scriptsize $\pm$0.0105} & 2 / 0.0000 & 2 / 0.0000 \\
W\_11 & 11 & 3 & 7.1{\scriptsize $\pm$2.0} / 0.0000{\scriptsize $\pm$0.0000} & 6.6{\scriptsize $\pm$1.6} / 0.0000{\scriptsize $\pm$0.0000} & 3 / 0.0000 & 3 / 0.0000 \\
W\_12 & 12 & 4 & 5.0{\scriptsize $\pm$1.3} / 0.0273{\scriptsize $\pm$0.0235} & 4.6{\scriptsize $\pm$1.3} / 0.0318{\scriptsize $\pm$0.0220} & 3 / 0.0455 & 3 / 0.0455 \\
W\_13 & 13 & 3 & 5.6{\scriptsize $\pm$2.2} / 0.0208{\scriptsize $\pm$0.0220} & 5.5{\scriptsize $\pm$1.4} / 0.0292{\scriptsize $\pm$0.0201} & 3 / 0.0000 & 3 / 0.0000 \\
W\_14 & 14 & 4 & 5.1{\scriptsize $\pm$0.9} / 0.0231{\scriptsize $\pm$0.0199} & 5.9{\scriptsize $\pm$2.5} / 0.0269{\scriptsize $\pm$0.0186} & 3 / 0.0385 & 3 / 0.0385 \\
Petersen & 10 & 3 & 4.6{\scriptsize $\pm$1.3} / 0.0000{\scriptsize $\pm$0.0000} & 5.2{\scriptsize $\pm$1.7} / 0.0000{\scriptsize $\pm$0.0000} & 3 / 0.0000 & 3 / 0.0000 \\
Icosahedral & 12 & 4 & 7.0{\scriptsize $\pm$1.4} / 0.0300{\scriptsize $\pm$0.0105} & 6.4{\scriptsize $\pm$1.4} / 0.0267{\scriptsize $\pm$0.0141} & 4 / 0.0333 & 5 / 0.0000 \\
KG(7,2) & 21 & 5 & 8.9{\scriptsize $\pm$1.0} / 0.0610{\scriptsize $\pm$0.0225} & 8.6{\scriptsize $\pm$1.2} / 0.0457{\scriptsize $\pm$0.0040} & 4 / 0.0286 & 4 / 0.0286 \\
KG(9,3) & 84 & 5 & 8.7{\scriptsize $\pm$1.1} / 0.0969{\scriptsize $\pm$0.0147} & 9.6{\scriptsize $\pm$0.5} / 0.0661{\scriptsize $\pm$0.0133} & 4 / 0.0119 & 4 / 0.0119 \\
KG(10,3) & 120 & 6 & 10.0{\scriptsize $\pm$0.0} / 0.1071{\scriptsize $\pm$0.0220} & 9.2{\scriptsize $\pm$1.0} / 0.0924{\scriptsize $\pm$0.0160} & timeout & 4 / 0.0333 \\
Mycielski(C5)\^{}0 & 5 & 3 & 3.0{\scriptsize $\pm$0.0} / 0.0000{\scriptsize $\pm$0.0000} & 3.4{\scriptsize $\pm$0.5} / 0.0000{\scriptsize $\pm$0.0000} & 3 / 0.0000 & 3 / 0.0000 \\
Mycielski(C5)\^{}1 & 11 & 4 & 5.9{\scriptsize $\pm$2.2} / 0.0050{\scriptsize $\pm$0.0158} & 4.5{\scriptsize $\pm$1.0} / 0.0000{\scriptsize $\pm$0.0000} & 3 / 0.0500 & 3 / 0.0500 \\
Mycielski(C5)\^{}2 & 23 & 5 & 7.1{\scriptsize $\pm$1.6} / 0.0310{\scriptsize $\pm$0.0129} & 6.6{\scriptsize $\pm$2.1} / 0.0352{\scriptsize $\pm$0.0137} & 4 / 0.0141 & 4 / 0.0141 \\
Mycielski(C5)\^{}3 & 47 & 6 & 9.2{\scriptsize $\pm$1.0} / 0.0665{\scriptsize $\pm$0.0182} & 9.5{\scriptsize $\pm$0.5} / 0.0449{\scriptsize $\pm$0.0100} & 4 / 0.0169 & 4 / 0.0254 \\
Mycielski(C5)\^{}4 & 95 & 7 & 9.1{\scriptsize $\pm$0.9} / 0.0873{\scriptsize $\pm$0.0167} & 9.4{\scriptsize $\pm$0.5} / 0.0519{\scriptsize $\pm$0.0155} & 4 / 0.0225 & 4 / 0.0212 \\
\bottomrule
\end{tabular}
\end{table}

{Overall, v1abs + GatedGCN is the most balanced learned recipe on the reported cycle-trained benchmark. Across ten seeds it reaches small-suite $k/\chi=1.48\pm0.10$, Mono $=0.0216\pm0.0026$, and hit-rate $0.89\pm0.01$; on the 20 large cycles it reaches $k/\chi=1.42\pm0.22$, Mono $=0.0133\pm0.0100$, and hit-rate $1.00\pm0.00$. The v2abs GatedGCN and UnitaryMP rows have slightly lower small-suite ratios ($1.44$), but they use more colors on the large-cycle stress test. On the small 40-graph subset, PI-GNN and the warm-start ColoringGNN baseline are more color-efficient, but both external baselines time out on all large-cycle instances under the 15-minute cap. The runtime split is therefore sharp: learned models require a short one-time training phase and then decode large cycles quickly, whereas the direct baselines become infeasible on the 7000-node cycle stress test.}

\section{Details on Experimental Setups}

All experiments were implemented in Python. Neural-network models were built using PyTorch (v2.5.1+cu121) and PyTorch Geometric (v2.7.0), with graph construction and synthetic benchmark generation handled through NetworkX (v3.4.2). Numerical computation and data processing used NumPy (v2.2.6), SciPy (v1.15.2), and pandas (v2.3.3). For the DGL-based baseline reruns we used DGL (v1.1.3). Models were trained using the optimization and initialization routines provided by the default PyTorch stack, with AdamW \cite{loshchilov2019decoupled} used as the optimizer in our learned pipelines. 

We ran the experiments on a local compute server, and the hardware configuration used is summarized in Table~\ref{tab:hardware-specs}.

\begin{table}[H]
\centering
\caption{Hardware specifications.}
\label{tab:hardware-specs}
\begin{tabular}{ll}
\toprule
Component & Specification \\
\midrule
Architecture & x86\_64 \\
OS & Rocky Linux 8.10 (Green Obsidian) \\
CPU & AMD EPYC 9374F 32-Core Processor ($\times 2$) \\
GPU & NVIDIA H200 \\
GPU memory & 143771 MiB (approximately 141 GB) \\
RAM & 1.5 TiB \\
\bottomrule
\end{tabular}
\end{table}

The experiments use three groups of datasets. First, the citation-graph experiments use the Cora, CiteSeer, and PubMed citation graphs, loaded through the standard Planetoid benchmark interface in PyTorch Geometric \cite{yang2016revisiting}. Second, the COLOR-family experiments use book, Myciel, and queen graphs; the book instances are taken from the CMU COLOR benchmark collection \cite{johnson1996cliques}, while the Myciel and queen families are generated programmatically from their standard graph constructions. Third, the cycle experiments use synthetic cycle graphs for training and held-out evaluation, together with additional manually specified benchmark graphs such as complete graphs, complete bipartite graphs, wheel graphs, Kneser graphs, Mycielski graphs, Petersen, Icosahedral, and very large cycles; these graphs are generated through NetworkX or through explicit graph-construction utilities in our codebase.

Table~\ref{tab:software-licenses} lists the licenses of the main software libraries used in the experiments.

\begin{table}[H]
\centering
\caption{Main software licenses.}
\label{tab:software-licenses}
\begin{tabular}{ll}
\toprule
Software & License \\
\midrule
PyTorch & BSD-3-Clause \\
PyTorch Geometric & MIT \\
NetworkX & BSD-3-Clause \\
NumPy & BSD-3-Clause \\
SciPy & BSD-3-Clause \\
pandas & BSD-3-Clause \\
DGL & Apache-2.0 \\
\bottomrule
\end{tabular}
\end{table}

\section{LLM Usage Disclosure}
We used an LLM to help with code writing and with polishing the paper text.

%% file: references.bib
@ARTICLE{li2020nonsmooth,
  author={Li, Jiajin and So, Anthony Man-Cho and Ma, Wing-Kin},
  journal={IEEE Signal Processing Magazine}, 
  title={Understanding Notions of Stationarity in Nonsmooth Optimization: A Guided Tour of Various Constructions of Subdifferential for Nonsmooth Functions}, 
  year={2020},
  volume={37},
  number={5},
  pages={18-31},
  doi={10.1109/MSP.2020.3003845}}

@book{rockafellar1998variational,
  author    = {Rockafellar, R. Tyrrell and Wets, Roger J.-B.},
  title     = {Variational Analysis},
  series    = {Grundlehren der Mathematischen Wissenschaften},
  volume    = {317},
  publisher = {Springer},
  address   = {Berlin, Heidelberg},
  year      = {1998},
  doi       = {10.1007/978-3-642-02431-3}
}

@article{hosseini2013nonsmooth,
  author  = {Hosseini, Seyedehsomayeh and Pouryayevali, Mohammad Reza},
  title   = {Nonsmooth Optimization Techniques on Riemannian Manifolds},
  journal = {Journal of Optimization Theory and Applications},
  volume  = {158},
  number  = {2},
  pages   = {328--342},
  year    = {2013},
  doi     = {10.1007/s10957-012-0250-z}
}

@inproceedings{kipf2017semisupervised,
  title={Semi-Supervised Classification with Graph Convolutional Networks},
  author={Kipf, Thomas N. and Welling, Max},
  booktitle={International Conference on Learning Representations},
  year={2017},
  url={https://openreview.net/forum?id=SJU4ayYgl}
}

@inproceedings{velickovic2018graph,
  title={Graph Attention Networks},
  author={Veli{\v{c}}kovi{\'c}, Petar and Cucurull, Guillem and Casanova, Arantxa and Romero, Adriana and Li{\`o}, Pietro and Bengio, Yoshua},
  booktitle={International Conference on Learning Representations},
  year={2018},
  url={https://openreview.net/forum?id=rJXMpikCZ}
}

@inproceedings{xu2019powerful,
  title={How Powerful Are Graph Neural Networks?},
  author={Xu, Keyulu and Hu, Weihua and Leskovec, Jure and Jegelka, Stefanie},
  booktitle={International Conference on Learning Representations},
  year={2019},
  url={https://openreview.net/forum?id=ryGs6iA5Km}
}

@inproceedings{hamilton2017inductive,
  title={Inductive Representation Learning on Large Graphs},
  author={Hamilton, William L. and Ying, Rex and Leskovec, Jure},
  booktitle={Advances in Neural Information Processing Systems},
  volume={30},
  year={2017}
}

@inproceedings{bresson2017residual,
  title={Residual Gated Graph ConvNets},
  author={Bresson, Xavier and Laurent, Thomas},
  booktitle={arXiv preprint arXiv:1711.07553},
  year={2017}
}

@inproceedings{li2019deepgcns,
  title={DeepGCNs: Can GCNs Go as Deep as CNNs?},
  author={Li, Guohao and M{\"u}ller, Matthias and Thabet, Ali and Ghanem, Bernard},
  booktitle={Proceedings of the IEEE/CVF International Conference on Computer Vision},
  year={2019}
}

@inproceedings{rampasek2022recipe,
  title={Recipe for a General, Powerful, Scalable Graph Transformer},
  author={Ramp{\'a}{\v{s}}ek, Ladislav and Galkin, Michael and Dwivedi, Vijay Prakash and Luu, Anh Tuan and Wolf, Guy and Beaini, Dominique},
  booktitle={Advances in Neural Information Processing Systems},
  volume={35},
  year={2022}
}

@inproceedings{he2023generalization,
  title={A Generalization of ViT/MLP-Mixer to Graphs},
  author={He, Xiaoxin and Hooi, Bryan and Laurent, Thomas and Perold, Adam and LeCun, Yann and Bresson, Xavier},
  booktitle={International Conference on Machine Learning},
  year={2023}
}

@inproceedings{shirzad2023exphormer,
  title={Exphormer: Sparse Transformers for Graphs},
  author={Shirzad, Hamed and Velingker, Ameya and Venkatachalam, Balaji and Sutherland, Danica J. and Sinop, Ali Kemal},
  booktitle={International Conference on Machine Learning},
  year={2023}
}

@inproceedings{kiani2024unitary,
title={Unitary Convolutions for Learning on Graphs and Groups},
author={Bobak Kiani and Lukas Fesser and Melanie Weber},
booktitle={The Thirty-eighth Annual Conference on Neural Information Processing Systems},
year={2024},
url={https://openreview.net/forum?id=lG1VEQJvUH}
}

@inproceedings{maron2019universality,
  title = 	 {On the Universality of Invariant Networks},
  author =       {Maron, Haggai and Fetaya, Ethan and Segol, Nimrod and Lipman, Yaron},
  booktitle = 	 {Proceedings of the 36th International Conference on Machine Learning},
  pages = 	 {4363--4371},
  year = 	 {2019},
  editor = 	 {Chaudhuri, Kamalika and Salakhutdinov, Ruslan},
  volume = 	 {97},
  series = 	 {Proceedings of Machine Learning Research},
  month = 	 {09--15 Jun},
  publisher =    {PMLR},
  url = 	 {https://proceedings.mlr.press/v97/maron19a.html}
}

@inproceedings{keriven2019universal,
    author = {Keriven, Nicolas and Peyr\'{e}, Gabriel},
    title = {Universal invariant and equivariant graph neural networks},
    year = {2019},
    publisher = {Curran Associates Inc.},
    address = {Red Hook, NY, USA},
    booktitle = {Proceedings of the 33rd International Conference on Neural Information Processing Systems},
    articleno = {637},
    numpages = {10}
}

@inproceedings{loukas2020depthwidth,
    title={What graph neural networks cannot learn: depth vs width},
    author={Andreas Loukas},
    booktitle={International Conference on Learning Representations},
    year={2020},
    url={https://openreview.net/forum?id=B1l2bp4YwS}
}

@inproceedings{azizian2021expressive,
  title={Expressive power of invariant and equivariant graph neural networks},
  author={Azizian, Wa{\"\i}ss and Lelarge, Marc},
  booktitle={International Conference on Learning Representations},
  year={2021},
  url={https://openreview.net/forum?id=lxHgXYN4bwl}
}

@inproceedings{abboud2021random,
  title={The Surprising Power of Graph Neural Networks with Random Node Initialization},
  author    = {Ralph Abboud and {\.I}smail {\.I}lkan Ceylan and Martin Grohe 
               and Thomas Lukasiewicz},
  booktitle={Proceedings of the Thirtieth International Joint Conference on Artifical Intelligence ({IJCAI})},
  year={2021}
}

@article{schuetz2022pignn,
  title = {Graph coloring with physics-inspired graph neural networks},
  author = {Schuetz, Martin J. A. and Brubaker, J. Kyle and Zhu, Zhihuai and Katzgraber, Helmut G.},
  journal = {Phys. Rev. Res.},
  volume = {4},
  issue = {4},
  pages = {043131},
  numpages = {10},
  year = {2022},
  month = {Nov},
  publisher = {American Physical Society},
  doi = {10.1103/PhysRevResearch.4.043131},
  url = {https://link.aps.org/doi/10.1103/PhysRevResearch.4.043131}
}

@article{vanderbush2026warmstart,
  title={Neural Algorithmic Reasoning for Approximate $ k $-Coloring with Recursive Warm Starts},
  author={Vanderbush, Knut and Weber, Melanie},
  journal={Machine Learning: Science and Technology},
  year={2026}
}

@inproceedings{selman1992newmethod,
  author    = {Bart Selman and Hector J. Levesque and David G. Mitchell},
  title     = {A New Method for Solving Hard Satisfiability Problems},
  booktitle = {Proceedings of the Tenth National Conference on Artificial Intelligence (AAAI-92)},
  pages     = {440--446},
  year      = {1992},
  publisher = {AAAI Press},
  address   = {San Jose, California}
}

@article{vangelder2008coloring,
  author  = {Allen Van Gelder},
  title   = {Another Look at Graph Coloring via Propositional Satisfiability},
  journal = {Discrete Applied Mathematics},
  volume  = {156},
  number  = {2},
  pages   = {230--243},
  year    = {2008},
  doi     = {10.1016/j.dam.2006.07.016}
}

@article{daniel1979dsatur,
author = {Br\'{e}laz, Daniel},
title = {New methods to color the vertices of a graph},
year = {1979},
issue_date = {April 1979},
publisher = {Association for Computing Machinery},
address = {New York, NY, USA},
volume = {22},
number = {4},
issn = {0001-0782},
url = {https://doi.org/10.1145/359094.359101},
doi = {10.1145/359094.359101},
journal = {Commun. ACM},
month = apr,
pages = {251–256},
numpages = {6}
}

@inproceedings{ji2020directional,
  author    = {Ji, Ziwei and Telgarsky, Matus},
  title     = {Directional Convergence and Implicit Bias in Deep Learning},
  booktitle = {Advances in Neural Information Processing Systems},
  year      = {2020}
}

@inproceedings{lyu2019gradient,
  author    = {Lyu, Kaifeng and Li, Jian},
  title     = {Gradient Descent Maximizes the Margin of Homogeneous Neural Networks},
  booktitle = {International Conference on Learning Representations},
  year      = {2020}
}

@article{papyan2020prevalence,
  author  = {Papyan, Vardan and Han, Xiaohui and Donoho, David},
  title   = {Prevalence of Neural Collapse during the Terminal Phase of Deep Learning Training},
  journal = {Proceedings of the National Academy of Sciences},
  volume  = {117},
  number  = {40},
  pages   = {24652--24663},
  year    = {2020}
}

@article{soudry2018implicit,
  author  = {Soudry, Daniel and Hoffer, Elad and Nacson, Mor Shpigel and Gunasekar, Suriya and Srebro, Nathan},
  title   = {The Implicit Bias of Gradient Descent on Separable Data},
  journal = {Journal of Machine Learning Research},
  volume  = {19},
  number  = {70},
  pages   = {1--57},
  year    = {2018}
}

@article{lovasz1979shannon,
  author  = {Lov{\'a}sz, L{\'a}szl{\'o}},
  title   = {On the {S}hannon Capacity of a Graph},
  journal = {IEEE Transactions on Information Theory},
  volume  = {25},
  number  = {1},
  pages   = {1--7},
  year    = {1979}
}

@article{Wigderson1983,
  author  = {Wigderson, Avi},
  title   = {Improving the Performance Guarantee for Approximate Graph Coloring},
  journal = {Journal of the ACM},
  volume  = {30},
  number  = {4},
  pages   = {729--735},
  year    = {1983}
}

@article{karger98colorsdp,
  author  = {Karger, David R. and Motwani, Rajeev and Sudan, Madhu},
  title   = {Approximate Graph Coloring by Semidefinite Programming},
  journal = {Journal of the ACM},
  volume  = {45},
  number  = {2},
  pages   = {246--265},
  year    = {1998}
}

@article{feigekilian1998zeroknowledge,
  author  = {Feige, Uriel and Kilian, Joe},
  title   = {Zero Knowledge and the Chromatic Number},
  journal = {Journal of Computer and System Sciences},
  volume  = {57},
  number  = {2},
  pages   = {187--199},
  year    = {1998}
}

@article{dinur2009coloringugc,
  author  = {Dinur, Irit and Mossel, Elchanan and Regev, Oded},
  title   = {Conditional Hardness for Approximate Coloring},
  journal = {SIAM Journal on Computing},
  volume  = {39},
  number  = {3},
  pages   = {843--873},
  year    = {2009}
}

@inproceedings{loshchilov2019decoupled,
  author    = {Ilya Loshchilov and Frank Hutter},
  title     = {Decoupled Weight Decay Regularization},
  booktitle = {International Conference on Learning Representations},
  year      = {2019},
  url       = {https://openreview.net/forum?id=Bkg6RiCqY7}
}

@inproceedings{yang2016revisiting,
  author    = {Zhilin Yang and William Cohen and Ruslan Salakhudinov},
  title     = {Revisiting Semi-Supervised Learning with Graph Embeddings},
  booktitle = {Proceedings of the 33rd International Conference on Machine Learning},
  series    = {Proceedings of Machine Learning Research},
  volume    = {48},
  pages     = {40--48},
  year      = {2016},
  publisher = {PMLR},
  url       = {https://proceedings.mlr.press/v48/yanga16.html}
}

@book{johnson1996cliques,
  editor    = {David S. Johnson and Michael A. Trick},
  title     = {Cliques, Coloring, and Satisfiability: Second {DIMACS} Implementation Challenge},
  series    = {DIMACS Series in Discrete Mathematics and Theoretical Computer Science},
  volume    = {26},
  publisher = {American Mathematical Society},
  address   = {Providence, RI},
  year      = {1996}
}
